\documentclass[runningheads]{llncs}

\usepackage{eccv}

\usepackage{eccvabbrv}

\usepackage{graphicx}
\usepackage{booktabs}
\usepackage[numbers,sort&compress]{natbib}
\usepackage[accsupp]{axessibility}  % Improves PDF readability for those with disabilities.

\usepackage{hyperref}

\usepackage{orcidlink}
\usepackage{microtype}
\usepackage{graphicx}
\usepackage{subcaption}
\usepackage{booktabs}
\usepackage{hyperref}
\usepackage{amsmath}
\usepackage{amssymb}
\usepackage{mathtools}

\usepackage{amsthm}
\usepackage{empheq}
\usepackage{enumitem}
\usepackage{nicefrac}
\usepackage{enumitem}
\setlist{noitemsep}
\usepackage{caption}
\usepackage{wrapfig}
\usepackage[capitalize,noabbrev]{cleveref}
\RequirePackage{algorithm}
\RequirePackage{algorithmic}
\newcommand{\RETURN}{\textbf{return}} 

\theoremstyle{plain}
\theoremstyle{definition}
\newtheorem{assumption}[theorem]{Assumption}
\theoremstyle{remark}

\begin{document}

\title{Spatially Adaptive Noise Injection} 

\titlerunning{Spatially Adaptive Noise Injection}

\author{Frantzeska Lavda\inst{1}\orcidlink{0000-0003-2868-8380} \and
Maciej Falkiewicz\inst{1, 2}\orcidlink{0000-0002-3851-1995} \and
Van Khoa Nguyen\inst{1, 2}\orcidlink{0009-0007-9328-790X} \and
Alexandros Kalousis\inst{1}\orcidlink{0000-0001-6282-0686}}

\authorrunning{F.~Lavda et al.}

\institute{University of Applied Sciences Western Switzerland, HES-SO Geneva \and
Computer Science Department, University of Geneva
\\
%\url{http://www.springer.com/gp/computer-science/lncs} 
\email{frantzeska.lavda@hesge.ch}}

\maketitle

\begin{abstract}
  Diffusion samplers reverse a learned noising process using either stochastic (DDPM) or deterministic (DDIM) updates, which represent endpoints of a single family controlled by a scalar noise-injection variance that is applied identically at every spatial location. This uniform approach neglects the geometry of natural images: high-curvature regions such as edges and textures, where the denoiser is uncertain, benefit from stochastic correction, whereas smooth regions, where the score is precise, are degraded by injected noise. This work investigates whether each pixel requires stochastic correction at a given timestep and introduces Spatially Adaptive Noise Injection (SANI), a novel sampling framework that dynamically adjusts noise application on a per-pixel basis. SANI integrates a probabilistic gating mechanism with a derived spatially adaptive variance, ensuring that noise is injected precisely where needed to refine complex features while preserving well-formed structures. Experimental results and decoupling ablations demonstrate that SANI consistently improves Fréchet Inception Distance (FID) over the vanilla DDPM and DDIM endpoint samplers across diverse sampling timesteps, while remaining competitive with variance-learning baselines, highlighting the importance of spatial adaptivity in diffusion sampling.
  %\and Fisher information \and Uncertainty 
\end{abstract}
%% ====================================================================
\section{Introduction}
\label{sec:intro}
%% ====================================================================
Diffusion Probabilistic Models (DPMs)~\citep{ho20, song2021scorebased} have become a leading approach in generative modeling, achieving state-of-the-art sample quality across images, audio, and molecular design. These models learn to reverse a gradual noising process by training a denoising network, resulting in a tractable variational bound and stable training. During inference, two primary families of samplers are typically used. Stochastic samplers~\citep{ho20, Kingma21_vdm} inject noise at each step, emulating Langevin dynamics to correct approximation errors. In contrast, deterministic samplers~\citep{song2021ddim, lu2022dpmsolver, ludpm++} follow a noise-free probability flow, ordinary differential equation (ODE) trajectory, which accelerates generation but limits error correction. \cite{song2021ddim} demonstrated that these two regimes present endpoints of an unified family, interpolated by a scalar parameter~$\eta$ that governs the noise injection variance at each step. However, a key limitation of both approaches is that the noise injection variance~$\sigma_t^2$ is applied uniformly across all spatial location, discarding the geometric heterogeneity inherent in high-dimensional data manifolds.

High-frequency regions, such as edges, textures, and fine details, correspond to areas of high curvature on the data manifold, where the denoiser's gradient changes rapidly and estimation error tends to be higher. These regions benefit from stochastic correction~\citep{song2024high_frequency}. In contrast, smooth regions such as uniform backgrounds are situated on locally flat portions of the manifold, where the score estimate is accurate and additional noise injection reduces fidelity. Ideally, a sampler would adapt its behavior to the local geometry: employing deterministic dynamics in regions where the model is confident and stochastic dynamics where uncertainty is greater, within the same image and timestep.

This consideration leads to a central question: \emph{how can one determine, at each pixel and timestep, whether stochastic correction is necessary?} We address this by deriving the Fisher Information of the denoising distribution, which provides a closed-form per-pixel measure of uncertainty without requiring additional supervision. The analysis begins with the observation that the denoising distribution $p(\mathbf{x}_0 \mid \mathbf{x}_t)$ constitutes an exponential family parametrized by the noisy observation~$\mathbf{x}_t$. This structure enables the derivation of the Fisher Information Matrix (FIM) of the posterior with respect to~$\mathbf{x}_t$ in closed form. We show that this FIM is proportional to both the Jacobian of the optimal denoiser and the posterior covariance. The trace of this relationship establishes a direct connection between the model's total sensitivity and reconstruction uncertainty. Regions exhibiting large Fisher Information correspond to areas where the model is uncertain about the underlying signal, indicating that stochastic correction is most beneficial in these locations.

Building on this insight, we introduce Spatially Adaptive Noise Injection (SANI), a sampling framework that modulates stochasticity on a per-pixel basis. We derive a closed-form gating function $g_{t,i}\in[0,1]$ with a clear probabilistic interpretation, representing the probability that the local denoising error at pixel~$i$ exceeds a specified tolerance threshold (Prop.~\ref{prop:gating}). This gating mechanism adjusts noise injection at each pixel, interpolating  per pixel between between ODE dynamics ($g_{t,i}=0$, deterministic) and stochastic DDPM-style update ($g_{t,i}=1$) within a single image and timestep. As a result, stochasticity increases in uncertain pixels while deterministic per-pixel updates are enforced in confident regions. The gating function is computed from the per-pixel posterior variance, which is directly related to the model's sensitivity. SANI operates at inference time and is compatible with any pre-trained diffusion model.

In summary, the main contributions are as follows: \textbf{(i) Theoretical Link.} We show that the Fisher Information Matrix of the denoising distribution is proportional to the denoiser Jacobian and posterior covariance (Prop.~\ref{prop:fim}), establishing a computable relationship between local sensitivity and reconstruction uncertainty. \textbf{(ii) Probabilistic gating.} We derive a closed-form, per-pixel gating function $g_{t,i}$ representing the probability that the local denoising error exceeds a specific tolerance (Eq.~\ref{eq:gating_closed}), and show that its spatial structure is robust to the choice of the residual-distribution model. \textbf{(iii) Spatially adaptive sampler.} We present SANI (Alg.~\ref{alg:sani}), which applies the deterministic (DDIM) and stochastic (DDPM-style) update rules on a per-pixel basis. \textbf{(iv) Empirical Validation.} We show that SANI consistently improves over vanilla DDPM and DDIM across all datasets and step counts, and is competitive with variance-learning DDPM-family samplers, with the largest gains in the low-step regime, while its gating maps track image geometry without requiring edge supervision.
%(the underlying Jacobian--posterior-covariance identity follows from classical results on conditional mean estimation in Gaussian noise~\citep{dytso2020general, manor2024posterior})
%% ====================================================================
\section{Background}
\label{sec:background}
%% ====================================================================
\subsection{Denoising Diffusion Probabilistic Models}
\label{sec:ddpm}
Diffusion models define a forward noising process and a learned reverse denoising process. Given data $\mathbf{x}_0 \sim q(\mathbf{x}_0)$, the forward process adds Gaussian noise according to a schedule $\beta_1, \ldots, \beta_T \in (0,1)$:
\begin{equation}\label{eq:forward_step}
    q(\mathbf{x}_t \mid \mathbf{x}_{t-1}) = \mathcal{N}\!\big(\mathbf{x}_t;\, \sqrt{1-\beta_t}\,\mathbf{x}_{t-1},\, \beta_t \mathbf{I}\big).
\end{equation}
Defining $\alpha_t = 1-\beta_t$ and $\bar\alpha_t = \prod_{s=1}^{t}\alpha_s$, the marginal at any timestep is $q(\mathbf{x}_t \mid \mathbf{x}_0) = \mathcal{N}(\mathbf{x}_t;\, \sqrt{\bar\alpha_t}\,\mathbf{x}_0,\, (1-\bar\alpha_t)\mathbf{I})$, so that $\mathbf{x}_t = \sqrt{\bar\alpha_t}\,\mathbf{x}_0 + \sqrt{1-\bar\alpha_t}\,\boldsymbol{\epsilon}$ with $\boldsymbol{\epsilon}\sim\mathcal{N}(\mathbf{0},\mathbf{I})$.
The reverse process learns to invert this chain by parameterizing $p_\theta(\mathbf{x}_{t-1}\mid\mathbf{x}_t) = \mathcal{N}(\mathbf{x}_{t-1};\,\boldsymbol{\mu}_\theta(\mathbf{x}_t,t),\,\boldsymbol{\Sigma}_\theta(\mathbf{x}_t,t))$.
Training maximizes the ELBO on $\log p_\theta(\mathbf{x}_0)$ via a noise prediction network $\boldsymbol{\epsilon}_\theta(\mathbf{x}_t,t)$ with the simplified objective $\mathcal{L}_{\mathrm{simple}} = \mathbb{E}_{t,\mathbf{x}_0,\boldsymbol{\epsilon}}[\|\boldsymbol{\epsilon} - \boldsymbol{\epsilon}_\theta(\mathbf{x}_t,t)\|^2]$.
This objective provides a learning signal only for the mean $\boldsymbol{\mu}_\theta$, the covariance $\boldsymbol{\Sigma}_\theta$ must be specified separately. Ho \etal~\cite{ho20} set it to one of two scalar schedules, $\sigma_t^2 = \beta_t$ or $\sigma_t^2 = \tilde\beta_t \coloneqq \frac{1-\bar\alpha_{t-1}}{1-\bar\alpha_t}\beta_t$, both spatially uniform.
\subsection{Generalized Sampling: Unifying DDPM and DDIM}
\label{sec:gen_update}
The noise prediction network implicitly defines a clean data estimate via Tweedie's formula~\citep{efron2011tweedie}: $\hat{\mathbf{x}}_0(\mathbf{x}_t,t) = \frac{1}{\sqrt{\bar\alpha_t}}(\mathbf{x}_t - \sqrt{1-\bar\alpha_t}\,\boldsymbol{\epsilon}_\theta(\mathbf{x}_t,t))$.
Song \etal~\cite{song2021ddim} showed that a family of non-Markovian reverse processes, parameterized by the noise injection variance $\sigma_t^2$, shares the same marginals as the forward process. The generalized update rule is:
\begin{equation}\label{eq:general_update}
    \mathbf{x}_{t-1} = \underbrace{\sqrt{\bar\alpha_{t-1}}\,\hat{\mathbf{x}}_0}_{\text{denoised mean}} + \underbrace{\sqrt{1-\bar\alpha_{t-1}-\sigma_t^2}\,\boldsymbol{\epsilon}_\theta(\mathbf{x}_t,t)}_{\text{direction correction}} + \underbrace{\sigma_t\,\mathbf{z}}_{\text{noise}},
\end{equation}
where $\mathbf{z}\sim\mathcal{N}(\mathbf{0},\mathbf{I})$ for $t>1$ and $\mathbf{z}=\mathbf{0}$ for $t=1$.
Setting $\sigma_t^2=0$ yields DDIM (deterministic), setting $\sigma_t^2=\tilde\beta_t$ recovers DDPM (stochastic).
\paragraph{\textbf{The spatial uniformity bottleneck.}}
In all of these methods, $\sigma_t^2$ is a scalar applied identically to every spatial location. This ignores the heterogeneous geometry of natural images: edges and textures lie on curved manifold regions where stochastic correction is beneficial, while smooth regions lie on flat portions where noise degrades fidelity. We argue that this trade-off should not be global but spatially adaptive.
% ===========================================================================
%  SECTION 3:  SPATIALLY ADAPTIVE NOISE INJECTION
%  
% ===========================================================================
%
\section{Spatially Adaptive Noise Injection}
\label{sec:method}
This section presents the proposed method and the theoretical framework. We first show that the local geometry of the diffusion process, as quantified by the Fisher Information, is directly related to posterior reconstruction uncertainty. This relationship is then leveraged to develop a sampling algorithm. The general concept involves deriving a gating function $g_{t,i} \in [0, 1]$ with clear probabilistic interpretation, which is used to modulate the per-pixel noise injection variance. This approach enables each pixel to follow its geometrically optimal trajectory between deterministic and stochastic dynamics.
%
% ───────────────────────────────────────────────────────────────────────
\subsection{Fisher Information, Local Sensitivity and Model's Uncertainty}
\label{sec:method:uncertainty}
The denoising distribution $p(\mathbf{x}_0\mid\mathbf{x}_t)$ for any fixed timestep~$t$, indexed by the noisy observation~$\mathbf{x}_t$, belongs to an exponential family. The  exponential-family structure guarantees that the posterior mean and covariance are related to the log-partition function by standard identities, and it enables the application of classical results from information geometry, in particular, the Fisher Information Matrix.

\begin{proposition}[FIM-Jacobian-Posterior Covariance Correspondence]
\label{prop:fim}
Let\, $\hat{\mathbf{x}}_0(\mathbf{x}_t)
  = \mathbb{E}[\mathbf{x}_0 \mid \mathbf{x}_t]$
denote the optimal (MMSE) denoiser. The Fisher Information Matrix (FIM) of the denoising distribution with respect to~$\mathbf{x}_t$ satisfies
\begin{equation}\label{eq:fim_jacobian}
  \boxed{
  \mathcal{I}(\mathbf{x}_t)
  \;=\;
  \frac{\sqrt{\bar\alpha_t}}{1 - \bar\alpha_t}\;
  \mathbf{J}_{\hat{\mathbf{x}}_0}(\mathbf{x}_t) 
  \;=\;
  \frac{\bar\alpha_t}{(1 - \bar\alpha_t)^2}\;
  \mathrm{Cov}[\mathbf{x}_0 \mid \mathbf{x}_t],
  }
\end{equation}
where\,
$\mathbf{J}_{\hat{\mathbf{x}}_0}
  = \frac{\partial \hat{\mathbf{x}}_0(\mathbf{x}_t)}{\partial \mathbf{x}_t}
  \in \mathbb{R}^{D \times D}$ is the Jacobian of the optimal denoiser.
\end{proposition}

The underlying Jacobian-posterior-covariance relation in Eq.~\eqref{eq:fim_jacobian} is closely related to classical results on conditional mean estimation in Gaussian noise~\citep{dytso2020general}, and per-pixel posterior moments computed from denoiser derivatives~\citep{manor2024posterior}. Proposition~\ref{prop:fim} is the information-geometric reading of these identities, identifying the shared quantity as the Fisher Information of the denoising distribution, which underlies the gating construction developed next.

The FIM quantifies the sensitivity of a distribution to perturbations in its parameters. In the context of the denoising distribution, the primary interest lies in sensitivity with respect to the observation~$\mathbf{x}_t$, specifically, the extent to which the posterior belief $\mathbf{x}_0$ about  changes when $\mathbf{x}_t$ is perturbed. Proposition~\ref{prop:fim} characterizes model sensitivity and uncertainty via the posterior covariance and the denoiser Jacobian. Regions exhibiting high sensitivity correspond to areas where the model is uncertain about the underlying clean signal. Furthermore, local sensitivity, as measured by the total Fisher Information, is proportional to reconstruction uncertainty (mean squared error, MSE), $\mathrm{Tr}(\mathcal{I})=\bar\alpha_t(1-\bar\alpha_t)^{-2},\mathrm{MSE}(\mathbf{x}_t)$. Areas with large $\mathrm{Tr}(\mathcal{I})$ indicate high model uncertainty, which are precisely the regions where stochastic correction is most beneficial.

Via Tweedie's formula, the denoiser Jacobian decomposes as $\mathbf{J}_{\hat{\mathbf{x}}_0} = \frac{1}{\sqrt{\bar\alpha_t}}(\mathbf{I}_D + (1-\bar\alpha_t)\mathbf{H}_t(\mathbf{x}_t))$, where $\mathbf{H}_t(\mathbf{x}_t) \coloneqq \nabla_{\mathbf{x}_t}\mathbf{s}_\theta(\mathbf{x}_t,t)$ approximates the Hessian of the log-marginal density.
Substituting into~ Eq.\eqref{eq:fim_jacobian} yields the posterior in terms of the log-marginal density Hessian:
\begin{equation}\label{eq:cov_hessian}
    \mathrm{Cov}[\mathbf{x}_0\mid\mathbf{x}_t] = \frac{1-\bar\alpha_t}{\bar\alpha_t}\big(\mathbf{I}_D + (1-\bar\alpha_t)\mathbf{H}_t(\mathbf{x}_t)\big).
\end{equation}
Extracting the diagonal gives the per-pixel posterior variance:
\begin{equation}\label{eq:v_i}
  v_i(\mathbf{x}_t,t)
  \;\coloneqq\;\bigl[\mathrm{Cov}[\mathbf{x}_0\mid\mathbf{x}_t]\bigr]_{ii}
  \;=\;\frac{1-\bar\alpha_t}{\bar\alpha_t}\bigl(1+(1-\bar\alpha_t)[\mathbf{H}_t]_{ii}\bigr).
\end{equation}
The sign of $[\mathbf{H}_t]_{ii}$ determines the local uncertainty regime. Strong concavity ($[\mathbf{H}_t]_{ii} \ll 0$) indicates that the marginal distribution is sharply peaked along coordinate $i$, resulting in a small $v_i$ and thus high confidence. In contrast, positive curvature is indicative of high uncertainty.

\paragraph{\textbf{Validity and clamping.}}
The right-hand side of Eq.~\ref{eq:cov_hessian} defines a valid covariance matrix if and only if all eigenvalues satisfy $\lambda_{\min}(\mathbf{H}_t^\star) \geq -(1-\bar\alpha_t)^{-1}$, a condition met by the true log-marginal Hessian $\mathbf{H}_t^\star$ by construction. Since the true Hessian $\mathbf{H}_t^\star$ is not accessible, it is approximated using either the Hutchinson estimator or a learned Hessian network, i.e., $\mathbf{H}_t(\mathbf{x}_t) \approx \mathbf{H}_t^\star(\mathbf{x}_t)$. The approximate Hessian may not satisfy the eigenvalue condition, which can result in the right-hand side of Eq.~\ref{eq:cov_hessian} failing to be positive semi-definite (PSD). Such violations, however, only occur in regions where the true variance $v_i^*$ is already close to zero (see App.~\cref{app:psd}). In these cases, $v_i$ is clamped as $v_i \gets \max(v_i, \epsilon)$ with $\epsilon=10^{-6}$, causing the affected pixels to default to deterministic dynamics.
%
% ───────────────────────────────────────────────────────────────────────
\subsection{Probabilistic Gating Function}
\label{sec:method:gating}
The per-pixel variance $v_i(\mathbf{x}_t, t)$ is mapped to a gating function $g_{t,i} \in [0,1]$ that regulates local noise injection.
\begin{definition}[Gating Function]
\label{def:gating}
The gate at pixel~$i$ and timestep~$t$ is defined as the probability that the squared denoising error exceeds a tolerance~$\tau$:
\begin{equation}\label{eq:gating_def}
  g_{t,i}\;\coloneqq\;\mathbb{P}\!\bigl((X_{0,i}-\hat{x}_{0,i})^2>\tau\,\big|\,\mathbf{X}_t=\mathbf{x}_t\bigr).
\end{equation}
\end{definition}
Assuming (A1) that the denoiser approximates the minimum mean squared error (MMSE) estimator, $\hat{x}_{0,i} \approx \mathbb{E}[X_{0,i} \mid \mathbf{x}_t]$, as justified by the squared-error training objective, and (A2) that the per-pixel posterior is approximately Gaussian (the maximum-entropy distribution given the first two moments from the score network), the residual $R_i = X_{0,i} - \hat{x}_{0,i} \mid \mathbf{x}_t$ is a zero-mean Gaussian with variance $v_i$. Standardization under these assumptions yields a closed-form expression for the gating function.
\begin{proposition}[Closed-Form Gating]
\label{prop:gating}
Under Assumptions~(A1)-(A2), the gating function yields the closed form
\begin{equation}\label{eq:gating_closed}
  g_{t,i}\;=\;2\Bigl(1-\Phi\bigl(\sqrt{\tau/v_i(\mathbf{x}_t,t)}\bigr)\Bigr),
\end{equation}
where $\Phi(\cdot)$ is the standard normal CDF.
\end{proposition}
The closed-form expression depends on $v_i$ solely through the ratio $\tau/v_i$, meaning that $\tau$ determines the threshold at which a pixel transitions from the stochastic regime ($v_i \gg \tau$, $g_{t,i} \to 1$) to the deterministic regime ($v_i \ll \tau$, $g_{t,i} \to 0$). The clamping operation $v_i \gets \max(v_i, \epsilon)$ ensures the argument remains real-valued. Assumption (A2) holds exactly in the high signal-to-noise ratio (SNR) limit (see App.~\cref{app:gaussian_concentration}), and empirical results in \cref{sec:exp_residual} indicate that the spatial gating pattern is robust to the choice of residual distribution. Re-derivation of the gate under Laplace residuals (App.~\cref{app:alt_gates}) affects only the sharpness of the stochastic-to-deterministic transition.
%
% ───────────────────────────────────────────────────────────────────────
\subsection{KL-Optimal Per-Pixel Noise Injection}
\label{sec:method:variance}
The gating variable $g_{t,i}$ determines whether a pixel receives stochastic correction. Conventional samplers employ a spatially uniform scale $\sigma_t^2\in\{\tilde\beta_t,\beta_t\}$, which does not account for local reconstruction quality or the model's prediction error. Instead, the proposed approach derives the scale that minimizes the expected Kullback-Leibler (KL) divergence between the true reverse transition $q(x_{t-1,i} \mid \mathbf{x}_t, x_{0,i})$ and the model transition $p_\theta(x_{t-1,i} \mid \mathbf{x}_t)$, with the model mean fixed at the Tweedie estimate $\hat\mu_{t,i}$ (see App.~\cref{app:kl_derivation}).
%We now decide how much. We answer this by minimizing the expected KL divergence between the true reverse transition $q(x_{t-1,i}\mid\mathbf{x}t,x{0,i})$ and the model transition $p_\theta(x_{t-1,i}\mid\mathbf{x}t)$ with the model mean fixed at the Tweedie estimate~$\hat\mu{t,i}$ (\cref{app:kl_derivation}).
%
\begin{proposition}[KL-Optimal Per-Pixel Variance]
\label{prop:kl_opt}
With the model mean fixed at~$\hat\mu_{t,i}$, the per-pixel variance that minimizes the expected KL to $q(x_{t-1,i}\mid\mathbf{x}_t,x_{0,i})$ is
\begin{equation}\label{eq:gamma_opt}
  \gamma^*_{t,i}\;=\;\tilde\beta_t+c_t^2\,v_i(\mathbf{x}_t,t),
  \qquad
  c_t\;\coloneqq\;\frac{\sqrt{\bar\alpha_{t-1}}\,\beta_t}{1-\bar\alpha_t}.
\end{equation}
\end{proposition}
The optimal variance has two terms with structurally different roles. This distinction drives how we allocate stochasticity in space. The term~$\tilde\beta_t$ is the variance of the true reverse posterior, the irreducible stochasticity of the reverse SDE, present even when the denoiser is perfect ($v_i=0$). The term~$c_t^2 v_i$, by contrast, equals the expected squared error of the Tweedie mean (App.~\cref{app:coupled}), 
\begin{equation}\label{eq:mean_err}
  \mathbb{E}\!\bigl[(\tilde\mu_{t,i}-\hat\mu_{t,i})^2\,\big|\,\mathbf{x}_t\bigr]
  \;=\;c_t^2\,v_i.
\end{equation}
Since SANI substitutes $\hat\mu_{t,i}$ for the unknown true posterior mean, this variance is required for the transition to remain calibrated to the true posterior. It is determined by the mean estimate and is not a tunable parameter.
This asymmetry motivates our spatial allocation strategy, the calibration term $c_t^2 v_i$ is applied at every pixel, while only the exploratory term $\tilde\beta_t$ is gated.
\begin{equation}\label{eq:sigma_split}
  \Sigma_{t,i}^2
  \;=\;\underbrace{c_t^2\,v_i\vphantom{\tilde\beta_t}}_{\text{mean-error correction}}
  \;+\;\underbrace{g_{t,i}\,\tilde\beta_t}_{\text{gated exploration}}.
\end{equation}
Gating only $\tilde\beta_t$ ensures that mean-error correction is preserved at pixels where the substituted Tweedie mean is least reliable. For confident pixels, both $g_{t,i}$ and $v_i$ are small, resulting in $\Sigma_{t,i}^2 \to 0$ and the update reducing to the deterministic DDIM step. In contrast, for uncertain pixels, $g_{t,i} \to 1$ and $\Sigma_{t,i}^2 \to \gamma_{t,i}$, corresponding to the KL-optimal variance.
\paragraph{\textbf{Scope of the optimality.}}
We emphasize that $\gamma^*_{t,i}$ is KL-optimal for the fully stochastic per-pixel transition with the mean fixed at the Tweedie estimate. The gated variance $\Sigma_{t,i}^2$ coincides with $\gamma^*_{t,i}$ only in the limit $g_{t,i}=1$. For $g_{t,i}<1$, and after the admissibility clip (App.~\cref{app:admissibility}), it departs from the KL-optimal value in exchange for deterministic updates at confident pixels. $\gamma^*_{t,i}$ therefore serves as the stochastic endpoint of the interpolation rather than a global optimality guarantee for SANI.
%
% ───────────────────────────────────────────────────────────────────────
%
\subsection{Spatially Adaptive Noise Injection (SANI) Update}
\label{sec:method:update}
By substituting Eq.~\ref{eq:sigma_split} into the generalized reverse step from~\cite{song2021ddim}, the SANI update is obtained.
\begin{proposition}[SANI Update]
\label{prop:update}
Let $\hat{\mathbf{x}}_0=\bar\alpha_t^{-1/2}(\mathbf{x}_t-\sqrt{1-\bar\alpha_t}\,\boldsymbol{\epsilon}_\theta(\mathbf{x}_t,t))$ denote the Tweedie estimate and $\mathbf{z}\sim\mathcal{N}(\mathbf{0},\mathbf{I})$. With $\Sigma_{t,i}^2$ as in Eq.~\ref{eq:sigma_split}, the SANI update at pixel~$i$ is
\begin{equation}\label{eq:sani_update}
  x_{t-1,i}
  \;=\;\sqrt{\bar\alpha_{t-1}}\,\hat{x}_{0,i}
  \;+\;\sqrt{1-\bar\alpha_{t-1}-\Sigma_{t,i}^2}\;\epsilon_{\theta,i}
  \;+\;\Sigma_{t,i}\,z_i,
\end{equation}
for $t>1$, with $\mathbf{z}=\mathbf{0}$ for $t=1$.
\end{proposition}
The full procedure is summarized in \cref{alg:sani}.\\ %SANI requires no retraining of the diffusion model and is compatible with any pre-trained checkpoint.
{\textbf{Coupled dynamics.}}
The coupling between stochastic and deterministic components is explicit in Eq.~\ref{eq:sani_update}. As $\Sigma_{t,i}$ increases, the noise term $\Sigma_{t,i}z_i$ becomes larger, while the direction coefficient $\sqrt{1-\bar\alpha_{t-1}-\Sigma_{t,i}^2}$ decreases. Because $\Sigma_{t,i}^2=c_t^2 v_i+g_{t,i}\tilde\beta_t$ increases with both~$v_i$ and~$g_{t,i}$, the predicted direction is automatically down-weighted and replaced by stochastic exploration in regions where the Tweedie mean is unreliable. The same quantity that increases the variance in Eq.~\ref{eq:mean_err} also indicates the unreliability of the direction. Therefore, a decoupled rule that adjusts only the variance would continue to follow an inaccurate direction with full weight.\\
%
% new 12/8
%{\textbf{Relation to the DDIM family.}}
%For a scalar, state-independent noise scale, the non-Markovian family of~\cite{song2021ddim} provably preserves the marginals of the forward process. The SANI variance $\Sigma_{t,i}^2$ is per-pixel and depends on $\mathbf{x}_t$ through $v_i$ and $g_{t,i}$, so this marginal-preservation argument does not carry over, and we do not claim that the spatially adaptive transition matches the forward marginals. SANI should therefore be understood as a per-pixel update rule whose endpoints coincide with the DDIM and DDPM updates, supported empirically in \cref{sec:experiments}\\
\textbf{Computational cost.}
Evaluating the gate requires the Hessian diagonal $[\mathbf{H}_t]_{ii}$. Using the lightweight pretrained network of~\cite{ou2025improving}, this adds a single forward pass of a model substantially smaller than the score network, i.e., a small constant overhead per step. Alternatively, a single Hutchinson probe (one Jacobian-vector product per step) requires no auxiliary network at roughly $2\times$ the cost of a DDIM step. All remaining SANI operations are elementwise in the pixels. Details are provided in App.~\cref{app:experimental_details}.
\begin{algorithm}[t]
\small
\caption{SANI Sampling}
\label{alg:sani}
\begin{algorithmic}
  \REQUIRE $\mathbf{x}_T\sim\mathcal{N}(\mathbf{0},\mathbf{I})$, model $\boldsymbol{\epsilon}_\theta$, schedule $\{\bar\alpha_t,\tilde\beta_t,c_t\}$, tolerance $\tau$, floor $\epsilon$
  \FOR{$t=T,\ldots,1$}
    \STATE $\mathbf{z}\sim\mathcal{N}(\mathbf{0},\mathbf{I})$ if $t>1$, else $\mathbf{z}=\mathbf{0}$
    \STATE $\hat{\boldsymbol{\epsilon}}\gets\boldsymbol{\epsilon}_\theta(\mathbf{x}_t,t)$;\quad $\hat{\mathbf{x}}_0\gets\bar\alpha_t^{-1/2}(\mathbf{x}_t-\sqrt{1-\bar\alpha_t}\,\hat{\boldsymbol{\epsilon}})$
    \STATE $\mathbf{v}\gets\mathrm{diag}\bigl(\mathrm{Cov}[\mathbf{x}_0\mid\mathbf{x}_t]\bigr)$;\quad $v_i\gets\max(v_i,\epsilon)$ \COMMENT{per-pixel variance, Eq~\ref{eq:v_i}}
    \STATE $\mathbf{g}_t\gets 2\bigl(\mathbf{1}-\Phi(\sqrt{\tau/\mathbf{v}})\bigr)$ \COMMENT{per-pixel gate, Eq.~\ref{eq:gating_closed}}
    \STATE $\boldsymbol{\Sigma}_t^2\gets c_t^2\,\mathbf{v}+\mathbf{g}_t\odot\tilde\beta_t$;\quad $\boldsymbol{\Sigma}_t^2\gets\min(\boldsymbol{\Sigma}_t^2,\,1-\bar\alpha_{t-1})$ \COMMENT{noise injection, Eq.~\ref{eq:sigma_split}}
    \STATE $\mathbf{x}_{t-1}\gets\sqrt{\bar\alpha_{t-1}}\,\hat{\mathbf{x}}_0+\sqrt{1-\bar\alpha_{t-1}-\boldsymbol{\Sigma}_t^2}\odot\hat{\boldsymbol{\epsilon}}+\boldsymbol{\Sigma}_t\odot\mathbf{z}$
  \ENDFOR
  
  \RETURN \, $\mathbf{x}_0$
\end{algorithmic}
\end{algorithm}
% ===========================================================================
% ===========================================================================
\section{Experiments}
\label{sec:experiments}
 The experimental evaluation follows the logical structure of the proposed method. First, we validate empirically the theoretical identities from \cref{sec:method:uncertainty}. Next, we show the alignment of the gating function with image structure without any edge supervision (\cref{sec:exp_qualitative}). After confirming the intended mechanism, we test the sole distributional approximation on which the method depends (\cref{sec:exp_residual}) and characterize the gating function along its two axes, the tolerance $\tau$ and the reverse-process time (\cref{sec:exp_gating_behavior}). Finally, we evaluate generation quality (\cref{sec:exp_quantitative}), showing that spatial adaptivity improves FID over stochastic samplers and we conclude with a decoupling ablation (\cref{sec:exp_ablation}) to isolate the contributions of per-pixel gating, the asymmetric allocation of the calibration term, and the coupling between injected variance and the direction term.
%\textbf{Datasets and models.} We experiment on CIFAR10 ($32 \times 32$) with both the linear (LS) and cosine (CS) noise schedules, CelebA ($64 \times 64$), and LSUN Bedroom ($256 \times 256$). For each dataset we use the publicly available pre-trained DDPM checkpoints from~\cite{ho20,nichol2021improved} and the Hessian-diagonal networks from~\cite{ou2025improving} where available. SANI requires no retraining the diffusion model, all experiments use inference-time gating only.

%\textbf{Baselines.} We compare against DDPM with $\sigma_t^2 = \tilde{\beta}_t$ and $\sigma_t^2 = \beta_t$~\citep{ho20}, DDIM~\citep{song2021ddim}, Analytic-DPM (A-DDPM/DDIM)~\citep{bao2022analytic}, NPR DDPM/DDIM~\citep{bao2022estimating} (models the noise prediction residual), SN-DDPM/DDIM~\citep{bao2022estimating} (models the second moment of the noise), and OCM-DDPM/DDIM~\citep{ou2025improving}.
%SN-DDPM learns the covariance by training a neural network to estimate the second moment of the noise $\epsilon_t$
% NPR-DDPM,  models the noise prediction residual instead
%A-DDPM assumes a state-independent isotropic covariance with analytic form 
%
%\textbf{Protocol.} FID is computed from $50{,}000$ generated samples with EMA weights, likelihoods are reported as variational bounds in bits per dimension. Full evaluation details are in \cref{app:nll} and the experiment-details appendix.
%===========================================================================
\subsection{The FIM-Jacobian-Covariance Correspondence}
\label{sec:exp_fim}
%===========================================================================
\Cref{prop:fim} establishes that the Fisher Information of the denoising distribution, the denoiser Jacobian, and the posterior covariance are proportional (Eq.~\ref{eq:fim_jacobian}), and that the total Fisher Information equals the reconstruction MSE up to a schedule-dependent factor. These identities provide the theoretical foundation for treating the posterior covariance as a proxy for local reconstruction uncertainty. We verify them on held-out validation images. At three noise levels we compute and visualize the one-step Tweedie prediction $\hat{\mathbf{x}}_0$, the denoiser Jacobian diagonal $[\mathbf{J}_{\hat{\mathbf{x}}_0}]_{ii}$, the posterior covariance diagonal $[\mathrm{Cov}[\mathbf{x}_0\mid\mathbf{x}_t]]_{ii}$, and the per-pixel reconstruction accuracy, Fig.~\ref{fig:fim_correspondence}. The Jacobian and posterior covariance maps coincide up to a global scale at every noise level, as Eq.~\ref{eq:fim_jacobian} requires. Since this proportionality is exact under Tweedie's identity, the agreement also confirms that the Hessian estimators used downstream introduce no systematic bias. Reconstruction accuracy is lowest where these maps are largest,  which is the empirical content of the trace identity. As the noise decreases from $80\%$ to $20\%$, all three maps sharpen and concentrate on edges and contours, tracking the denoiser's increasing precision.
\begin{figure}[htbp]
  \centering
  \begin{minipage}[t]{0.52\textwidth}
    %\centering
    \includegraphics[width=0.8\linewidth]{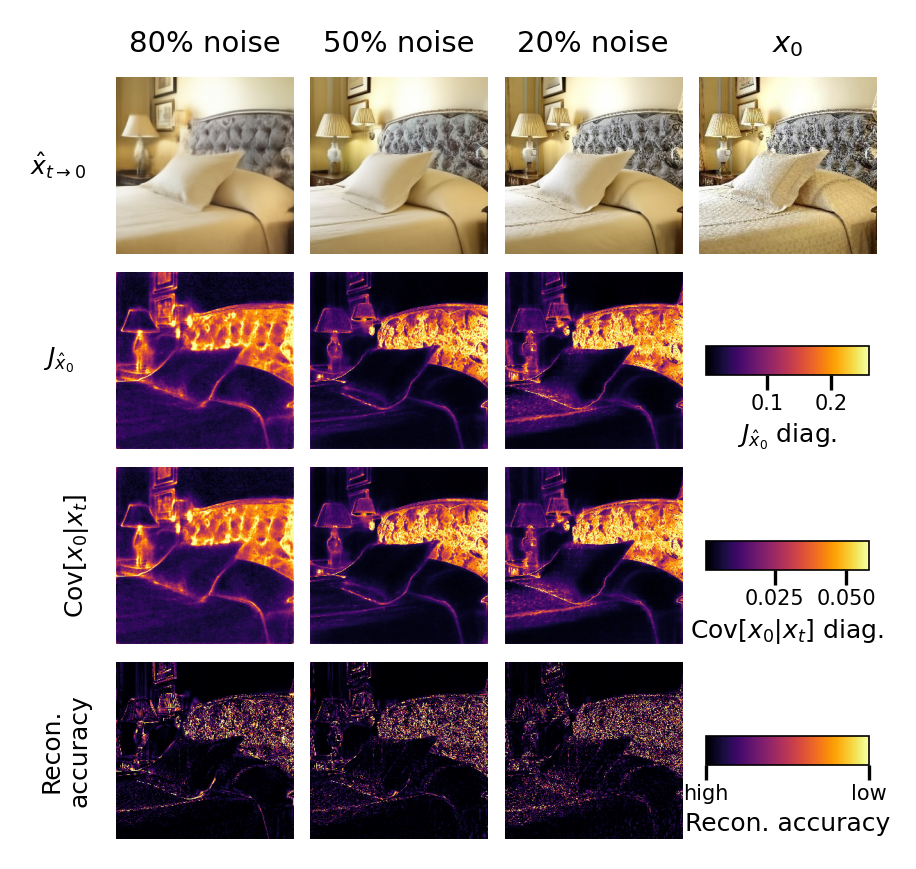}
\caption{Empirical validation of Prop.~\ref{prop:fim} on LSUN Bedroom. Columns:
    decreasing noise ($80\%$, $50\%$, $20\%$) and the clean image~$\mathbf{x}_0$. Rows:
    \textbf{(1)} one-step prediction $\hat{\mathbf{x}}_0$; \textbf{(2)} denoiser Jacobian
    diagonal $[\mathbf{J}_{\hat{\mathbf{x}}_0}]_{ii}$; \textbf{(3)} posterior covariance
    diagonal $[\mathrm{Cov}[\mathbf{x}_0\mid\mathbf{x}_t]]_{ii}$; \textbf{(4)} per-pixel
    reconstruction accuracy. Rows~2-3 are proportional (Eq.~\ref{eq:fim_jacobian}) and both
    are spatially anti-correlated with reconstruction accuracy.}
    \label{fig:fim_correspondence}
  \end{minipage}
  \hfill
  \begin{minipage}[t]{0.45\textwidth}
    %\centering
    \includegraphics[width=0.85\linewidth]{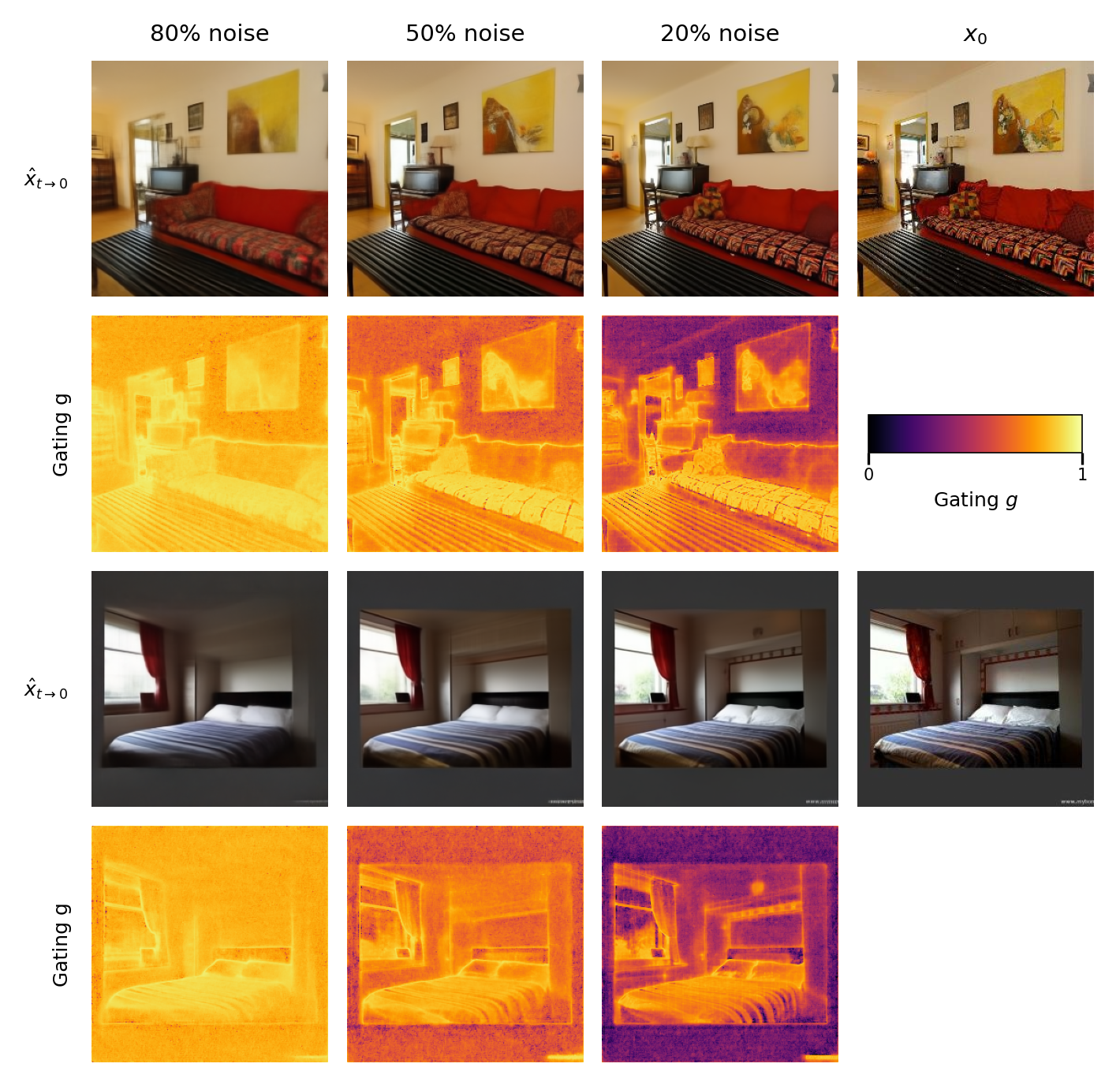}
    \caption{One-step predictions at three noise levels ($80\%$, $50\%$, $20\%$) and the
    corresponding gating map (dark$\,{=}\,0$, bright$\,{=}\,1$) on LSUN Bedroom. The gate
    assigns high stochasticity to edges and textures, near-zero to flat regions, with the
    separation sharpening as noise decreases.}
    \label{fig:gating_noise_levels}
    \end{minipage}
\end{figure}
%===========================================================================
\subsection{Gating Maps Align with Image Geometry}
\label{sec:exp_qualitative}
%===========================================================================
We now test whether the gate $g_{t,i}$ derived from it allocates stochasticity to
the geometrically complex regions of an image, through controlled-noise maps, correlation with
per-pixel error, and alignment with edges during sampling.\\
\textbf{Spatial allocation of stochasticity.}
Fig.~\ref{fig:gating_noise_levels} presents gating maps for LSUN Bedroom validation images at three noise levels. The gate assigns values near~$1$ to complex regions (furniture contours, painting edges) and near~$0$ to smooth regions (walls, uniform bed sheets). As noise decreases, the separation sharpens, outlining the image’s geometric structure and concentrating stochasticity at the most challenging features. This behavior results solely from the per-pixel gating function estimated by the pretrained denoising network.\\
\textbf{Correlation with reconstruction error.}
We compute the per-pixel gate $g_{t,i}$ and the validation loss $(x_{0,i}-\hat{x}_{0,i})^2$ on held-out images, measuring their log-space Pearson correlation at six noise levels (Fig.~\ref{fig:gating_vs_loss}). The correlation increases monotonically as noise decreases indicating that pixels with high-$g_{t,i}$ are reconstructed less accurately. At high noise, the lower correlation reflects that most gate values are compressed near $1$, reducing variation across pixels.\\
\begin{figure}[t]
    \centering
    \includegraphics[width=\linewidth]{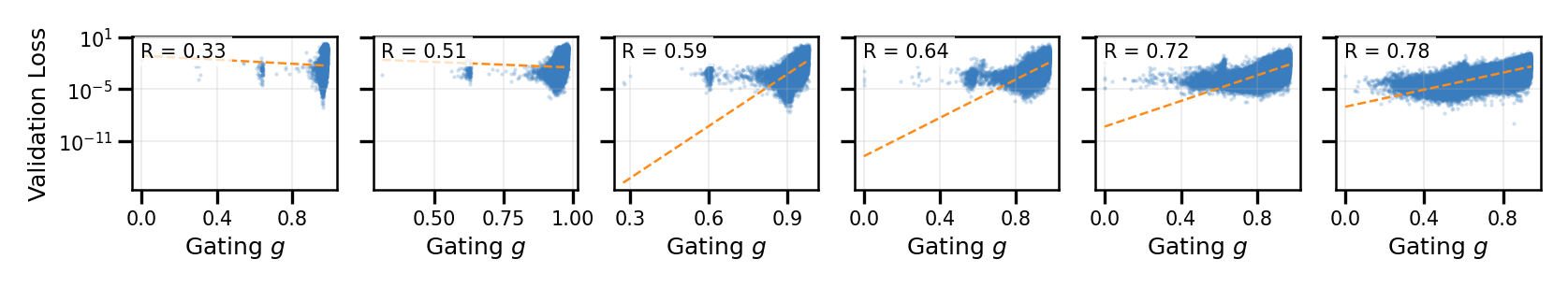}
    \caption{Per-pixel gating value $g_{t,i}$ vs.\ validation loss $(x_{0,i}-\hat{x}_{0,i})^2$ on LSUN Bedroom at six noise levels (left to right: decreasing noise). The log-space Pearson correlation rises monotonically from $0.33$ to $0.78$}% Each point is one pixel; the dashed line is a log-space linear fit. The Pearson correlation increases monotonically from $0.33$ to $0.78$, confirming that higher gating values are associated with higher reconstruction error.}
    \label{fig:gating_vs_loss}
\end{figure}
\textbf{Alignment with edge structure during sampling.}
The previous tests use controlled noise levels. Fig.~\ref{fig:gating_vs_edges_lsun} compares the gate to the Sobel gradient magnitude of $\hat{\mathbf{x}}_0$ along the reverse trajectory. At $t=800$, the model is globally uncertain and $g_{t,i}\approx1$ everywhere. As sampling progresses ($t=400$ to $100$), flat regions with $|\nabla\hat{\mathbf{x}}_0| \approx 0$ shift to $g_{t,i} \approx 0$, while edges and fine details retain high $g_{t,i}$. At $t=25$, $g_{t,i} \approx 0$ globally, with residual stochasticity only at the highest-gradient pixels. This map identifies structural difficulty without requiring explicit training, as other methods do~\citep{schusterbauer2026patchforcing}.
\begin{figure}[t]
    \centering
    \begin{subfigure}[b]{0.19\textwidth}
        \centering
        \includegraphics[width=\linewidth]{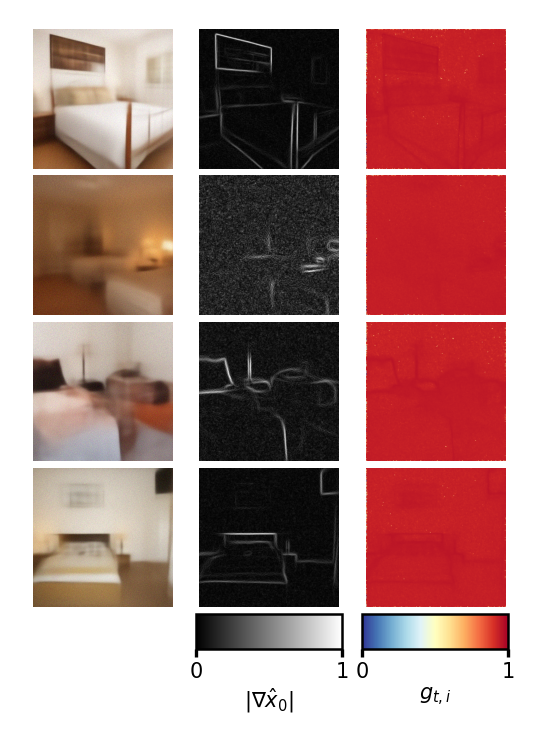}
        \caption{\footnotesize $t\!=\!800$}
        \label{fig:gating_vs_edges_lsun_800}
    \end{subfigure}
    \begin{subfigure}[b]{0.19\textwidth}
        \centering
        \includegraphics[width=\linewidth]{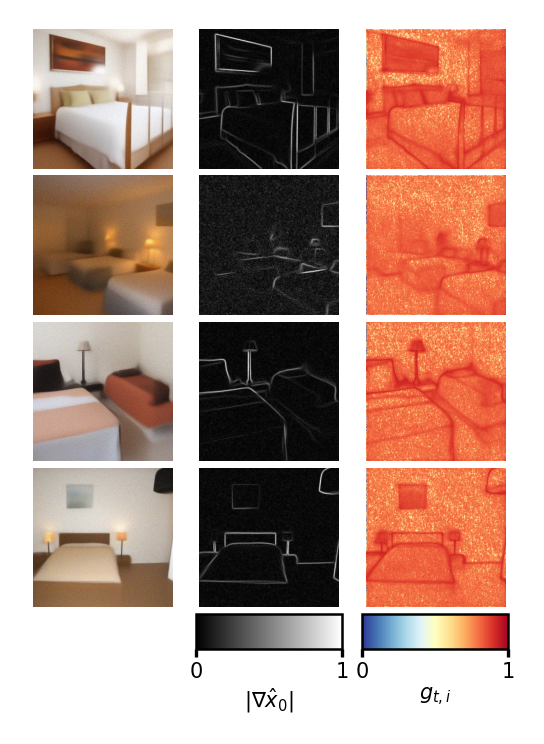}
        \caption{\footnotesize $t\!=\!400$}
        \label{fig:gating_vs_edges_lsun_400}
    \end{subfigure}
    \begin{subfigure}[b]{0.19\textwidth}
        \centering
        \includegraphics[width=\linewidth]{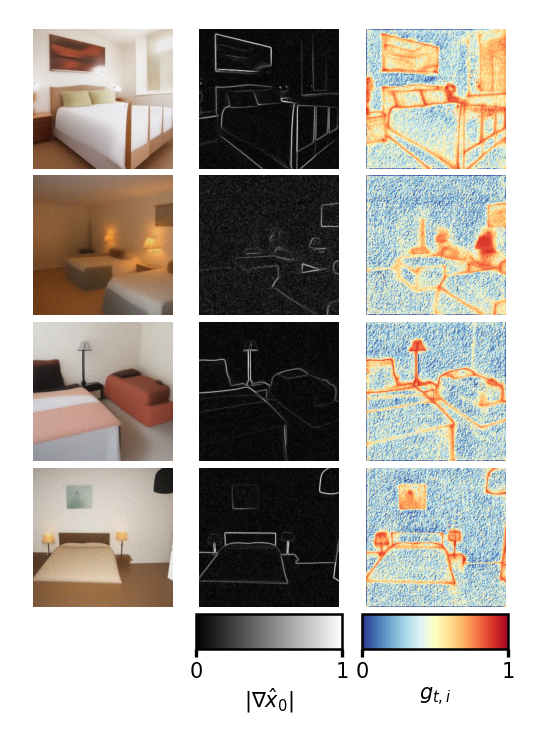}
        \caption{\footnotesize $t\!=\!200$}
        \label{fig:gating_vs_edges_lsun_200}
    \end{subfigure}
    \begin{subfigure}[b]{0.19\textwidth}
        \centering
        \includegraphics[width=\linewidth]{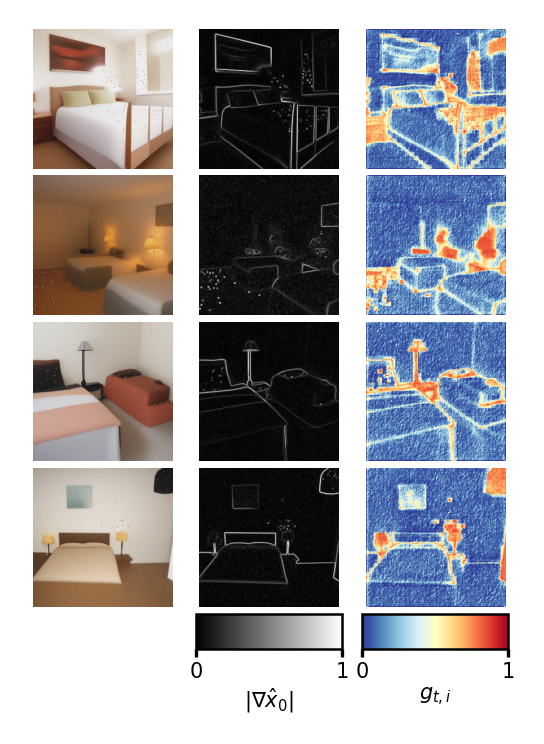}
        \caption{\footnotesize $t\!=\!100$}
        \label{fig:gating_vs_edges_lsun_100}
    \end{subfigure}
    \begin{subfigure}[b]{0.19\textwidth}
        \centering
        \includegraphics[width=\linewidth]{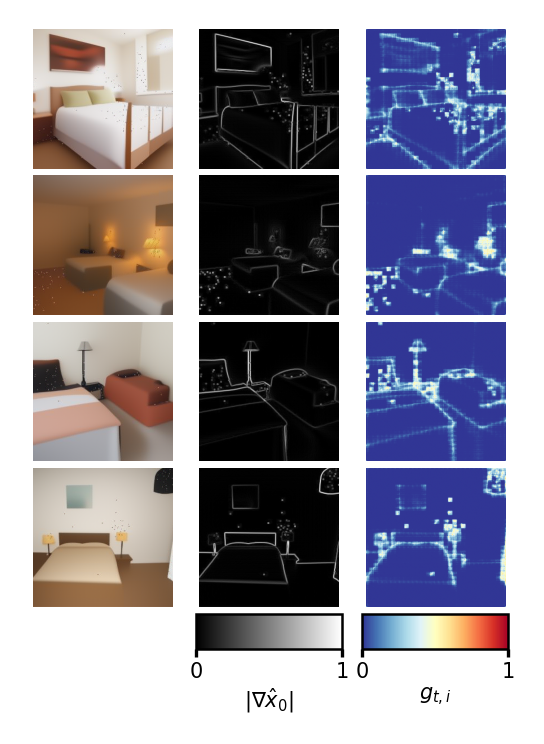}
        \caption{\footnotesize $t\!=\!25$}
        \label{fig:gating_vs_edges_lsun_25}
    \end{subfigure}
    \caption{Per-pixel gating $g_{t,i}$ vs.\ Sobel edge map of $\hat{\mathbf{x}}_0$ along a single reverse trajectory on LSUN Bedroom. Each subplot shows, for the same generated sample at the indicated timestep: the predicted image (left), the Sobel gradient magnitude $|\nabla\hat{\mathbf{x}}_0|$ (center), and the gate $g_{t,i}$ (right; blue~$=0$, red~$=1$). The gate is near $1$ in high noise, progressively aligns with edges and textures during denoising, and collapses to $\approx 0$ everywhere except for the highest-gradient pixels at the end of the trajectory.}
    \label{fig:gating_vs_edges_lsun}
\end{figure}
%===========================================================================
\subsection{Validating the Residual-Distribution Assumption}
\label{sec:exp_residual}
%===========================================================================
To validate our assumption that the per-pixel posterior residual
$R_i = X_{0,i} - \hat{x}_{0,i}\mid\mathbf{x}_t$ is Gaussian with variance~$v_i$, we compare the empirical distribution of the standardized residuals, $Z_i = R_i/\sqrt{v_i}$ to Gaussian and Laplace unit-variance references at four representative timesteps, Fig.~\ref{fig:residual_multidist}. The empirical standardized residuals are sharply peaked at high noise, resembling the unit-variance Laplace distribution, and converge toward $\mathcal{N}(0,1)$ as $t$ decreases. The Kolmogorov-Smirnov (KS) distance to the Gaussian drops from $0.23$ at $t{=}800$ to $0.03$ at $t{=}25$. This matches the high-SNR limit of Lemma~\ref{lem:posterior_factorization} in Appendices, where the Gaussian forward kernel dominates the posterior. Re-deriving the gate under Laplace residuals (App.~\cref{app:alt_gates}) preserves the spatial gating structure and only affects the sharpness of the stochastic-to-deterministic transition. Sampling results confirm that the Gaussian and Laplace gates (SANI$^{G}$ and SANI$^{L}$, \cref{tab:fid_all}) yield comparable FID.
\begin{figure}[t]
    \centering
    \includegraphics[width=0.9\linewidth]{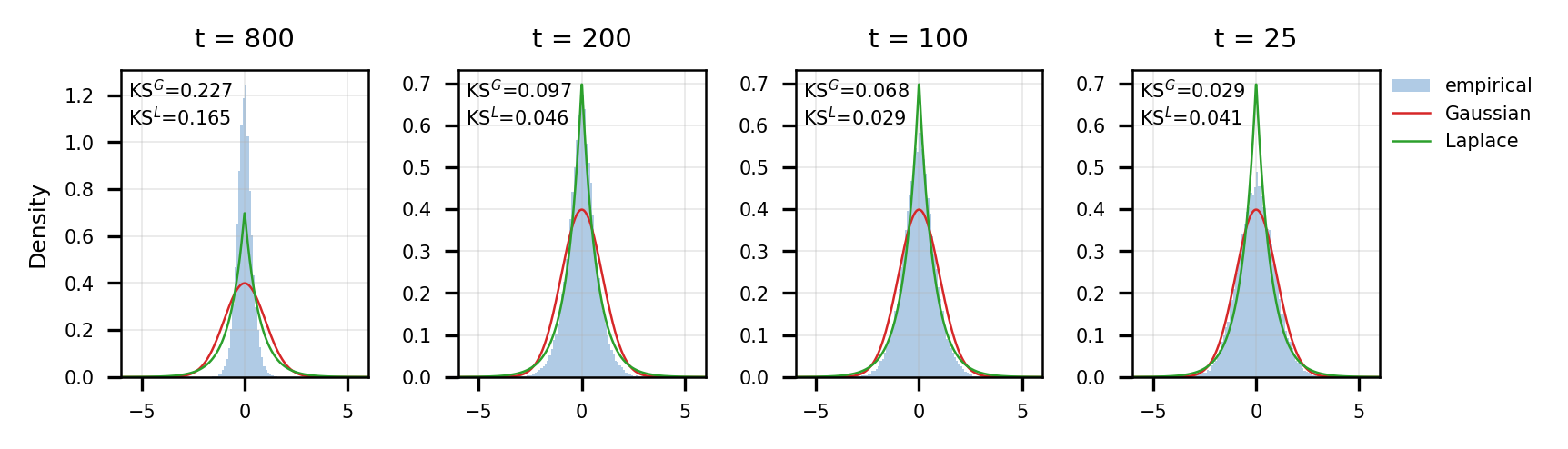}
    \caption{Density histograms of standardized residuals $Z_i=R_i/\sqrt{v_i}$ on CIFAR10 (CS) against Gaussian and Laplace unit-variance reference distributions at four reverse-process timesteps. The empirical density is sharply peaked at high noise and concentrates toward the Gaussian as $t$ decreases.}
    \label{fig:residual_multidist}
\end{figure}
%===========================================================================
\subsection{Behavior of the Gating Function}
\label{sec:exp_gating_behavior}
%===========================================================================
We analyze the gate along two axes of variation. The tolerance $\tau$, which determines its operating point, and reverse-process time $t$, which governs its evolution. Fig.~\ref{fig:gating_histograms} presents the spatial distribution of gating values across the full reverse trajectory for four tolerances, with density shown a logarithmic scale. The general pattern is consistent. A near-point mass at $t{=}1000$, broad support over $[0,1]$ mid-trajectory as the gate distinguishes between resolved and unresolved regions, and a collapse toward $0$ at the end. Increasing $\tau$ shifts the stochastic$\to$deterministic transition earlier, as expected, since the gate depends only on $\tau/v_i$. Fig.~\ref{fig:perstep_stats} summarizes the dynamics per-step mean and variance across datasets. The band is widest in the middle, reflecting peak spatial heterogeneity.
\begin{figure}[t]
  \begin{subfigure}[b]{0.23\textwidth}
    \centering
    \includegraphics[width=\linewidth]{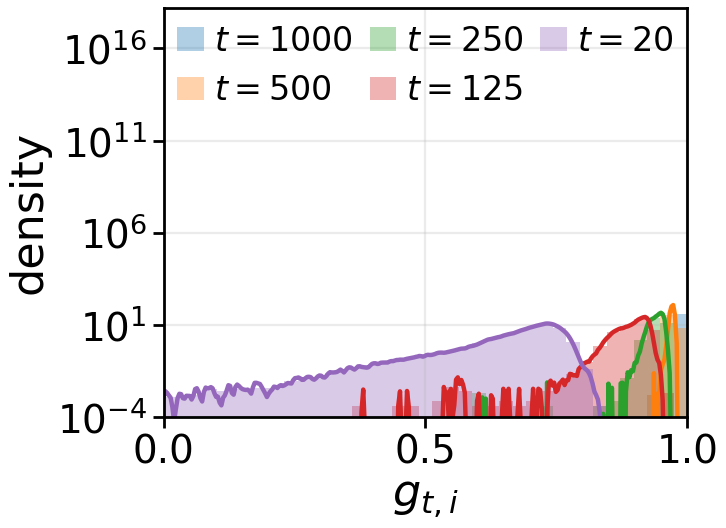}
    \caption{\footnotesize $\tau\!=\!10^{-4}$}
    \label{fig:gating_histograms_a}
  \end{subfigure}
  \begin{subfigure}[b]{0.23\textwidth}
    \centering
    \includegraphics[width=\linewidth]{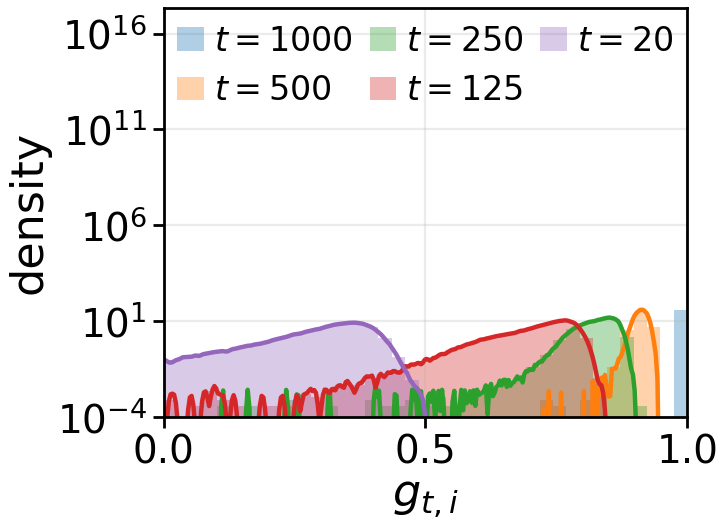}
    \caption{\footnotesize $\tau\!=\!10^{-3}$}
    \label{fig:gating_histograms_b}
  \end{subfigure}
  \begin{subfigure}[b]{0.23\textwidth}
    \centering
    \includegraphics[width=\linewidth]{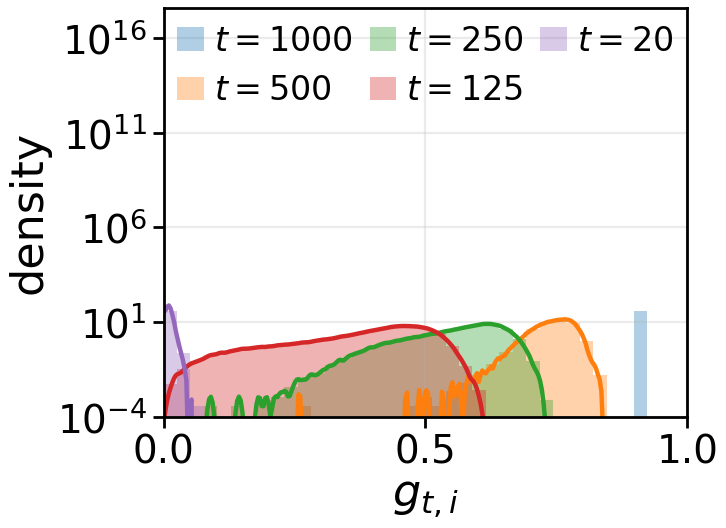}
    \caption{\footnotesize $\tau\!=\!10^{-2}$}
    \label{fig:gating_histograms_c}
  \end{subfigure}
  \begin{subfigure}[b]{0.23\textwidth}
    \centering
    \includegraphics[width=\linewidth]{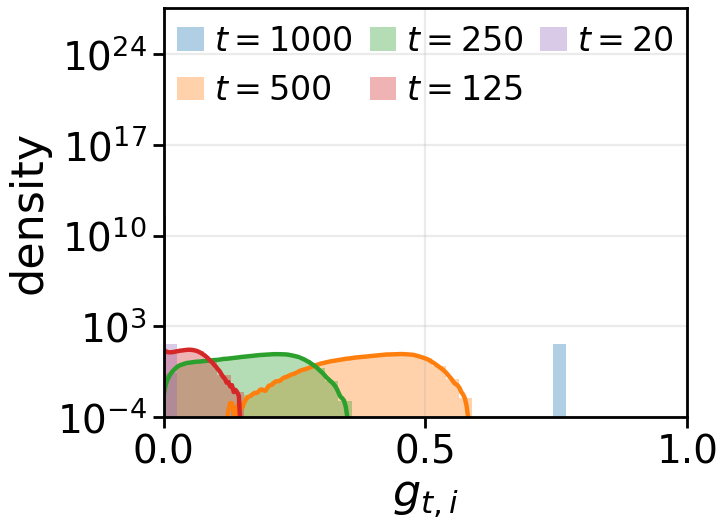}
    \caption{\footnotesize $\tau\!=\!10^{-1}$}
    \label{fig:gating_histograms_d}
  \end{subfigure}
  \caption{Spatial distribution of gating values $g_{t,i}$ along the full reverse trajectory on CIFAR-10 (CS), for four tolerances (subplot). Each subplot overlays five timesteps $t\in\{1000,500,250,125,20\}$, with density on a logarithmic scale. The distribution is a near-point mass at $t=1000$ (spatially uniform uncertainty under pure noise), spreads over $[0,1]$ at intermediate timesteps (spatial discrimination between confident and uncertain regions), and collapses toward $g=0$ at the final steps. Increasing $\tau$ by a decade translates this stochastic-to-deterministic transition earlier along the trajectory without changing its shape, as predicted by the dependence of the gate on $\tau/v_i$ alone.}%Corresponding histograms for the remaining datasets, exhibiting the same behavior, are in \cref{app:exp_crossdataset}}
  \label{fig:gating_histograms}
\end{figure}
\begin{figure}[t]
  \centering
    \includegraphics[width=0.80\linewidth]{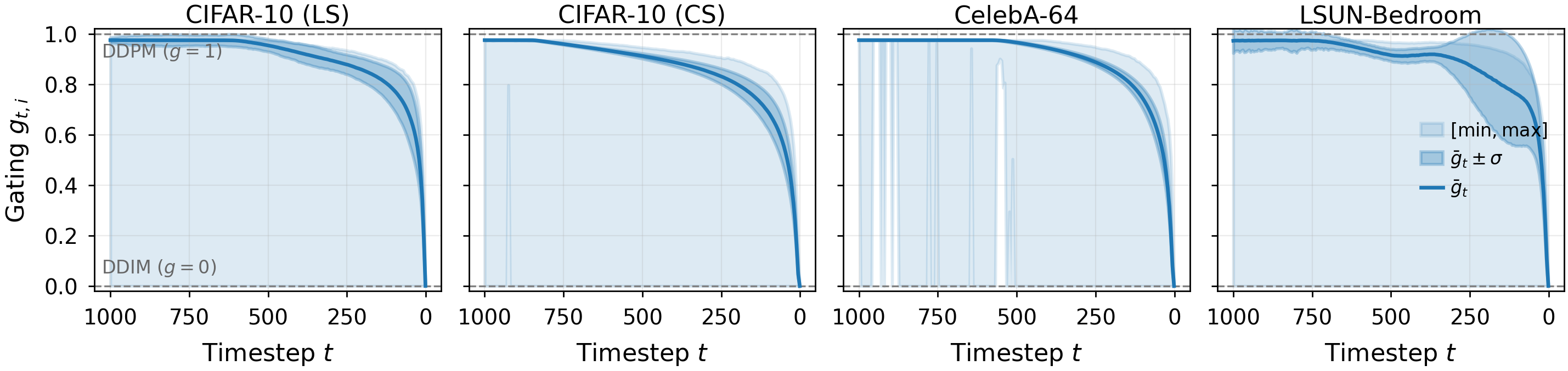}
  %\caption{Per-step gating statistics ($\bar{g}$, $g_{\min}$, $g_{\max}$) over the reverse process with $\tau = 10^{-3}$, averaged over a batch of generated images. Across datasets, $\bar{g}$ drops monotonically from~$1$ to near~$0$, confirming that SANI transitions from stochastic (early) to deterministic (late)  sampling as the model becomes confident.}
    \caption{Per-step gating statistics over the reverse process at $\tau=10^{-3}$, averaged over $1000$ generated images per dataset. \textbf{Solid line:} spatial mean $\bar{g}_t$; dark band: $\bar{g}_t\pm\sigma_{g,t}$ \textbf{Light band:} per-step $[\min,\max]$. The mean decreases monotonically from $1$ (stochastic) to $0$ (deterministic) on every dataset, and the band is widest at intermediate timesteps where spatial heterogeneity peaks.}
  \label{fig:perstep_stats}
\end{figure}
%===========================================================================
\subsection{Quantitative Evaluation}
\label{sec:exp_quantitative}
%===========================================================================
\Cref{tab:fid_all} reports FID scores across different sampling steps, likelihoods are reported in Appendices, \cref{app:nll}. We evaluate two gate variants: (i) SANI$^{G}$, which uses the Gaussian residual assumption, and (ii) SANI$^{L}$, which follows the Laplace assumption (App.~\cref{app:laplace}). Since SANI interpolates between stochastic and deterministic regimes on a per pixel basis, it does not fit into either family, so we compare it to both.

A primary test for per-pixel interpolation is whether it outperforms the two samplers it interpolates. SANI$^{G}$ consistently achieves lower FID than both vanilla DDPM and vanilla DDIM across all dataset and step-count configurations, often by a significant margin. So adjusting noise levels spatially is consistently as good as, and typically outperforms, to applying a uniform global regime. %On CelebA, SANI is the strongest stochastic sampler at every step count, improving on the best DDPM-family baseline (SN-DDPM) and surpassing the uniform analytic variance A-DDIM at every $K\ge25$. On CIFAR-10 (CS), SANI is strongest where it matters most for quality. It matches the best method in the table at $K{=}100$ and attains the best FID overall at $K{=}200$. In CIFAR-10 (LS), SANI is the best stochastic sampler in the low-step regime, while at $K\ge50$ the deterministic variance-learning baselines remain ahead.
%Against the stronger variance-learning baselines, SANI is most competitive in the low-step regime and attains the best stochastic-sampler FID in several configurations (e.g., CelebA at $K\in\{25,50,200\}$ and CIFAR-10 (CS) at $K\in\{100,200\}$), while SN- and OCM-based samplers retain the lead in several mid- and high-step CIFAR-10 settings (\cref{tab:fid_all}). We therefore position SANI as consistently better than the endpoint samplers it interpolates and competitive with, rather than uniformly dominant over, learned-variance methods.

The Gaussian and Laplace gates perform comparably across all three settings, with neither showing clear superiority. This confirms that sampling is insensitive to the residual model. SANI also remains competitive in NLL compared to variance-learning baselines (App.~\cref{app:nll}), indicating that spatial gating does not sacrifice likelihood for sample quality. Additionally, the per-pixel gating maps (Fig.~\ref{fig:gating_vs_edges_lsun}) offer interpretable uncertainty diagnostics.

\begin{table}[t]
\centering
\small \caption{FID ($\downarrow$) across datasets and sampling steps. \textbf{Bold}: best among stochastic samplers (DDPM family and SANI). \underline{Underline}: best in the DDIM family. SANI values are additionally underlined when they outperform the best DDIM-family sampler.}
\label{tab:fid_all}
\vskip -0.1in
\small
\setlength{\tabcolsep}{3.5pt}
%\begin{tabular}{lrrrrr|rrrrr}
\begin{tabular}{lrrrrrrrrrr}
\toprule
& \multicolumn{5}{c}{\textbf{DDPM family}} & \multicolumn{5}{c}{\textbf{DDIM family}} \\
\cmidrule(lr){2-6} \cmidrule(lr){7-11}
\textbf{Method} & 10 & 25 & 50 & 100 & 200 & 10 & 25 & 50 & 100 & 200\\
\midrule
\multicolumn{11}{l}{\textbf{CelebA}} \\
Base, $\tilde\beta$ & 36.69 & 24.46 & 18.96 & 14.31 & 10.48 & 20.54 & 13.54 & 9.33 & 6.60 & 4.96 \\
A-                  & 28.99 & 16.01 & 11.23 & 8.08  & 6.51  & 15.62 & 9.22  & 6.13 & 4.29 & 3.46 \\
NPR-                & 28.37 & 15.74 & 10.89 & 8.23  & 7.03  & 14.98 & 8.93  & 6.04 & 4.27 & 3.59 \\
SN-                 & 20.60 & 12.00 & 7.88  & 5.89  & 5.02  & \underline{10.20} & \underline{5.48} & \underline{3.83} & \underline{3.04} & \underline{2.85} \\
OCM-                & 21.55 & 12.71 & 9.24  & 6.97  & 5.92  & 10.28 & 5.72  & 4.42 & 3.54 & 3.17 \\
%\midrule
$\mathrm{SANI}^G$   & \multicolumn{10}{c}{K=10: 17.09 \, K=25: \textbf{8.87} \, K=50: \textbf{5.56} \, K=100: 4.64 \, K=200: \textbf{3.04}} \\
$\mathrm{SANI}^L$   & \multicolumn{10}{c}{K=10: \textbf{16.07} \, K=25: 9.84 \, K=50: 6.25 \, K=100: \textbf{3.93} \, K=200: 4.52} \\ 
\midrule
\multicolumn{11}{l}{\textbf{CIFAR-10~(LS)}} \\
Base, $\tilde\beta$ & 44.45 & 21.83 & 15.21 & 10.94 & 8.23  & 21.31 & 10.70 & 7.74 & 6.08 & 5.07 \\
A-                  & 34.26 & 11.60 & 7.25  & 5.40  & 4.01  & 14.00 & 5.81  & 4.04 & 3.55 & 3.39 \\
NPR-                & 32.35 & 10.55 & 6.18  & 4.52  & 3.57  & 13.34 & 5.38  & 3.95 & 3.53 & 3.42 \\
SN-                 & 24.06 &  6.91 & \textbf{4.63} & \textbf{3.67} & \textbf{3.11} & 12.19 & \underline{4.28} & \underline{3.39} & \underline{3.23} & \underline{3.22} \\
OCM-                & 24.94 &  9.19 & 5.95  & 4.36  & 3.48  & \underline{10.66} & 4.35  & 3.48 & 3.27 & 3.29 \\
%\midrule
$\mathrm{SANI}^G$   & \multicolumn{10}{c}{K=10: \textbf{15.27} \, K=25: \textbf{6.64} \, K=50: 6.50 \, K=100: 3.82 \, K=200: 3.99} \\
$\mathrm{SANI}^L$  & \multicolumn{10}{c}{K=10: 17.96 \, K=25: 8.13 \, K=50: 5.65 \, K=100: 4.85 \, K=200: 5.85} \\
\midrule
\multicolumn{11}{l}{\textbf{CIFAR-10~(CS)}} \\
Base, $\tilde\beta$ & 34.76 & 16.18 & 11.11 & 8.38  & 6.66  & 34.34 & 16.68 & 10.48 & 7.94 & 6.69 \\
A-                  & 22.94 &  8.50 & 5.50  & 4.45  & 4.04  & 26.43 & 9.96  & 6.02  & 4.88 & 4.92 \\
NPR-                & 19.94 &  7.99 & 5.31  & 4.52  & 4.10  & 22.81 & 9.47  & 6.04  & 5.02 & 5.06 \\
SN-                 & 16.33 &  6.05 & \textbf{4.17} & \textbf{3.83} & 3.72 & 17.90 & 7.36  & 5.16  & 4.63 & 4.63 \\
OCM-                & \textbf{14.32} & \textbf{5.54} & 4.10 & 3.84 & 3.75 & {16.70} & {6.71} & \underline{4.72} & {4.30} & {4.54} \\
\midrule
$\mathrm{SANI}^G$   & \multicolumn{10}{c}{K=10: 17.01 \, K=25: 7.35 \, K=50: 4.95 \, K=100: \textbf{\underline{3.83}} \, K=200: 4.17} \\
$\mathrm{SANI}^L$   & \multicolumn{10}{c}{K=10: \underline{16.37} \, K=25: \underline{6.55} \, K=50: 4.88 \, K=100: 4.50 \, K=200: \textbf{\underline{3.52}}} \\
\bottomrule
\end{tabular}
%\vspace{-4mm}
\end{table}
\subsection{Ablation: Spatial Allocation and Coupling}
\label{sec:exp_ablation}
%% ====================================================================
To isolate the contributions of the probabilistic gating mechanism and the spatially adaptive variance formulation, we perform a decoupling ablation on CelebA-64 (\cref{tab:ablation}). We compare standard endpoint samplers, ungated spatial variances (with a fixed gate), and various decoupled gated combinations. The results show that combining optimal per-pixel variance with the probabilistic gating mechanism (SANI) consistently achieves the best FID across all evaluated step counts, confirming that both components are essential for optimal performance.
\begin{table}[h]
\centering
\caption{Decoupling ablation on CelebA-64. FID ($\downarrow$). Gated variants report the best FID across a $\tau$ sweep at each step count, endpoint and ungated variants are $\tau$-independent.}
\label{tab:ablation}
\vskip -0.1in
\small
\setlength{\tabcolsep}{1.5pt}
\begin{tabular}{lrrrrr}
\toprule
\textbf{Variant/ $\#$Timesteps} & $10$ & $25$ & $50$ & $100$ & $200$ \\
\midrule
\multicolumn{6}{l}{\textbf{Endpoint samplers}} \\
DDIM                            & 20.54 & 13.54 &  9.33 &  6.60 &  4.96 \\
DDPM ($\tilde\beta_t$)          & 36.69 & 24.46 & 18.96 & 14.31 & 10.48 \\
\midrule
\multicolumn{6}{l}{\textbf{Ungated variance choices}} \\
OCM-DDPM ($\gamma^*$ on noise only) & 21.55 & 12.71 &  9.24 &  6.97 &  5.92 \\
Uniform $\gamma^*$ (both terms) & 34.13 & 18.97 & 13.98 & 10.66 &  9.70 \\
\midrule
\multicolumn{6}{l}{\textbf{Gated variants}} \\
$g_{t,i},\gamma_{t,i}^*$-coupled                    & 24.74 &  8.81 &  7.88 &  4.74 & 3.20 \\
SANI-naive ($g_{t,i},\tilde\beta_t$)                  & 22.53 & \textbf{7.32} &  8.21 & \textbf{4.12} & 3.34 \\
SANI-decoupled                  & 22.09 & 13.57 & 10.40 &  8.70 & 8.88 \\
\textbf{SANI (proposed)}        & \textbf{17.09} & 8.87 & \textbf{5.56} & 4.64 & \textbf{3.04} \\
\bottomrule
\end{tabular}
%\vspace{-4mm}
\end{table}
%% ====================================================================
\section{Related Work}
\label{sec:related}
%% ====================================================================
\textbf{Diffusion models and sampling.}
\cite{ho20, song2021scorebased} established the foundational framework of diffusion as a hierarchical latent variable model.
\cite{song2021ddim} revealed the probability flow ODE connection, enabling deterministic sampling.
\cite{Karras_Elucidating} refined sampling via schedule design and higher-order solvers.
Our work extends this line of research by introducing a spatial degree of freedom that interpolates between deterministic and stochastic dynamics at the pixel level.\\
\textbf{Variance learning.}
Improved DDPM~\citep{nichol2021improved} first proposed learning $\Sigma_\theta$ via $\mathcal{L}_{\mathrm{vlb}}$.
Analytic-DPM~\citep{bao2022analytic} derived the optimal scalar variance.
\cite{bao2022estimating} explored diagonal covariances.
OCM~\citep{ou2025improving} estimates diagonal covariance by matching the score Hessian.
However, OCM applies learned variances uniformly, all pixels follow the same regime at each step.
SANI employs Hessian-diagonal estimates for a distinct purpose: deriving per-pixel gating that enables spatially heterogeneous interpolation within a single image.\\
\textbf{Adaptive sampling.}
Dpm-Solver~\citep{lu2022dpmsolver} and Dpm-Solver++~\citep{ludpm++} adapt step spacing temporally but not spatially. \cite{liu2025regionadaptive} skip computations in easy regions for efficiency, but do not modulate the stochastic properties of the process. MULAN~\citep{sahoo2024mulan} learns a multivariate noise schedule during training. In contrast, SANI derives noise modulation from the pre-trained score network at inference time.\\
\textbf{Uncertainty in diffusion models.}
\cite{De_Vita_2025_WACV} estimate pixel-wise uncertainty via ensembles, incurring multiple forward passes.
\cite{song2024high_frequency} show that high-frequency regions benefit from stochastic exploration in wavelet-based compression. \cite{jazbec25aUncertainty} use posterior variance for OOD detection. Closest to our use of posterior moments, \cite{manor2024posterior} compute per-pixel posterior moments from derivatives of a pretrained denoiser for uncertainty quantification and posterior exploration in image restoration.
The proposed approach derives the same per-pixel uncertainty from the FIM diagonal, but uses it to control the sampler itself, gating the noise injection at inference time.\\
\textbf{Information geometry.}
\cite{amari} established links between Fisher Information and model sensitivity.
\cite{vincent2011connection} connected denoising autoencoders to score estimation.
Tweedie's formula~\citep{efron2011tweedie} relates scores to posterior means, and \cite{meng2021estimating} derived uncertainty bounds via the score Hessian.
\cite{DarasSurvey} showed denoiser Jacobian norms correlate with sample diversity. This work builds on these identities by interpreting them through the Fisher Information of the denoising distribution and by deriving a practical sampling algorithm from this relationship.\\
%Tweedie's formula~\citep{efron2011tweedie} relates scores to posterior means, \cite{dytso2020general} establish a general derivative identity linking derivatives of the conditional mean to posterior moments in Gaussian noise, and \cite{meng2021estimating} derived uncertainty bounds via the score Hessian. \cite{DarasSurvey} showed denoiser Jacobian norms correlate with sample diversity. This work builds on these identities by interpreting them through the Fisher Information of the denoising distribution and by deriving a practical sampling algorithm from this relationship.\\
%
\textbf{Spatially Adaptive Sampling.} A recognized limitation of standard diffusion models is the uniform allocation of computational resources and stochasticity across spatially heterogeneous images. Recent work, such as Patch Forcing~\citep{schusterbauer2026patchforcing}, addresses this  limitation by dynamically varying the diffusion timestep per patch, advancing ``easy'' patches faster than ``hard'' patches. However, such methods typically require architectural modifications, such as training a dedicated difficulty head, and operate at the coarser patch level. In contrast, SANI addresses spatial heterogeneity at inference time without requiring model fine-tuning or structural modifications. %By leveraging the information geometry of the pretrained score field, SANI modulates stochasticity at the pixel level without requiring any model fine-tuning or structural modifications.
%% ====================================================================
\section{Conclusion}
\label{sec:conclusion}
%% ====================================================================
In this work, we introduced Spatially Adaptive Noise Injection (SANI), a training-free inference framework designed to address the spatial uniformity bottleneck in standard diffusion samplers. Our approach, grounded in information geometry, establishes a formal theoretical link between the Fisher Information Matrix of the denoising distribution, the denoiser Jacobian, and the posterior covariance. We translated this geometric insight into a tractable, closed-form probabilistic gating function that quantifies the probability the local denoising error exceeds a specific tolerance, enabling SANI to modulate noise injection per-pixel. By coupling this gating mechanism with a spatially adaptive variance, SANI allocates stochastic corrections to uncertain, high-curvature regions while preserving deterministic DDIM-style updates in confident, smooth areas. Empirically, SANI improves generation quality over established DDIM and DDPM samplers across various step counts and is competitive with variance-learning DDPM-family baselines, with the largest gains in the low-step regime. Our decoupling analysis further confirms that the synergy between probabilistic gating and adaptive variance is essential for the method's performance, especially in low-step regime where standard samplers are less effective.

\paragraph{\textbf{Limitations and Future Work.}} While SANI establishes a foundation for spatially adaptive sampling, our results reveal an asymmetric impact of spatial gating across sampling steps. At low step counts, gating overcomes the imprecision of the underlying model. However, as the number of sampling steps increase, the coupled update rule becomes more sensitive to the diagonal Hessian approximation. In addition, our quantitative evaluation is restricted to small, low-resolution benchmarks (CIFAR-10 at $32\times32$ and CelebA at $64\times64$, with qualitative results on LSUN Bedroom at $256\times256$), and the spatially adaptive transition carries no marginal-preservation guarantee. Future work should explore dynamic tolerance scheduling, adjusting the threshold $\tau$ as a function of the sampling steps to improve the transition from high-stochasticity exploration to high-fidelity deterministic refinement, as well as validation on high-resolution and latent-space diffusion models.

\section*{Acknowledgements}
We acknowledge financial support from the Swiss National Science Foundation through the LegoMol project (grant no. 207428). The computations were performed on the Baobab and Yggdrasil cluster at the University of Geneva.
% - Bibliography -
\bibliographystyle{splncs04}
\bibliography{main}

\begin{thebibliography}{10}
\providecommand{\url}[1]{\texttt{#1}}
\providecommand{\urlprefix}{URL }
\providecommand{\doi}[1]{https://doi.org/#1}

\bibitem{amari}
Amari, S.i.: Natural gradient works efficiently in learning. Neural Computation
   \textbf{10}(2),  251--276 (1998). \doi{10.1162/089976698300017746}

\bibitem{bao2022estimating}
Bao, F., Li, C., Sun, J., Zhu, J., Zhang, B.: Estimating the optimal covariance
  with imperfect mean in diffusion probabilistic models. In: Chaudhuri, K.,
  Jegelka, S., Song, L., Szepesvari, C., Niu, G., Sabato, S. (eds.) Proceedings
  of the 39th International Conference on Machine Learning. Proceedings of
  Machine Learning Research, vol.~162, pp. 1555--1584. PMLR (17--23 Jul 2022),
  \url{https://proceedings.mlr.press/v162/bao22d.html}

\bibitem{bao2022analytic}
Bao, F., Li, C., Zhu, J., Zhang, B.: Analytic-dpm: an analytic estimate of the
  optimal reverse variance in diffusion probabilistic models. In: International
  Conference on Learning Representations (ICLR) (2022),
  \url{https://openreview.net/forum?id=ly5M8mE4v49}

\bibitem{DarasSurvey}
Daras, G., Chung, H., Lai, C., Mitsufuji, Y., Ye, J.C., Milanfar, P., Dimakis,
  A.G., Delbracio, M.: A survey on diffusion models for inverse problems. CoRR
  \textbf{abs/2410.00083} (2024). \doi{10.48550/ARXIV.2410.00083},
  \url{https://doi.org/10.48550/arXiv.2410.00083}

\bibitem{De_Vita_2025_WACV}
De~Vita, M., Belagiannis, V.: Diffusion model guided sampling with pixel-wise
  aleatoric uncertainty estimation. In: Proceedings of the Winter Conference on
  Applications of Computer Vision (WACV). pp. 3844--3854 (February 2025)

\bibitem{dytso2020general}
Dytso, A., Poor, H.V., Shitz, S.S.: A general derivative identity for the
  conditional mean estimator in gaussian noise and some applications. In: 2020
  IEEE International Symposium on Information Theory (ISIT). pp. 1183--1188.
  IEEE (2020)

\bibitem{efron2011tweedie}
Efron, B.: Tweedie’s formula and selection bias. Journal of the American
  Statistical Association  \textbf{106}(496),  1602--1614 (2011)

\bibitem{ho20}
Ho, J., Jain, A., Abbeel, P.: Denoising diffusion probabilistic models. In:
  Larochelle, H., Ranzato, M., Hadsell, R., Balcan, M., Lin, H. (eds.) Advances
  in Neural Information Processing Systems. vol.~33, pp. 6840--6851. Curran
  Associates, Inc. (2020),
  \url{https://proceedings.neurips.cc/paper_files/paper/2020/file/4c5bcfec8584af0d967f1ab10179ca4b-Paper.pdf}

\bibitem{jazbec25aUncertainty}
Jazbec, M., Wong-Toi, E., Xia, G., Zhang, D., Nalisnick, E., Mandt, S.:
  Generative uncertainty in diffusion models. In: Chiappa, S., Magliacane, S.
  (eds.) Proceedings of the Forty-first Conference on Uncertainty in Artificial
  Intelligence. Proceedings of Machine Learning Research, vol.~286, pp.
  1837--1858. PMLR (21--25 Jul 2025),
  \url{https://proceedings.mlr.press/v286/jazbec25a.html}

\bibitem{Karras_Elucidating}
Karras, T., Aittala, M., Aila, T., Laine, S.: Elucidating the design space of
  diffusion-based generative models. In: Koyejo, S., Mohamed, S., Agarwal, A.,
  Belgrave, D., Cho, K., Oh, A. (eds.) Advances in Neural Information
  Processing Systems. vol.~35, pp. 26565--26577. Curran Associates, Inc. (2022)

\bibitem{Kingma21_vdm}
Kingma, D., Salimans, T., Poole, B., Ho, J.: Variational diffusion models. In:
  Ranzato, M., Beygelzimer, A., Dauphin, Y., Liang, P., Vaughan, J.W. (eds.)
  Advances in Neural Information Processing Systems. vol.~34, pp. 21696--21707.
  Curran Associates, Inc. (2021),
  \url{https://proceedings.neurips.cc/paper_files/paper/2021/file/b578f2a52a0229873fefc2a4b06377fa-Paper.pdf}

\bibitem{liu2025regionadaptive}
Liu, Z., Yang, Y., Zhang, C., Zhang, Y., Qiu, L., You, Y., Yang, Y.:
  Region-adaptive sampling for diffusion transformers (2025),
  \url{https://openreview.net/forum?id=gCZYei3PLL}

\bibitem{lu2022dpmsolver}
Lu, C., Zhou, Y., Bao, F., Chen, J., Li, C., Zhu, J.: Dpm-solver: A fast ode
  solver for diffusion probabilistic model sampling in around 10 steps. In:
  Advances in Neural Information Processing Systems (2022)

\bibitem{ludpm++}
Lu, C., Zhou, Y., Bao, F., Chen, J., Li, C., Zhu, J.: Dpm-solver++: Fast solver
  for guided sampling of diffusion probabilistic models. Machine Intelligence
  Research  \textbf{22}(4),  730–751 (Jun 2025).
  \doi{10.1007/s11633-025-1562-4},
  \url{http://dx.doi.org/10.1007/s11633-025-1562-4}

\bibitem{manor2024posterior}
Manor, H., Michaeli, T.: On the posterior distribution in denoising:
  Application to uncertainty quantification. In: International Conference on
  Learning Representations. vol.~2024, pp. 49233--49263 (2024)

\bibitem{meng2021estimating}
Meng, C., Song, Y., Li, W., Ermon, S.: Estimating high order gradients of the
  data distribution by denoising. In: Advances in Neural Information Processing
  Systems. vol.~34, pp. 25364--25377 (2021)

\bibitem{nichol2021improved}
Nichol, A.Q., Dhariwal, P.: Improved denoising diffusion probabilistic models
  (2021), \url{https://openreview.net/forum?id=-NEXDKk8gZ}

\bibitem{ou2025improving}
Ou, Z., Zhang, M., Zhang, A., Xiao, T.Z., Li, Y., Barber, D.: Improving
  probabilistic diffusion models with optimal diagonal covariance matching. In:
  The Thirteenth International Conference on Learning Representations (2025),
  \url{https://openreview.net/forum?id=fV0t65OBUu}

\bibitem{peebles2023scalable}
Peebles, W., Xie, S.: Scalable diffusion models with transformers. In:
  Proceedings of the IEEE/CVF International Conference on Computer Vision. pp.
  4195--4205 (2023)

\bibitem{sahoo2024mulan}
Sahoo, S.S., Gokaslan, A., De, C., Kuleshov, V.: Diffusion models with learned
  adaptive noise. In: Globerson, A., Mackey, L., Belgrave, D., Fan, A., Paquet,
  U., Tomczak, J., Zhang, C. (eds.) Advances in Neural Information Processing
  Systems. vol.~37, pp. 105730--105779. Curran Associates, Inc. (2024).
  \doi{10.52202/079017-3354}

\bibitem{schusterbauer2026patchforcing}
Schusterbauer, J., Gui, M., Li, Y., Ma, P., Krause, F., Ommer, B.: Denoising,
  fast and slow: Difficulty-aware adaptive sampling for image generation. In:
  Proceedings of the IEEE/CVF Conference on Computer Vision and Pattern
  Recognition (2026)

\bibitem{song2021ddim}
Song, J., Meng, C., Ermon, S.: Denoising diffusion implicit models. In:
  International Conference on Learning Representations (2021),
  \url{https://openreview.net/forum?id=St1giarCHLP}

\bibitem{song2024high_frequency}
Song, J., He, J., Yang, L., Feng, M., Wang, K.: High frequency matters:
  Uncertainty guided image compression with wavelet diffusion. arXiv preprint
  arXiv:2407.12538  (2024)

\bibitem{song2021scorebased}
Song, Y., Sohl-Dickstein, J., Kingma, D.P., Kumar, A., Ermon, S., Poole, B.:
  Score-based generative modeling through stochastic differential equations.
  In: International Conference on Learning Representations (2021),
  \url{https://openreview.net/forum?id=PxTIG12RRHS}

\bibitem{vincent2011connection}
Vincent, P.: A connection between score matching and denoising autoencoders.
  Neural Computation  \textbf{23}(7),  1661--1674 (2011)

\end{thebibliography}

% ===========================================================================
%  APPENDIX
% 
% ===========================================================================
% =====================================================================
\newpage

%%%%%%%%%%%%%%%%%%%%%%%%%%%%%%%%%%%%%%%%%%%%%%%%%%%%%%%%%%%%

\appendix
\nolinenumbers
\pagenumbering{arabic}
\setcounter{page}{1}
\setcounter{figure}{0}
\setcounter{table}{0}
\setcounter{equation}{0}
\setcounter{proposition}{0}
\renewcommand{\thefigure}{A\arabic{figure}}
\renewcommand{\thetable}{A\arabic{table}}
\renewcommand{\theequation}{A\arabic{equation}}
\renewcommand{\theproposition}{A\arabic{proposition}}

% ----------------------------------------------------------------------------
\begin{large}
    \textbf{Appendix for ``Spatially Adaptive Noise Injection''}
\end{large}

% \etocdepthtag.toc{mtappendix}
% \etocsettagdepth{mtchapter}{none}
% \etocsettagdepth{mtappendix}{subsection}
% \etocsettocstyle{}{}% optional: drop the automatic "Table of Contents" heading
% {\small \tableofcontents}

\section{Notation}
\label{app:notation}
% ----------------------------------------------------------------------------
Let $\mathbf{x}=(x^1,\ldots,x^D)^\top\in\mathbb{R}^D$ denote a point in $D$-dimensional Euclidean space (represented as a column vector), and let $\mathrm{Tr}(\mathbf{A})=\sum_i A_{ii}$ denote the trace of $\mathbf{A}\in\mathbb{R}^{k\times k}$.  For a scalar function $f:\mathbb{R}^D\to\mathbb{R}$, $\nabla_{\mathbf{x}}f(\mathbf{x})\in\mathbb{R}^D$ is the gradient and $\nabla^2_{\mathbf{x}}f(\mathbf{x})\in\mathbb{R}^{D\times D}$ is the Hessian. For a vector-valued function $f:\mathbb{R}^k\to\mathbb{R}^m$, $\partial f/\partial\mathbf{x}\in\mathbb{R}^{m\times k}$ is the Jacobian. Throughout we write $\bar\alpha_t=\prod_{s=1}^t\alpha_s$, $\tilde\beta_t=(1-\bar\alpha_{t-1})\beta_t/(1-\bar\alpha_t)$, and $c_t=\sqrt{\bar\alpha_{t-1}}\,\beta_t/(1-\bar\alpha_t)$.
% ----------------------------------------------------------------------------
\section{Recovering DDPM from the Generalized Update}
\label{app:ddpm_recovery}
% ----------------------------------------------------------------------------

\begin{proposition}
Setting $\sigma_t^2=\frac{1-\bar\alpha_{t-1}}{1-\bar\alpha_t}\bigl(1-\frac{\bar\alpha_t}{\bar\alpha_{t-1}}\bigr)$ in the generalized update Eq.~\ref{eq:general_update} yields $\sigma_t^2=\tilde\beta_t$.
\end{proposition}
\begin{proof}
From $\bar\alpha_t=\bar\alpha_{t-1}\,\alpha_t$, we have $1-\bar\alpha_t/\bar\alpha_{t-1}=1-\alpha_t=\beta_t$, and substitution gives $\sigma_t^2=\frac{1-\bar\alpha_{t-1}}{1-\bar\alpha_t}\beta_t=\tilde\beta_t$.
\end{proof}
\noindent\textit{Remark.}
The DDIM paper \citep{song2021ddim} uses $\alpha_t$ to denote the cumulative product, which corresponds to $\bar\alpha_t$ in this work, with the convention $\alpha_0:=1$.  Eq.~\ref{eq:general_update} is presented to ensure consistency with the main text.
%
% ----------------------------------------------------------------------------
\section{Exponential-Family Structure of the Denoising Distribution}
\label{app:exp_family}
% ----------------------------------------------------------------------------
%
\begin{proposition}[Exponential Family]
\label{prop:exp_family}
The denoising distribution $p(\mathbf{x}_0\mid\mathbf{x}_t)$ belongs to an exponential family with canonical form 
$$p(\mathbf{x}_0\mid\mathbf{x}_t)=h(\mathbf{x}_0)\exp\bigl(\boldsymbol{\eta}(\mathbf{x}_t,t)^\top\mathbf{T}(\mathbf{x}_0)-\psi(\mathbf{x}_t,t)\bigr),$$ 
where
\begin{align*}
  h(\mathbf{x}_0)&=q(\mathbf{x}_0),\quad
  \boldsymbol{\eta}(\mathbf{x}_t,t)=\Bigl(\tfrac{\sqrt{\bar\alpha_t}}{1-\bar\alpha_t}\mathbf{x}_t,\;-\tfrac{\bar\alpha_t}{2(1-\bar\alpha_t)}\Bigr),\\
  \mathbf{T}(\mathbf{x}_0)&=(\mathbf{x}_0,\,\|\mathbf{x}_0\|^2),\quad
  \psi(\mathbf{x}_t,t)=\log p_t(\mathbf{x}_t)+\tfrac{D}{2}\log(2\pi(1-\bar\alpha_t))+\tfrac{\|\mathbf{x}_t\|^2}{2(1-\bar\alpha_t)}.
\end{align*}
\end{proposition}

\begin{proof}
By Bayes' rule $p(\mathbf{x}_0\mid\mathbf{x}_t)=p(\mathbf{x}_t\mid\mathbf{x}_0)q(\mathbf{x}_0)/p_t(\mathbf{x}_t)$, with the Gaussian forward kernel $p(\mathbf{x}_t\mid\mathbf{x}_0)=\mathcal{N}(\sqrt{\bar\alpha_t}\,\mathbf{x}_0,(1-\bar\alpha_t)\mathbf{I})$,  $ q $ the data distribution, and and $ p_t(\mathbf{x}_t) = \int p(\mathbf{x}_t|\mathbf{x}_0)q(\mathbf{x}_0)d\mathbf{x}_0 $ is the marginal distribution at time $ t $. Expanding and collecting terms in~$\mathbf{x}_0$:
\begin{align*}
p(\mathbf{x}_t\mid\mathbf{x}_0)
=\exp\Bigl(-\tfrac{D}{2}\log(2\pi(1-\bar\alpha_t))-\tfrac{\|\mathbf{x}_t\|^2}{2(1-\bar\alpha_t)}\Bigr)
\exp\Bigl(-\tfrac{\bar\alpha_t}{2(1-\bar\alpha_t)}\|\mathbf{x}_0\|^2+\tfrac{\sqrt{\bar\alpha_t}}{1-\bar\alpha_t}\mathbf{x}_t^\top\mathbf{x}_0\Bigr).
\end{align*}
Substituting and identifying with the canonical form gives the stated parameters.
\end{proof}

% ----------------------------------------------------------------------------
\section{Proof of \cref{prop:fim}: FIM-Jacobian-posterior Covariance Correspondence}
\label{app:fim}
% ----------------------------------------------------------------------------

\begin{proof}
For an exponential family with natural parameters $\boldsymbol{\eta}$, the FIM with respect to $\boldsymbol{\eta}$ equals the covariance of the sufficient statistics: $\mathcal{I}(\boldsymbol{\eta})=\mathrm{Cov}[\mathbf{T}(\mathbf{x}_0)\mid\mathbf{x}_t]$. Since $\mathbf{x}_t$ parameterizes the distribution via $\boldsymbol{\eta}(\mathbf{x}_t)$, the change-of-variables formula gives
\begin{equation}\label{eq:app_fim_chain}
  \mathcal{I}(\mathbf{x}_t)=\bigl(\nabla_{\mathbf{x}_t}\boldsymbol{\eta}\bigr)^\top\mathcal{I}(\boldsymbol{\eta})\bigl(\nabla_{\mathbf{x}_t}\boldsymbol{\eta}\bigr).
\end{equation}
From \cref{prop:exp_family}, the Jacobian of the first component of~$\boldsymbol{\eta}$ with respect to~$\mathbf{x}_t$ is $\frac{\sqrt{\bar\alpha_t}}{1-\bar\alpha_t}\mathbf{I}_D$ and the second component is constant in~$\mathbf{x}_t$. The relevant block of $\mathcal{I}(\boldsymbol{\eta})$ is $\mathrm{Cov}[\mathbf{x}_0\mid\mathbf{x}_t]$. Substituting into~\eqref{eq:app_fim_chain}:
\begin{equation}\label{eq:app_fim_cov}
  \mathcal{I}(\mathbf{x}_t)=\frac{\bar\alpha_t}{(1-\bar\alpha_t)^2}\,\mathrm{Cov}[\mathbf{x}_0\mid\mathbf{x}_t].
\end{equation}
By Tweedie's identity~\citep{efron2011tweedie}, $\mathrm{Cov}[\mathbf{x}_0\mid\mathbf{x}_t]=\frac{1-\bar\alpha_t}{\sqrt{\bar\alpha_t}}\,\mathbf{J}_{\hat{\mathbf{x}}_0}(\mathbf{x}_t)$. Substituting into~\eqref{eq:app_fim_cov} yields the result.
\end{proof}

\paragraph{Trace identity.} Taking the trace of \eqref{eq:app_fim_cov} yields
\begin{equation}\label{eq:app_trace_mse}
  \mathrm{Tr}\bigl(\mathcal{I}(\mathbf{x}_t)\bigr)
  =\frac{\bar\alpha_t}{(1-\bar\alpha_t)^2}\,\mathrm{MSE}(\mathbf{x}_t),
\end{equation}
where $\mathrm{MSE}(\mathbf{x}_t)\coloneqq\mathbb{E}[\|\mathbf{x}_0-\hat{\mathbf{x}}_0\|^2\mid\mathbf{x}_t]=\mathrm{Tr}(\mathrm{Cov}[\mathbf{x}_0\mid\mathbf{x}_t])$ is the optimal reconstruction error. Eq.~\ref{eq:app_trace_mse} certifies that the total Fisher information is proportional to the MMSE, so regions of large $\mathrm{Tr}(\mathcal{I})$ are exactly the regions where the model is uncertain about the underlying signal.

\paragraph{Relation to known identities.} Combining Eq.~\eqref{eq:app_fim_cov} with Tweedie's identity recovers, in the diffusion parameterization, the general derivative identity for the conditional mean estimator in Gaussian noise of Dytso\etal~\cite{dytso2020general}, which relates derivatives of the posterior mean to posterior moments and Manor\etal~\cite{manor2024posterior} exploit the same relation to compute per-pixel posterior moments from a pretrained denoiser for uncertainty quantification. Proposition~\ref{prop:fim} adds the information-geometric interpretation of this quantity as the Fisher Information of the denoising distribution.
% ----------------------------------------------------------------------------
\section{PSD Condition and Self-Limiting Violations}
\label{app:psd}
% ----------------------------------------------------------------------------

\begin{assumption}[PSD Condition]
\label{ass:psd}
For the true log-marginal Hessian $\mathbf{H}_t^\star(\mathbf{x}_t)=\nabla^2_{\mathbf{x}_t}\log p_t(\mathbf{x}_t)$, all eigenvalues satisfy $\lambda_{\min}(\mathbf{H}_t^\star)\geq -(1-\bar\alpha_t)^{-1}$.
\end{assumption}

Since $\mathrm{Cov}[\mathbf{x}_0\mid\mathbf{x}_t]$ is PSD by definition and the prefactor $(1-\bar\alpha_t)/\bar\alpha_t$ in Eq.~\ref{eq:cov_hessian} is positive, $\mathbf{I}_D+(1-\bar\alpha_t)\mathbf{H}_t^\star$ must be PSD, which is equivalent to $\lambda_{\min}(\mathbf{H}_t^\star)\geq-(1-\bar\alpha_t)^{-1}$. Assumption~\ref{ass:psd} is therefore not an additional restriction on the data distribution; it is a structural consequence of the forward process. Its role is to make explicit the spectral condition that the true Hessian automatically satisfies, so that we can diagnose when a learned approximation $\mathbf{H}_t\approx\mathbf{H}_t^\star$ may violate it.

\begin{proposition}[Self-Limiting PSD Violations]
\label{prop:self_limit}
Let $\delta_i\coloneqq[\mathbf{H}_t]_{ii}-[\mathbf{H}_t^\star]_{ii}$ denote the per-pixel approximation error. A PSD violation at pixel~$i$ (i.e., $v_i<0$ in Eq.~\ref{eq:v_i}) requires $[\mathbf{H}_t^\star]_{ii}+\delta_i<-(1-\bar\alpha_t)^{-1}$. For bounded error $|\delta_i|\leq\Delta$, the true per-pixel variance at any violating pixel satisfies
\begin{equation}\label{eq:app_self_limit}
  v_i^\star=\frac{1-\bar\alpha_t}{\bar\alpha_t}\bigl(1+(1-\bar\alpha_t)[\mathbf{H}_t^\star]_{ii}\bigr)
  \;\leq\;\frac{(1-\bar\alpha_t)^2}{\bar\alpha_t}\,\Delta.
\end{equation}
\end{proposition}

\begin{proof}
A PSD violation means $1+(1-\bar\alpha_t)[\mathbf{H}_t]_{ii}<0$. Writing $[\mathbf{H}_t]_{ii}=[\mathbf{H}_t^\star]_{ii}+\delta_i$:
\[
  1+(1-\bar\alpha_t)[\mathbf{H}_t^\star]_{ii}<-(1-\bar\alpha_t)\delta_i\leq(1-\bar\alpha_t)|\delta_i|\leq(1-\bar\alpha_t)\Delta,
\]
since a violation requires $\delta_i$ negative with $|\delta_i|$ large. Multiplying by the positive prefactor yields~\eqref{eq:app_self_limit}.
\end{proof}

\paragraph{Interpretation.} At late timesteps ($1-\bar\alpha_t\ll 1$), the bound~\eqref{eq:app_self_limit} is extremely small: PSD violations can only arise where the true variance is already negligible. At early timesteps, the bound is loose, but the gate saturates at $g_{t,i}\approx 1$ for all pixels regardless, so the precise value of~$v_i$ is irrelevant. In both regimes, clamping $v_i\gets\max(v_i,\epsilon)$ correctly defaults violating pixels to near-deterministic dynamics. Empirically, the clamp is active for $<0.1\%$ of pixels at any timestep when using the Hessian network of Ou~\etal~\cite{ou2025improving}.

% ----------------------------------------------------------------------------
\section{Gating Function: Derivation and Justification of Assumptions}
\label{app:gating}
% ----------------------------------------------------------------------------

The following section restates the assumptions, provides justification, and proves the closed-form equation for the gating function (Eq.~\ref{prop:gating}).

\subsection{Assumptions}

\noindent\textbf{(A1) MMSE Optimality.} $\hat{x}_{0,i}\approx\mathbb{E}[X_{0,i}\mid\mathbf{x}_t]$, justified by the squared-error training objective $\mathcal{L}_{\mathrm{simple}}$, whose Bayes-optimal solution is the MMSE estimator,  $\hat{x}_{0,i} \approx \mathbb{E}[X_{0,i} \mid \mathbf{x}_t]$.

\noindent\textbf{(A2) Gaussian Marginal Residual.} $R_i\mid\mathbf{x}_t\approx\mathcal{N}(0,v_i)$ with $R_i\coloneqq X_{0,i}-\hat{x}_{0,i}$.

The denoiser network determines only the first two moments of~$R_i$, so~(A2) is a modeling choice (specifically, the maximum-entropy distribution given those moments) rather than a consequence of the model. The justification proceeds in two steps.

\subsection{Residual Moments}
\label{app:residual_moments}

\begin{lemma}[Residual Moments]
\label{lem:residual_moments}
Under~(A1), $\mathbb{E}[R_i\mid\mathbf{x}_t]=0$ and $\mathrm{Var}(R_i\mid\mathbf{x}_t)=v_i(\mathbf{x}_t,t)$.
\end{lemma}

\begin{proof}
Since $\hat{x}_{0,i}=\mathbb{E}[X_{0,i}\mid\mathbf{x}_t]$ is a deterministic function of~$\mathbf{x}_t$,
$\mathbb{E}[R_i\mid\mathbf{x}_t]=\mathbb{E}[X_{0,i}\mid\mathbf{x}_t]-\hat{x}_{0,i}=0$, and
$\mathrm{Var}(R_i\mid\mathbf{x}_t)=\mathrm{Var}(X_{0,i}\mid\mathbf{x}_t)=v_i$.
\end{proof}

Lemma~\ref{lem:residual_moments} fixes only the first two moments of~$R_i$ but infinitely many distributions are consistent with them. Step~2 justifies the Gaussian shape as a limit.

\subsection{Posterior Concentration as a Gaussian}
\label{app:gaussian_concentration}

\begin{lemma}[Posterior as Prior $\times$ Gaussian Kernel]
\label{lem:posterior_factorization}
The denoising posterior factorizes as
\begin{equation}\label{eq:app_posterior_factorization}
  p(\mathbf{x}_0\mid\mathbf{x}_t)
  =\frac{1}{Z(\mathbf{x}_t)}\,q(\mathbf{x}_0)\,\exp\Bigl(-\tfrac{\lambda_t}{2}\|\mathbf{x}_0-\boldsymbol{\mu}_t\|^2\Bigr),
\end{equation}
where $\lambda_t\coloneqq\bar\alpha_t/(1-\bar\alpha_t)$ is the precision, $\boldsymbol{\mu}_t\coloneqq\mathbf{x}_t/\sqrt{\bar\alpha_t}$, and $Z(\mathbf{x}_t)$ normalizes.
\end{lemma}

\begin{proof}
By Bayes' rule $p(\mathbf{x}_0\mid\mathbf{x}_t)\propto q(\mathbf{x}_0)\,p(\mathbf{x}_t\mid\mathbf{x}_0)$, with $p(\mathbf{x}_t\mid\mathbf{x}_0)\propto\exp\!\bigl(-\|\mathbf{x}_t-\sqrt{\bar\alpha_t}\mathbf{x}_0\|^2/(2(1-\bar\alpha_t))\bigr)$. Completing the square in~$\mathbf{x}_0$:
\[
\frac{\|\mathbf{x}_t-\sqrt{\bar\alpha_t}\mathbf{x}_0\|^2}{2(1-\bar\alpha_t)}
=\frac{\bar\alpha_t}{2(1-\bar\alpha_t)}\|\mathbf{x}_0-\mathbf{x}_t/\sqrt{\bar\alpha_t}\|^2+C(\mathbf{x}_t),
\]
where $C(\mathbf{x}_t)$ is absorbed into~$Z(\mathbf{x}_t)$.
\end{proof}

\paragraph{Why this justifies~(A2).}
The decomposition~\eqref{eq:app_posterior_factorization} holds exactly for any data distribution~$q$. The Gaussian kernel arises from the forward process and is not an approximation. Assumption~(A2) states that this kernel dominates the prior in shaping the per-pixel marginal, a claim controlled by~$\lambda_t$. For large~$\lambda_t$ (late reverse timesteps), the kernel becomes high concentrated near~$\boldsymbol{\mu}_t$, allowing the smooth prior to be locally well-approximated by a constant and the per-pixel marginal then inherits the Gaussian shape of the kernel. For small~$\lambda_t$ (early timesteps), the prior dominates and the marginal may be strongly non-Gaussian. However, in this regime $v_i\gg\tau$ for all pixels and $g_{t,i}\approx 1$ universally (\cref{fig:perstep_stats}), making the gate insensitive to the distributional shape of~$R_i$. Thus, the gate relies on~(A2) only in the intermediate regime where it actively discriminates. Even in this case, \cref{sec:exp_residual} shows that the spatial gating pattern remains consistent across different residual distributions.

\subsection{Proof of \Cref{prop:gating}(Closed-Form Gating)}

\begin{proof}
Under~(A1), $\hat{x}_{0,i}=\mathbb{E}[X_{0,i}\mid\mathbf{x}_t]$, so the squared error equals $R_i^2$ and
\begin{equation}\label{eq:app_gating_step1}
  \mathbb{P}\!\bigl((X_{0,i}-\hat{x}_{0,i})^2>\tau\mid\mathbf{x}_t\bigr)
  =\mathbb{P}\!\bigl(|R_i|>\sqrt{\tau}\mid\mathbf{x}_t\bigr).
\end{equation}
Under~(A2) and \Cref{lem:residual_moments}, $R_i\mid\mathbf{x}_t\sim\mathcal{N}(0,v_i)$. The clamp $v_i\gets\max(v_i,\epsilon)$ guarantees $\sqrt{v_i}>0$, so we may define $Z\coloneqq R_i/\sqrt{v_i}\sim\mathcal{N}(0,1)$. Then $|R_i|>\sqrt{\tau}\iff|Z|>\sqrt{\tau/v_i}$, and by symmetry of the standard normal,
\begin{equation}\label{eq:app_two_tail}
  \mathbb{P}(|Z|>a)=2(1-\Phi(a))\quad\text{for any }a\geq 0.
\end{equation}
Combining~\eqref{eq:app_gating_step1} and~\eqref{eq:app_two_tail} with $a=\sqrt{\tau/v_i}$ yields Eq.~\ref{eq:gating_closed}.
\end{proof}

% ----------------------------------------------------------------------------
\section{Gating Under Alternative Residual Distributions}
\label{app:alt_gates}
% ----------------------------------------------------------------------------

Using the same gating function definition $g_{t,i} \;=\; \mathbb{P}\!\bigl(R_i^2 > \tau \mid \mathbf{x}_t\bigr) \;=\; \mathbb{P}\!\bigl(|R_i| > \sqrt{\tau} \mid \mathbf{x}_t\bigr)$, we re-derive the gate for an alternative residual model, the Laplace distribution parameterized to have zero mean and variance~$v_i$, matching \Cref{lem:residual_moments}. The Laplace distribution has sharper peak and heavier tails than the Gaussian.

\subsection{Laplace Residual}
\label{app:laplace}
\begin{proposition}[Laplace Gating]
\label{prop:gating_laplace}
If $R_i\mid\mathbf{x}_t\sim\mathrm{Laplace}(0,b_i)$ with $b_i=\sqrt{v_i/2}$ (matching $\mathrm{Var}(R_i)=v_i$), then
\begin{equation}\label{eq:app_gating_laplace}
  g_{t,i}^{\mathrm{Lap}}=\exp\!\bigl(-\sqrt{2\tau/v_i}\bigr).
\end{equation}
\end{proposition}

\begin{proof}
The Laplace survival function is $\mathbb{P}(|R_i|>a)=\exp(-a/b_i)$ for $a\geq 0$. Setting $a=\sqrt{\tau}$ and substituting $b_i=\sqrt{v_i/2}$ yields Eq.~\ref{eq:app_gating_laplace}.
\end{proof}

\subsection{Comparison}

Both gates exhibit the same boundary behavior ($g\to 0$ as $v_i/\tau\to 0$, $g\to 1$ as $v_i/\tau\to\infty$) and differ only in their tail decay, as shown in~\Cref{tab:app_gates}. Heavier-tailed distributions assign more probability to large denoising errors, so for a given $\tau/v_i$ the ordering at large $\rho$ is $g^{\mathrm{Gauss}}<g^{\mathrm{Lap}}$ (for small~$\nu$). At moderate~$\rho$, the ordering is non-monotone. The Laplace distribution's sharper peak concentrates mass near zero, reducing the probability in the tails. For $\rho\lesssim 3$, the Gaussian gate is higher than the Laplace, with a crossover point where the heavier tails of the Laplace dominate.

\begin{table}[h]
\centering
\small
\caption{Gating functions under alternative residual models. $\rho\coloneqq\tau/v_i$, $\Phi$ is the standard normal CDF, $I_x$ is the regularized incomplete beta function.}
\label{tab:app_gates}
\begin{tabular}{lll}
\toprule
Distribution & Gating function $g_{t,i}$ & Tail decay ($\rho\to\infty$) \\
\midrule
Gaussian & $2(1-\Phi(\sqrt{\rho}))$ & $\sim e^{-\rho/2}/\sqrt{\rho}$ \\
Laplace  & $\exp(-\sqrt{2\rho})$ & $\sim e^{-\sqrt{2\rho}}$ \\
%Student-$t_\nu$ & $I_{(\nu-2)/((\nu-2)+\rho)}(\nu/2,1/2)$ & $\sim\rho^{-\nu/2}$ \\
\bottomrule
\end{tabular}
\end{table}

%The Gaussian gate (used by default) is closed-form in elementary terms and the Laplace gate is similarly cheap. 
Empirical results (\cref{sec:exp_residual}, \cref{tab:fid_all}) show that $\mathrm{SANI}^G$ and $\mathrm{SANI}^L$ achieve similar FID scores, confirming that the spatial gating structure is consistent across these choices.
 
% ----------------------------------------------------------------------------
\section{Proof of \Cref{prop:kl_opt} (KL-Optimal Per-Pixel Variance)}
\label{app:kl_derivation}
% ----------------------------------------------------------------------------

\paragraph{Setup.} The true reverse posterior at pixel~$i$, conditioned on~$\mathbf{x}_t$ and~$x_{0,i}$, is given by~\citep{ho20}
\begin{equation}\label{eq:app_true_reverse}
  q(x_{t-1,i}\mid\mathbf{x}_t,x_{0,i})=\mathcal{N}\!\bigl(\tilde\mu_{t,i}(\mathbf{x}_t,x_{0,i}),\,\tilde\beta_t\bigr),
\end{equation}
with $\tilde\mu_{t,i}(\mathbf{x}_t,x_0)=\frac{\sqrt{\bar\alpha_{t-1}}\beta_t}{1-\bar\alpha_t}x_{0,i}+\frac{\sqrt{\alpha_t}(1-\bar\alpha_{t-1})}{1-\bar\alpha_t}x_{t,i}$. The model transition is $p_\theta(x_{t-1,i}\mid\mathbf{x}_t)=\mathcal{N}(\hat\mu_{t,i},(\sigma_{t,i}^\theta)^2)$, where $\hat\mu_{t,i}$ is obtained by substituting $\hat{x}_{0,i}$ for $x_{0,i}$ in $\tilde\mu_{t,i}$.

\paragraph{Step 1: Mean-error decomposition.} Writing $x_{0,i}=\hat{x}_{0,i}+e_i$ with $e_i\coloneqq x_{0,i}-\hat{x}_{0,i}$:
\begin{equation}\label{eq:app_kl_mean_decomp}
  \tilde\mu_{t,i}=\hat\mu_{t,i}+c_t\,e_i,\qquad c_t=\tfrac{\sqrt{\bar\alpha_{t-1}}\beta_t}{1-\bar\alpha_t}.
\end{equation}
Under~(A1), $\mathbb{E}[e_i\mid\mathbf{x}_t]=0$ and $\mathbb{E}[e_i^2\mid\mathbf{x}_t]=v_i$. Therefore,
\begin{equation}\label{eq:app_kl_mean_sq}
  \mathbb{E}\!\bigl[(\tilde\mu_{t,i}-\hat\mu_{t,i})^2\mid\mathbf{x}_t\bigr]=c_t^2 v_i.
\end{equation}

\paragraph{Step 2: Expected KL.} For two univariate Gaussian distributions $\mathcal{N}(\mu_q,\sigma_q^2)$ and $\mathcal{N}(\mu_p,\sigma_p^2)$, the Kullback-Leibler divergence is given by $\mathrm{KL}=\log(\sigma_p/\sigma_q)+(\sigma_q^2+(\mu_q-\mu_p)^2)/(2\sigma_p^2)-1/2$. by setting $\mu_q=\tilde\mu_{t,i}$, $\sigma_q^2=\tilde\beta_t$, $\mu_p=\hat\mu_{t,i}$, $\sigma_p^2=(\sigma_{t,i}^\theta)^2$ and taking the expectation over $x_{0,i}\mid\mathbf{x}_t$, using Eq.~\ref{eq:app_kl_mean_sq}:
\begin{equation}\label{eq:app_expected_kl}
  \mathbb{E}[\mathrm{KL}]
  =\log\frac{\sigma_{t,i}^\theta}{\sqrt{\tilde\beta_t}}+\frac{\tilde\beta_t+c_t^2 v_i}{2(\sigma_{t,i}^\theta)^2}-\tfrac{1}{2}.
\end{equation}

\paragraph{Step 3: Optimization.} Differentiating Eq.~\ref{eq:app_expected_kl} with respect to $(\sigma_{t,i}^\theta)^2$:
\[
  \frac{\partial}{\partial(\sigma_{t,i}^\theta)^2}\mathbb{E}[\mathrm{KL}]
  =\frac{1}{2(\sigma_{t,i}^\theta)^2}-\frac{\tilde\beta_t+c_t^2 v_i}{2(\sigma_{t,i}^\theta)^4}.
\]
Setting the derivative to zero yields $(\sigma_{t,i}^\theta)^2_{\mathrm{opt}}=\tilde\beta_t+c_t^2 v_i=\gamma^*_{t,i}$, which is Eq.~\ref{eq:gamma_opt}. The second derivative at the optimum is $1/(2\gamma_{t,i}^{*4})>0$, confirming that this point is a minimum.

\paragraph{Consistency with prior work.} Averaging Eq.~\ref{eq:gamma_opt} over pixels and applying the score-Hessian relation recovers the scalar optimal variance of Bao \etal\cite{bao2022analytic}, up to the diagonal approximation. This derivation extends their result to a per-pixel quantity, which is the form required for spatially adaptive gating.

\paragraph{Recovery of standard samplers.} This update recovers DDIM in the joint limit $g_{t,i}=0$, $v_i=0$ (confident pixels), and the KL-optimal stochastic sampler when $g_{t,i}=1$ ($\Sigma_{t,i}^2=\gamma^*_{t,i}$). It also recovers DDPM when $g_{t,i}=1$ and $v_i=0$ (perfect denoiser, $\Sigma_{t,i}^2=\tilde\beta_t$). A spatially uniform gate $g_{t,i}=\eta$ implements the $\eta$-interpolation of~\cite{song2021ddim}, but with a per-pixel, data-aware variance instead of the fixed~$\tilde\beta_t$.

\paragraph{Scope of optimality.} As discussed in \cref{sec:method:variance}, the optimality of $\gamma^*_{t,i}$ concerns the fully stochastic per-pixel transition with the mean fixed at the Tweedie estimate. The deployed SANI variance $\Sigma_{t,i}^2=c_t^2 v_i+g_{t,i}\tilde\beta_t$ equals $\gamma^*_{t,i}$ only when $g_{t,i}=1$, and the admissibility clip of \cref{app:admissibility} may reduce it further. SANI as a whole therefore carries no global KL-optimality guarantee.

\paragraph{Relation to the DDIM family.}
For a scalar, state-independent noise scale, the non-Markovian family of~\cite{song2021ddim} preserves the marginals of the forward process. The SANI variance $\Sigma_{t,i}^2$ is per-pixel and depends on $\mathbf{x}_t$ through $v_i$ and $g_{t,i}$, so this marginal-preservation argument does not carry over, and we do not claim that the spatially adaptive transition matches the forward marginals. SANI should therefore be understood as a per-pixel update rule whose endpoints coincide with the DDIM and DDPM updates, supported empirically in \cref{sec:experiments}
%
% ----------------------------------------------------------------------------
\section{Mean-Error Identity and the Asymmetric Allocation}
\label{app:coupled}
% ----------------------------------------------------------------------------
%
The mean-error identity~(Eq.~\ref{eq:mean_err}) stated in the main text is proved here. This result underpins the rationale for applying the calibration term $c_t^2 v_i$ in $\Sigma_{t,i}^2$ uniformly, while only $\tilde\beta_t$ is gated.
\begin{proposition}[Mean-Error Identity]
\label{prop:mean_error}
Let $\hat\mu_{t,i}$ denote the model's posterior mean obtained by substituting the Tweedie estimate~$\hat{x}_{0,i}$ for~$x_{0,i}$ in the true posterior mean, and let $\tilde\mu_{t,i}$ denote the true posterior mean. Under~(A1),
\begin{equation}\label{eq:app_mean_err}
  \mathbb{E}\!\bigl[(\tilde\mu_{t,i}-\hat\mu_{t,i})^2\mid\mathbf{x}_t\bigr]=c_t^2\,v_i(\mathbf{x}_t,t).
\end{equation}
\end{proposition}

\begin{proof}
As in Step~1 of \cref{app:kl_derivation}: substituting $x_{0,i}=\hat{x}_{0,i}+e_i$ into $\tilde\mu_{t,i}$ yields $\tilde\mu_{t,i}=\hat\mu_{t,i}+c_t e_i$, so $(\tilde\mu_{t,i}-\hat\mu_{t,i})^2=c_t^2 e_i^2$. Under~(A1), $\hat{x}_{0,i}=\mathbb{E}[X_{0,i}\mid\mathbf{x}_t]$ is a deterministic function of~$\mathbf{x}_t$, so $\mathbb{E}[e_i^2\mid\mathbf{x}_t]=\mathrm{Var}(X_{0,i}\mid\mathbf{x}_t)=v_i$.
\end{proof}

\paragraph{Implication for the asymmetric allocation.} \cref{prop:mean_error} demonstrates that when the posterior variance~$v_i$ is large, the model's mean~$\hat\mu_{t,i}$ becomes an unreliable estimate of $\tilde\mu_{t,i}$, with error variance precisely~$c_t^2 v_i$. Two key consequences arise from this observation.

First, the KL-optimal variance~$\gamma_{t,i}^*=\tilde\beta_t+c_t^2 v_i$ in Eq.~\ref{prop:kl_opt} includes this term as a necessary additive component. This term ensures that the transition remains calibrated to the true posterior when~$\hat\mu_{t,i}$ is used in place of~$\tilde\mu_{t,i}$. Gating this term (i.e., scaling it by some $g\in[0,1)$) would suppress the correction that compensates for the mean error, particularly at pixels where the mean error is largest. %We therefore apply $c_t^2 v_i$ uniformly across pixels.

Second, the exploratory term~$\tilde\beta_t$ represents the irreducible stochasticity of the true reverse SDE. This term persists even with a perfect denoiser ($v_i=0$). The decision to include or suppress this term distinguishes DDIM (where $\tilde\beta_t$ is suppressed, resulting in a deterministic process) from DDPM (where $\tilde\beta_t$ is retained, resulting in a stochastic process). This is the term that is gated on a per-pixel basis:
$$
  \Sigma_{t,i}^2=c_t^2 v_i+g_{t,i}\tilde\beta_t.
$$
The coupling in the SANI update (Eq.~\ref{eq:sani_update}) reinforces this point. Since $\Sigma_{t,i}$ appears in both the noise scale and through $\sqrt{1-\bar\alpha_{t-1}-\Sigma_{t,i}^2}$, the direction coefficient, increasing $v_i$ simultaneously inflates the noise and reduces the weight of the unreliable direction. The same quantity that increases the variance also indicates the unreliability of the direction, so the update addresses both aspects simultaneously. In contrast, a decoupled rule that adjusted only the variance would continue to follow an inaccurate direction at full weight.
%
% ----------------------------------------------------------------------------
\section{Admissibility of the SANI Update}
\label{app:admissibility}
% ----------------------------------------------------------------------------
%
The SANI update (Eq.~\ref{eq:sani_update}) is well-defined when
\begin{equation}\label{eq:app_admissible_range}
  0\leq\Sigma_{t,i}^2\leq 1-\bar\alpha_{t-1}
\end{equation}
at every pixel, ensuring that both $\Sigma_{t,i}$ and the direction coefficient $\sqrt{1-\bar\alpha_{t-1}-\Sigma_{t,i}^2}$ are real. The lower bound is satisfied due to the clamp $v_i\gets\max(v_i,\epsilon)$ and the condition $g_{t,i}\geq 0$. The following establishes the upper bound.
\subsection{Per-Pixel Bounds}
Since $g_{t,i}\in[0,1]$, $\Sigma_{t,i}^2\leq c_t^2 v_i+\tilde\beta_t=\gamma_{t,i}^*$. It therefore suffices to bound~$\gamma_{t,i}^*$.
\begin{proposition}[Per-pixel bounds on $\gamma_{t,i}^*$]
\label{prop:app_bounds}
$\gamma_{t,i}^*\geq\tilde\beta_t$, with equality iff $v_i=0$. Moreover, $\gamma_{t,i}^*\leq\beta_t/\alpha_t$ iff $[\mathbf{H}_t]_{ii}\leq 0$, with equality at $[\mathbf{H}_t]_{ii}=0$ (equivalently $v_i=(1-\bar\alpha_t)/\bar\alpha_t$).
\end{proposition}

\begin{proof}
The lower bound is immediate. For the upper bound, recall from Eq.~\ref{eq:v_i} that $v_i\leq(1-\bar\alpha_t)/\bar\alpha_t\iff[\mathbf{H}_t]_{ii}\leq 0$. Evaluating $\gamma_{t,i}^*$ at $v_i=(1-\bar\alpha_t)/\bar\alpha_t$ with $c_t=\sqrt{\bar\alpha_{t-1}}\beta_t/(1-\bar\alpha_t)$:
\begin{align*}
  \gamma_{t,i}^*\big|_{v_i=(1-\bar\alpha_t)/\bar\alpha_t}
  &=\tilde\beta_t+c_t^2\,\frac{1-\bar\alpha_t}{\bar\alpha_t}
  =\frac{\beta_t}{1-\bar\alpha_t}\Bigl[(1-\bar\alpha_{t-1})+\frac{\bar\alpha_{t-1}\beta_t}{\bar\alpha_t}\Bigr]\\
  &=\frac{\beta_t}{1-\bar\alpha_t}\cdot\frac{1-\bar\alpha_t}{\alpha_t}
  =\frac{\beta_t}{\alpha_t},
\end{align*}
where we used $\bar\alpha_{t-1}/\bar\alpha_t=1/\alpha_t$ and $(1-\bar\alpha_{t-1})+\beta_t/\alpha_t=(1-\bar\alpha_t)/\alpha_t$. Since $\gamma_{t,i}^*$ is increasing in~$v_i$, $\gamma_{t,i}^*\leq\beta_t/\alpha_t$ holds exactly when $v_i\leq(1-\bar\alpha_t)/\bar\alpha_t$.
\end{proof}

\Cref{prop:app_bounds} demonstrates that $\beta_t/\alpha_t$ bounds $\gamma_{t,i}^*$ only at pixels of negative local curvature. At pixels of positive curvature, $\gamma_{t,i}^*$ may exceed $\beta_t/\alpha_t$, indicating that the bound is not uniform. However, this value serves as the average-case bound. Averaging $\gamma_{t,i}^*$ over all pixels and applying the score-Hessian relation yiels the scalar Analytic-DPM variance $\sigma_t^{*2}$, which satisfies $\tilde\beta_t\leq\sigma_t^{*2}\leq\beta_t/\alpha_t$~\citep{bao2022analytic}.

\subsection{Schedule Condition}

\begin{lemma}\label{lem:app_schedule}
$\beta_t/\alpha_t\leq 1-\bar\alpha_{t-1}\iff\bar\alpha_t\leq 1-2\beta_t$.
\end{lemma}

\begin{proof}
$\beta_t/\alpha_t\leq 1-\bar\alpha_{t-1}\iff\beta_t\leq\alpha_t(1-\bar\alpha_{t-1})=\alpha_t-\bar\alpha_t$. With $\beta_t=1-\alpha_t$, this is $1-\alpha_t\leq\alpha_t-\bar\alpha_t$, i.e., $\bar\alpha_t\leq 2\alpha_t-1=1-2\beta_t$.
\end{proof}

For standard schedules where $\beta_t\ll 1$, the right-hand side approaches~$1$, so the condition is violated only when $\bar\alpha_t$ is close to~$1$, corresponding to the earliest forward steps (or latest reverse steps). At $t=1$, $\bar\alpha_0=1$ and the condition fails, however, the algorithm sets $\mathbf{z}=\mathbf{0}$, so admissibility vacuous. At $t=2$, $1-\bar\alpha_1=\beta_1$ and the condition requires $\beta_2/\alpha_2\leq\beta_1$, which does not hold for monotonically increasing schedules. Numerically, the condition is satisfied for all $t\geq t^*$, with $t^*=3$ for the linear schedule ($\beta_1=10^{-4}$, $\beta_T=0.02$, $T=1000$) and $t^*=6$ for the cosine schedule \cite{bao2022analytic}.

\subsection{Admissibility and the Per-Pixel Clip}

Combining \cref{prop:app_bounds} and Lemma~\ref{lem:app_schedule}: for $t\geq t^*$, every pixel with $[\mathbf{H}_t]_{ii}\leq 0$ satisfies
\begin{equation}\label{eq:app_admit_chain}
  \Sigma_{t,i}^2\leq\gamma_{t,i}^*\leq\frac{\beta_t}{\alpha_t}\leq 1-\bar\alpha_{t-1},
\end{equation}
and is admissible by construction. Only pixels with positive local curvature, where $v_i$ exceeds the forward-noise level, can violate Eq.~\ref{eq:app_admissible_range}. For these cases, the per-pixel clip is applied
\begin{equation}\label{eq:app_clip}
  \Sigma_{t,i}^2\gets\min\!\bigl(\Sigma_{t,i}^2,\,1-\bar\alpha_{t-1}\bigr).
\end{equation}
This procedure guarantees a real direction coefficient and preserves $\Sigma_{t,i}^2\geq 0$.
\paragraph{Empirical negligibility.} The clip (Eq.~\ref{eq:app_clip}) is activated for only a negligible fraction of pixels. During sampling, the iterate~$\mathbf{x}_t$ remains within high-probability regions, where the log-marginal is locally concave and $[\mathbf{H}_t]_{ii}\leq 0$ at the majority of coordinates. By \Cref{prop:app_bounds}, these are automatically admissible. The remaining cases are further suppressed at timesteps where Eq.~\ref{eq:app_admissible_range} is most restrictive. At $t=1$ no noise is injected ($\mathbf{z}=\mathbf{0}$), and for $t\in\{2,\ldots,t^*-1\}$ the gate satisfies $g_{t,i}\approx 0$ and $v_i$ is small (the signal-to-noise ratio $\bar\alpha_t/(1-\bar\alpha_t)\gg 1$), so $\Sigma_{t,i}^2\approx c_t^2 v_i\approx 0$ regardless of the clip (Fig.~\ref{fig:perstep_stats}).
% ===========================================================================
%  End theory appendix
% ===========================================================================
%
% ===========================================================================
%  Start theory experiments
% ===========================================================================
%
\section{Experimental Details}
\label{app:experimental_details}
This appendix provides reproducibility details, pretrained-model sources, per-dataset hyperparameters, and information about the Hessian-diagonal source and its training setup to support~\Cref{sec:experiments}. The experiments include comparisons with DDPM using $\sigma_t^2 = \tilde{\beta}_t$ and $\sigma_t^2 = \beta_t$~\citep{ho20}, DDIM~\citep{song2021ddim}, Analytic-DPM (A-DDPM/DDIM)~\citep{bao2022analytic}, NPR DDPM/DDIM~\citep{bao2022estimating} (which models the noise prediction residual), SN-DDPM/DDIM~\citep{bao2022estimating} (which models the second moment of the noise), and OCM-DDPM/DDIM~\citep{ou2025improving}.
Consistent with \cite{bao2022estimating}, the approach relies on a pretrained score-based neural network in fixed parameters throughout our procedures.

\paragraph{\textbf{Details of Pretrained Models.}}
\Cref{tab:app_pretrained} summarizes the pretrained score and Hessian-diagonal network used to estimate~$v_i$. Score models parameterize the noise prediction, $\boldsymbol{\epsilon}_\theta(\mathbf{x}_t,t)$, from which the score is recovered as $\mathbf{s}_\theta(\mathbf{x}_t)=\nabla_{\mathbf{x}_t}\log p_\theta(\mathbf{x}_t)=-\boldsymbol{\epsilon}_\theta(\mathbf{x}_t,t)/\sqrt{1-\bar\alpha_t}$. 

For the per-pixel posterior variance $v_i(\mathbf{x}_t,t)=\bigl[\mathrm{Cov}[\mathbf{x}_0\mid\mathbf{x}_t]\bigr]_{ii}$ of Eq.~\ref{eq:v_i}, the pre-trained Hessian-diagonal networks released by~\cite{ou2025improving} are employed, which predict $[\mathbf{H}_t]_{ii}$ in a single forward pass. These networks are lightweight, being significantly smaller than the score network, and are trained by minimizing $\mathbb{E}_{t,\mathbf{x}_t,\mathbf{r}}\|h_\phi(\mathbf{x}_t,t)-\mathbf{r}\odot(\mathbf{H}_t\mathbf{r})\|^2$ with Rademacher probes $\mathbf{r}\in\{-1,+1\}^D$. The same EMA weightsare used at inference as in~\cite{ou2025improving}.
 
The alternative Hutchinson estimator $[\mathbf{H}_t]_{ii}\approx r_i\cdot[\mathbf{H}_t\mathbf{r}]_i$ does not require an auxiliary network but introduces one additional Jacobian-vector product (JVP) per step, resulting approximately twice the computational cost of DDIM.

\begin{table}[h]
\centering
\caption{Source of pretrained score prediction networks used in our experiments.}
\label{tab:app_pretrained}
%\vskip -0.1in
\begin{small}
\begin{sc}
\setlength{\tabcolsep}{.55pt}{
\begin{tabular}{lc}
\toprule 
     & Provided By \\
\midrule 
    CIFAR10 (LS) & \cite{bao2022analytic} \\
    CIFAR10 (CS) & \cite{bao2022analytic} \\
    CelebA 64X64 & \cite{song2021ddim} \\
    LSUN Bedroom & \cite{ho20} \\
    Hessian net & \cite{ou2025improving}\\
\bottomrule
\end{tabular}}
\end{sc}
\end{small}
\end{table}
 
% \begin{table}[h]
% \centering
% \caption{Pretrained score and Hessian networks used in the experiments.}
% \label{tab:app_pretrained}
% \small
% \begin{tabular}{lc}
% \toprule
%  & Source \\
% \midrule
% CIFAR-10 (LS) & \cite{bao2022analytic} \\
% CIFAR-10 (CS) & \cite{bao2022analytic} \\
% CelebA $64\times 64$ & \cite{song2021ddim} \\
% LSUN Bedroom $256\times 256$ & \cite{ho20} \\
% Hessian-diagonal network & \cite{ou2025improving} \\
% \bottomrule
% \end{tabular}
% \end{table}
 
\subsection{Per-Dataset Hyperparameters}
The only hyperparameter introduced by SANI is the tolerance~$\tau$, which determines the operating point of the gating function via the ratio $\tau/v_i$. The value of~$\tau$ is calibrated for each dataset, selected from the range $[10^{-4}-10^{-2}]$ in which the gate exhibits the intended stochastic-to-deterministic transition within the trajectory.  \Cref{tab:app_hyperparams} lists the values used for the main FID and qualitative results.
\begin{table}[h]
\centering
\caption{Per-dataset hyperparameters. $T$ is the training-time number of timesteps; $K\in\{10,25,50,100,200,1000\}$ at inference is reported in the relevant figures.}
\label{tab:app_hyperparams}
\small
\begin{tabular}{lcccc}
\toprule
& CIFAR-10 (LS) & CIFAR-10 (CS) & CelebA $64$ & LSUN Bedroom \\
\midrule
Resolution        & $32\times 32$ & $32\times 32$ & $64\times 64$ & $256\times 256$ \\
Schedule          & linear        & cosine        & linear        & linear \\
$T$ (training)    & 1000          & 1000          & 1000          & 1000 \\
Tolerance $\tau$  & $10^{-3}$     & $10^{-4}$     & $10^{-3}$     & $10^{-3}$ \\
Variance floor $\epsilon$ & $10^{-6}$ & $10^{-6}$ & $10^{-6}$ & $10^{-6}$ \\
\bottomrule
\end{tabular}
\end{table}
 
% \paragraph{\textbf{Hessian-Diagonal.}} For the per-pixel posterior variance $v_i(\mathbf{x}_t,t)=\bigl[\mathrm{Cov}[\mathbf{x}_0\mid\mathbf{x}_t]\bigr]_{ii}$ of Eq.~\ref{eq:v_i}, we use the pre-trained Hessian-diagonal networks released by~\cite{ou2025improving}, which predict $[\mathbf{H}_t]_{ii}$ in a single forward pass. These networks are lightweight (significantly smaller than the score network) and are trained by minimizing $\mathbb{E}_{t,\mathbf{x}_t,\mathbf{r}}\|h_\phi(\mathbf{x}_t,t)-\mathbf{r}\odot(\mathbf{H}_t\mathbf{r})\|^2$ with Rademacher probes $\mathbf{r}\in\{-1,+1\}^D$. We use them with the same EMA weights at inference as~\cite{ou2025improving}.
 
% The alternative Hutchinson estimator $[\mathbf{H}_t]_{ii}\approx r_i\cdot[\mathbf{H}_t\mathbf{r}]_i$ requires no auxiliary network but adds one Jacobian--vector product per step (about $2\times$ the cost of DDIM).
 
\paragraph{\textbf{Evaluation.}} All evaluations employ exponential-moving-average (EMA) weights with decay~$0.9999$, which is the standard setting in the diffusion literature~\citep{nichol2021improved,bao2022estimating,peebles2023scalable}. FID is computed from $50{,}000$ generated samples. For CIFAR-10, the reference statistics are extracted from the entire training set (50{,}000 images), while for CelebA, the standard test split is used. Likelihoods are reported as variational upper bounds in bits per dimension along the same sub-sequence used at sampling time, ensuring that FID and NLL are reported under matched schedules. 
% ===========================================================================
 %
\section{Extended Theory Validation}
\label{app:theory_validation}
 This appendix provides extended visualizations of the theoretical identities across the remaining datasets and presents the complete gate-error correlation scatter to support \Cref{sec:exp_fim}.
 %
% ---------------------------------------------------------------------------
\subsection{FIM-Jacobian-Posterior Covariance Correspondence on Additional Datasets}
\label{app:fim_validation_extended}
% ---------------------------------------------------------------------------
\Cref{fig:fim_correspondence_extended} extends \Cref{fig:fim_correspondence} from LSUN Bedroom to CelebA and CIFAR-10 (LS and CS). Despite the lower resolution ($32\times 32$ to $64\times 64$), the spatial correspondence between the denoiser Jacobian diagonal (row~2), the posterior covariance diagonal (row~3), and the reconstruction accuracy (row~4) remains clearly visible along object boundaries and persists across different noise schedules.
\begin{figure}[h]
  \centering
  \begin{subfigure}[b]{0.32\linewidth}
    \centering
    \includegraphics[width=\linewidth]{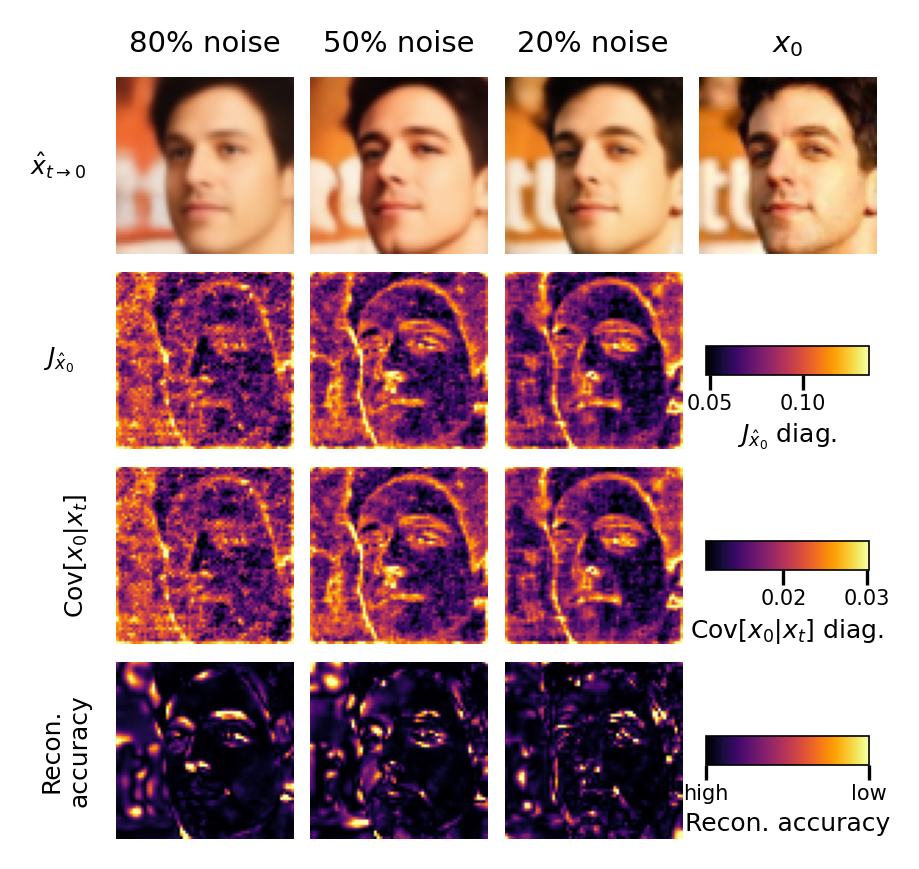}
    \caption{CelebA $64\times 64$}
  \end{subfigure}\hfill
  \begin{subfigure}[b]{0.32\linewidth}
    \centering
    \includegraphics[width=\linewidth]{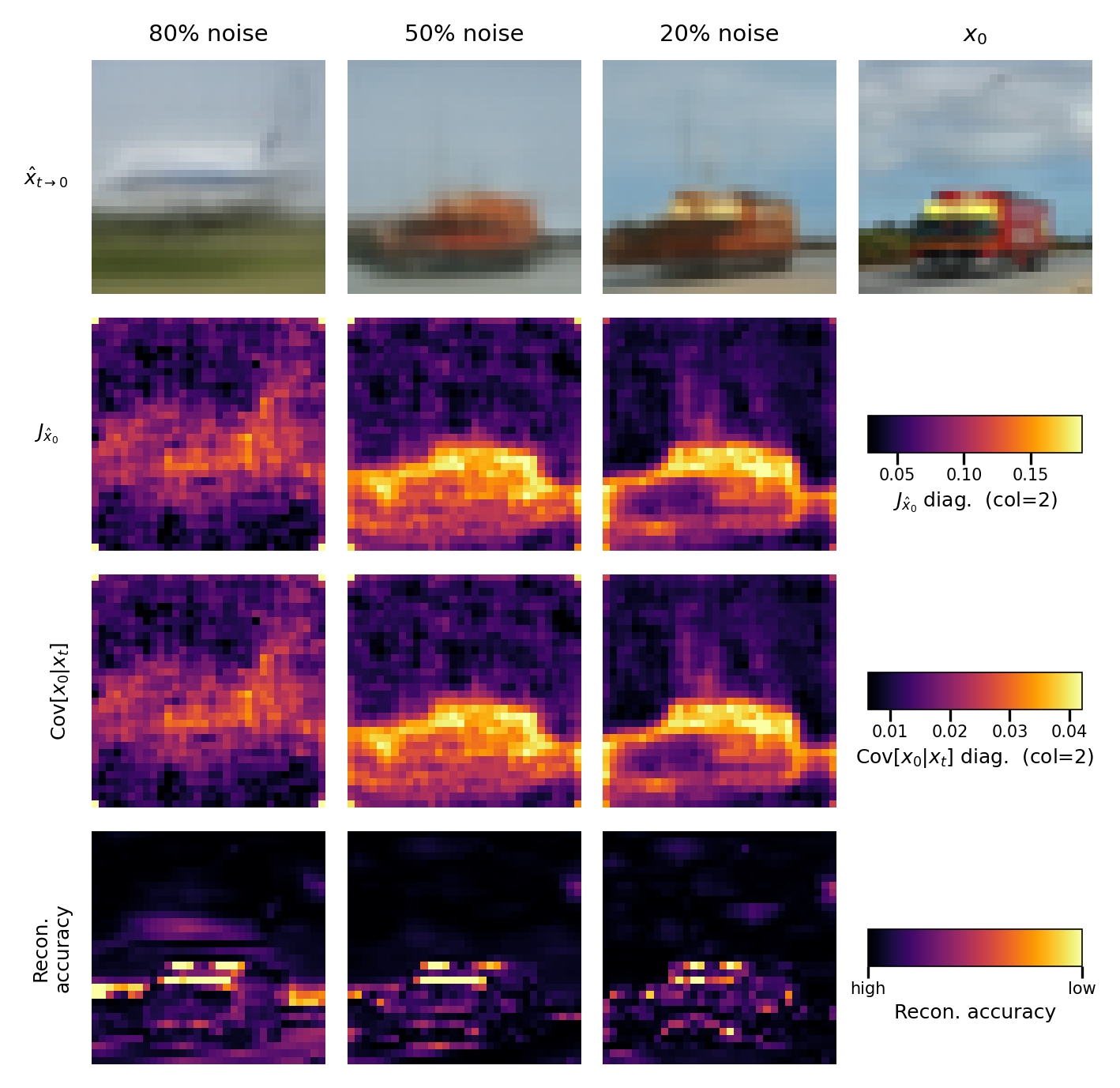}
    \caption{CIFAR-10 (LS)}
  \end{subfigure}\hfill
  \begin{subfigure}[b]{0.32\linewidth}
    \centering
    \includegraphics[width=\linewidth]{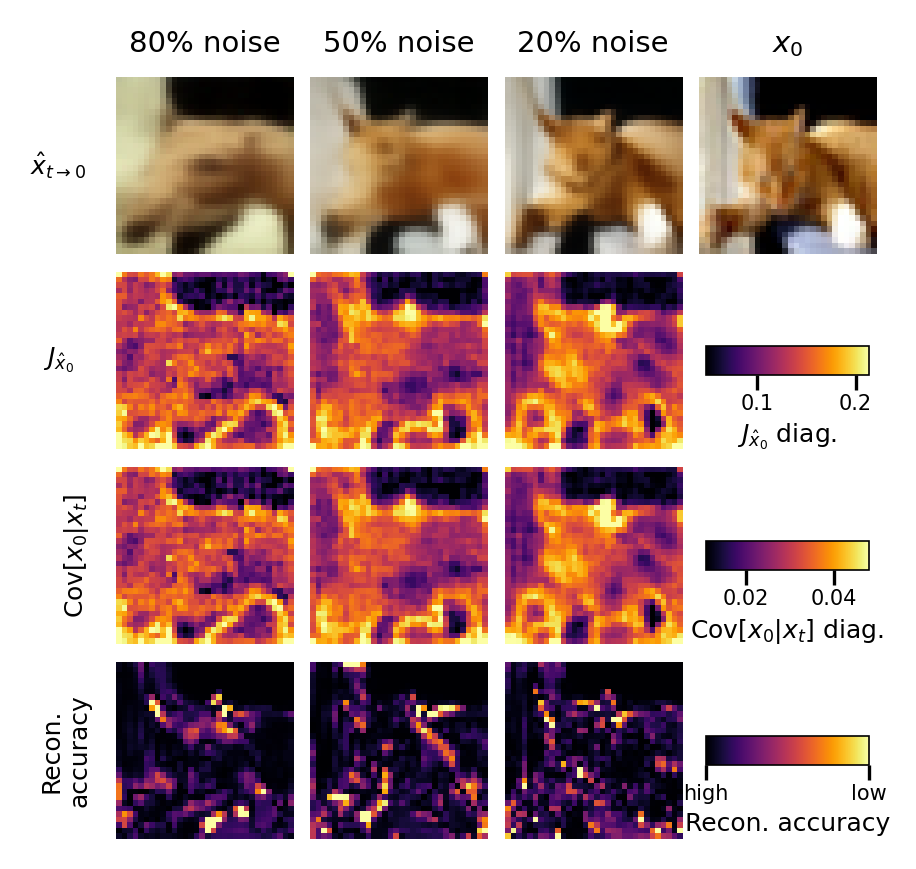}
    \caption{CIFAR-10 (CS)}
  \end{subfigure}
  \caption{FIM-Jacobian-Posterior Covariance correspondence on additional datasets, same layout as Fig.~\ref{fig:fim_correspondence}. The spatial correspondence between the Jacobian diagonal, the posterior covariance diagonal, and reconstruction accuracy is preserved at all resolutions and across both noise schedules.}
  \label{fig:fim_correspondence_extended}
\end{figure}
 
% ---------------------------------------------------------------------------
\subsection{Gate-Error Correlation Across Noise Levels}
\label{app:gate_error_correlation}
% ---------------------------------------------------------------------------
\Cref{fig:app_gate_error_correlation} presents the per-pixel gating value~$g_{t,i}$ versus the validation loss $(x_{0,i}-\hat{x}_{0,i})^2$ on held-out images at six noise levels for LSUN Bedroom, CelebA, and CIFAR-10 (LS and CS) datasets. The log-space Pearson correlation rises monotonically from high (80\%) noise levels to lower noise levels across all datasets. The lower correlation at high noise due to the compression of $g_{t,i}$ near~$1$ across all pixels, resulting in limited variance for correlation. The correlation reported here uses the gating value~$g_{t,i}$ directly, but the same monotonic trend is observed if $v_i$ is substituted for $g_{t,i}$, as the two are related by the monotone CDF transform of~\Cref{prop:gating}.
\begin{figure}[h]
  %\centering
  \begin{subfigure}[b]{0.49\linewidth}%\centering
    \includegraphics[width=\linewidth]{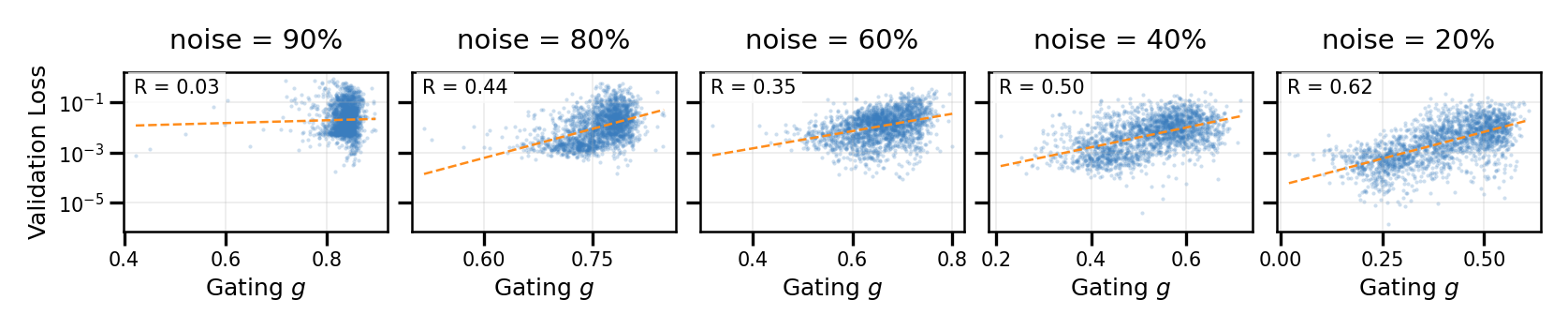}
    \caption{CIFAR-10 (LS)}
  \end{subfigure}\hfill
  \begin{subfigure}[b]{0.49\linewidth}%\centering
    \includegraphics[width=\linewidth]{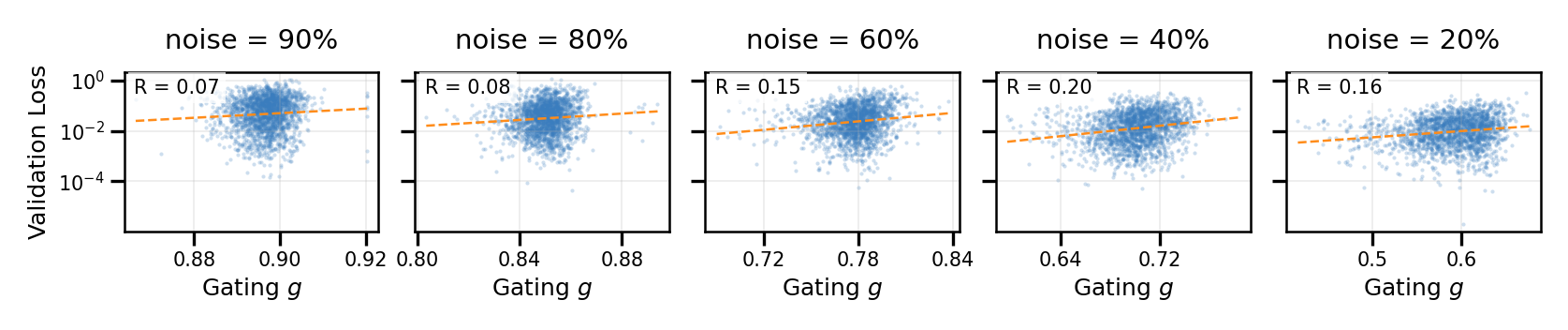}
    \caption{CIFAR-10 (CS)}
  \end{subfigure}

  \begin{subfigure}[b]{0.49\linewidth}%\centering
    \includegraphics[width=\linewidth]{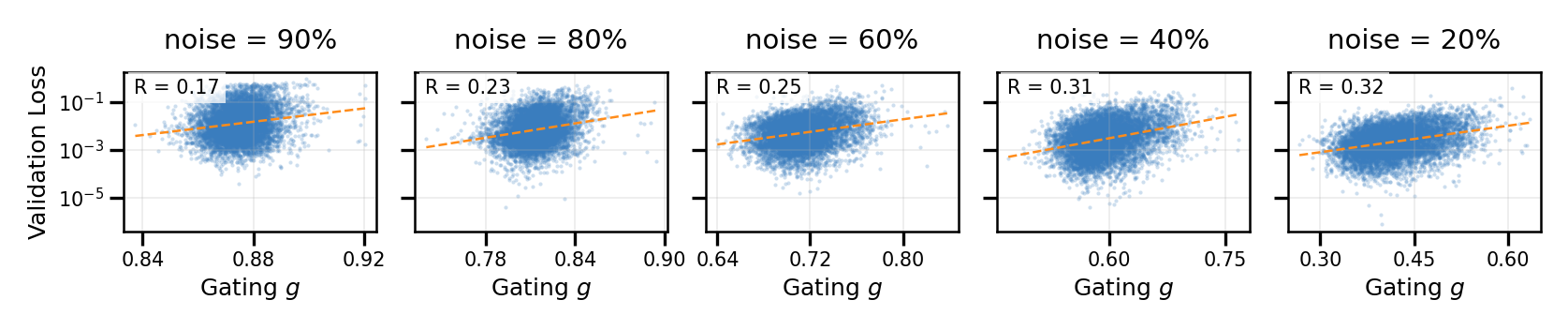}
    \caption{CelebA $64$}
  \end{subfigure}\hfill
  \begin{subfigure}[b]{0.49\linewidth}%\centering
    \includegraphics[width=\linewidth]{figs/heatmaps_new/gating_vs_loss_sang_lsum.png}
    \caption{LSUN Bedroom}
  \end{subfigure}
  \caption{Per-pixel gating value $g_{t,i}$ versus validation loss $(x_{0,i}-\hat{x}_{0,i})^2$ at six noise levels (left to right: decreasing noise). The log-space Pearson correlation $R$ rises monotonically, confirming that high-$g_{t,i}$ pixels are those the model reconstructs poorly.}
  \label{fig:app_gate_error_correlation}
\end{figure}

% ===========================================================================
%  APPENDIX: Extended Gating Visualizations
%  Supports Section 5.3
% ===========================================================================
 
\section{Extended Gating Visualizations}
\label{app:extended_gating}
This appendix provides additional results supporting \cref{sec:exp_qualitative} on the remaining datasets. First, the gate is evaluated at controlled levels of forward noise on validation images, which isolates the gate's behavior from trajectory dynamics by removing dependence on the unfolding of the reverse process (see \cref{app:gating_controlled_noise}). Subsequently, the same gate-versus-Sobel analysis is conducted along the reverse trajectory (see \cref{app:gating_during_sampling}).
% --------------------------------------------------------------------
% ---------------------------------------------------------------------------
\subsection{Gate at Controlled Noise Levels}
\label{app:gating_controlled_noise}
% ---------------------------------------------------------------------------
To isolate the gate's response from trajectory dynamics, we evaluate it on held-out validation images at controlled levels of forward noise. Given a clean image $\mathbf{x}_0$, we sample $\mathbf{x}_t\sim q(\mathbf{x}_t\mid\mathbf{x}_0)$ at noise levels $80\%$, $50\%$, and $20\%$, and compute the gate from the Tweedie prediction. This setting tests the gate's behavior as a function of the pair $(\mathbf{x}_0,t)$ alone.
 
\Cref{fig:app_gating_noise_levels} presents the resulting gating maps on three datasets. The spatial selectivity observed in this setting is consistent with the behavior during sampling. The gate saturates near~$1$ across the image at high noise, sharpens onto geometrically complex regions at moderate noise, and collapses to near-zero on flat regions at low noise. These results confirm that alignment with image structure is a property of the score field at $(\mathbf{x}_t,t)$, rather than an artifact of the specific trajectory followed by the sampler.

\begin{figure}[t]
    \centering
    \begin{subfigure}[b]{0.3\textwidth}
        \centering
        \includegraphics[width=\linewidth]{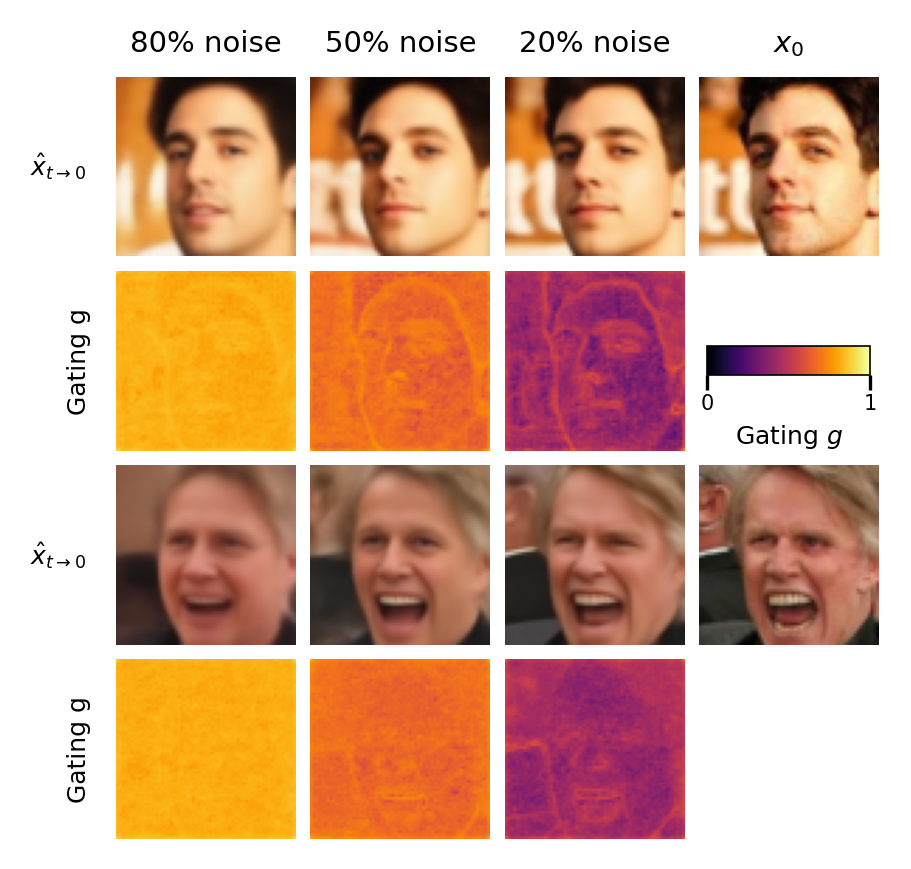}
        \caption{\footnotesize CelebA}
    \end{subfigure}
    \begin{subfigure}[b]{0.3\textwidth}
        \centering
        \includegraphics[width=\linewidth]{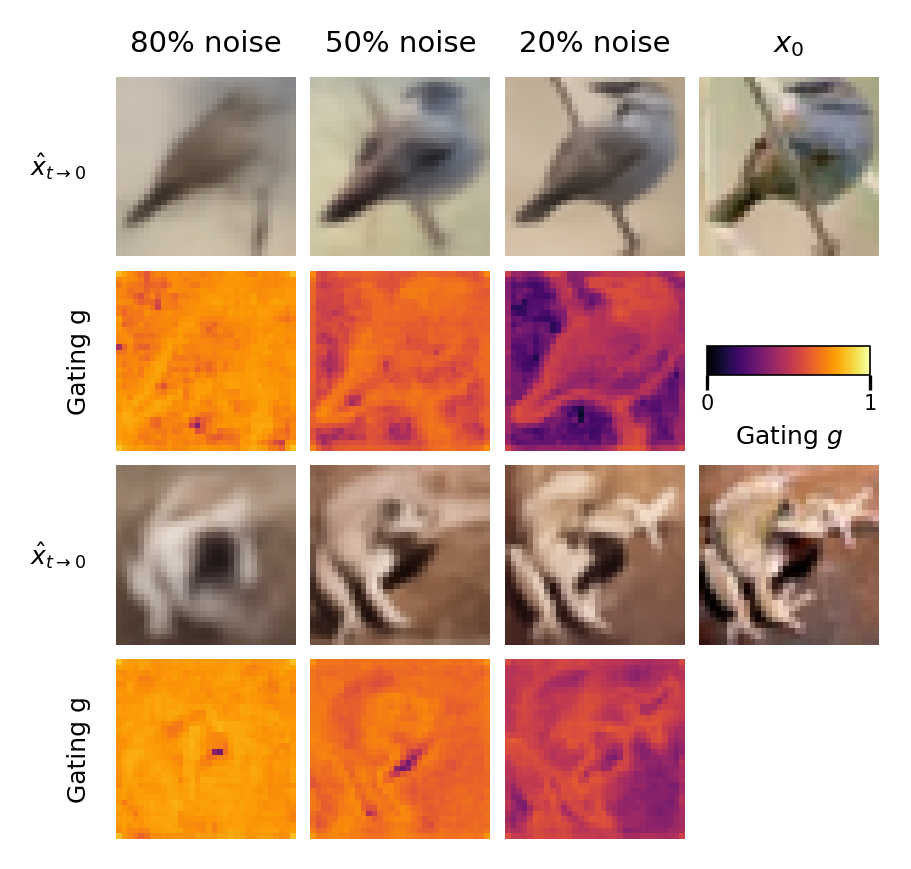}
        \caption{\footnotesize Cifar10 (LS)}
    \end{subfigure}
    \begin{subfigure}[b]{0.3\textwidth}
        \centering
        \includegraphics[width=\linewidth]{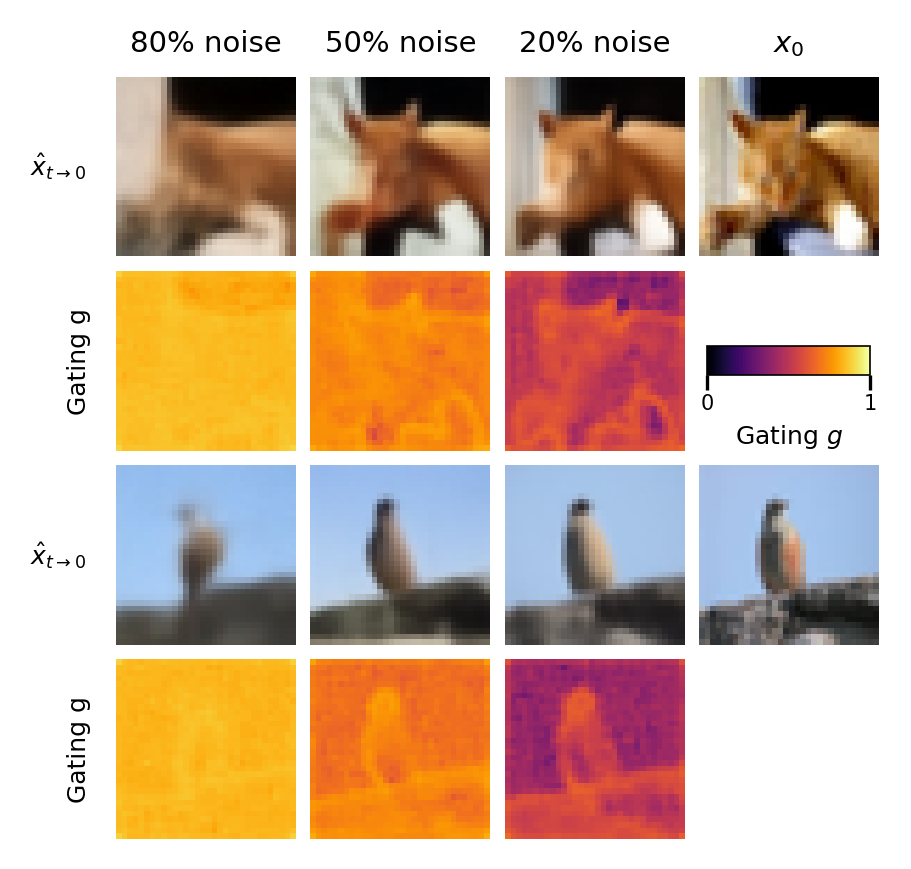}
        \caption{\footnotesize Cifar10 (CS)}
    \end{subfigure}
        \caption{Visualization of one-step predictions at varying timesteps (80\%, 50\%, 20\% noise) and the corresponding gating map (dark$\,{=}\,0$, bright$\,{=}\,1$). The gating function assigns high stochasticity to geometrically complex regions such as edges and textures, and near-zero stochasticity to flat regions such as walls and bed sheets. The spatial separation becomes more pronounced as noise decreases.}
    \label{fig:app_gating_noise_levels}
\end{figure}
 
% ---------------------------------------------------------------------------
\subsection{Gate vs.\ Sobel Along the Reverse Trajectory: Additional Datasets}
\label{app:gating_during_sampling}
% ---------------------------------------------------------------------------
The gate is next evaluated along an actual reverse trajectory. Figures~\ref{fig:app_gating_vs_edges_lsunb}-\ref{fig:app_gating_vs_edges_cifar10cs} further supports that SANI effectively isolates ``difficult'' spatial regions bases solely on inherent geometry, without requiring explicit edge-detection supervision or the architectural modifications proposed in recent work (e.g., \cite{schusterbauer2026patchforcing}).

%The same structural alignment observed on LSUN Bedroom is preserved across resolutions and schedules: the gate is globally near $1$ at high noise, progressively localizes on facial landmarks (CelebA) or object silhouettes (CIFAR) during denoising, and collapses to near-zero on background regions while retaining residual mass on edges and textures at the end of the trajectory.

\begin{figure}[t]
    \centering
    \begin{subfigure}[b]{0.24\textwidth}
        \centering
        \includegraphics[width=\linewidth]{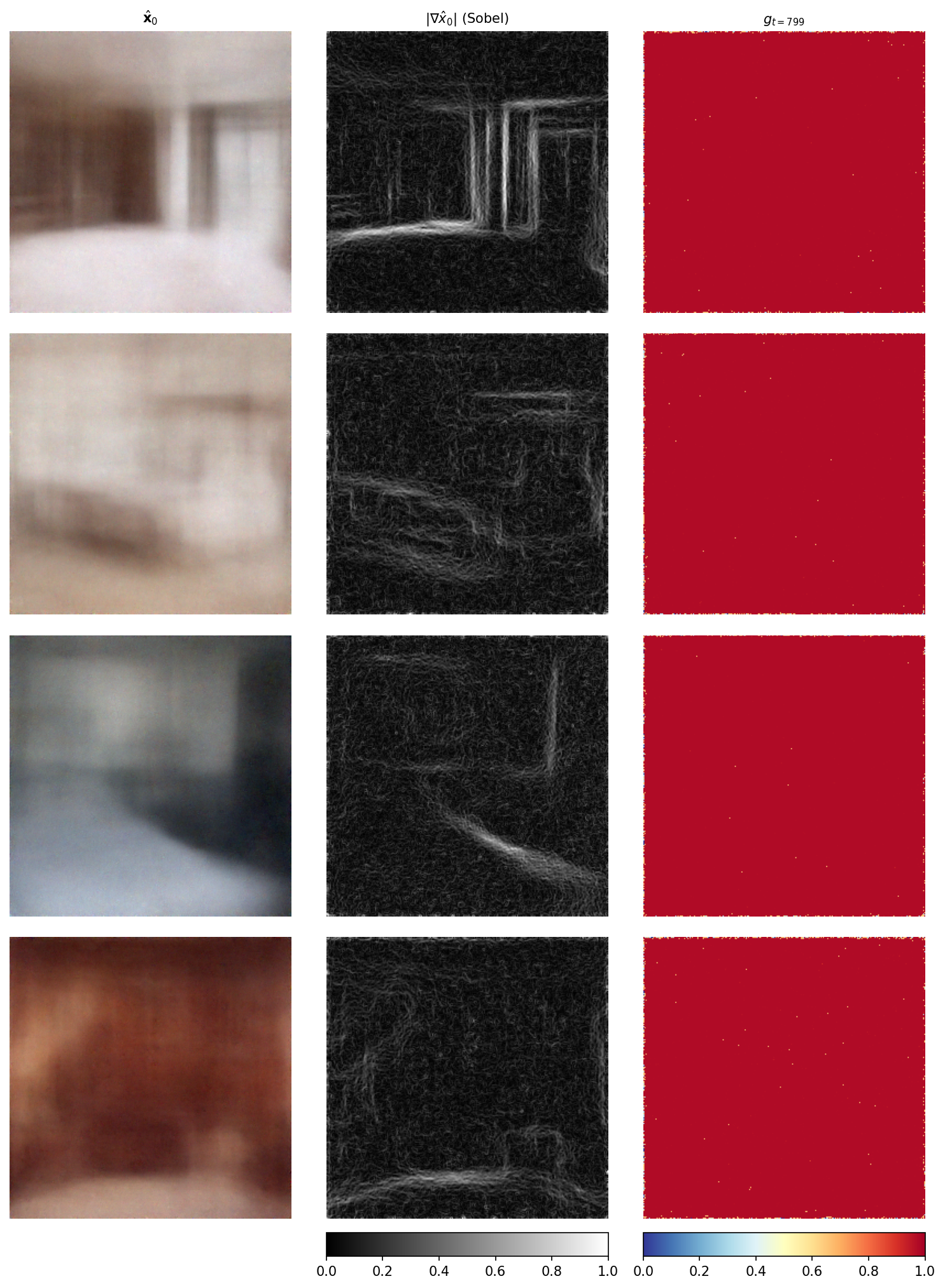}
        \caption{\footnotesize $t\!=\!800$}
        %\label{fig:gating_vs_edges_lsun_800}
    \end{subfigure}
    \begin{subfigure}[b]{0.24\textwidth}
        \centering
        \includegraphics[width=\linewidth]{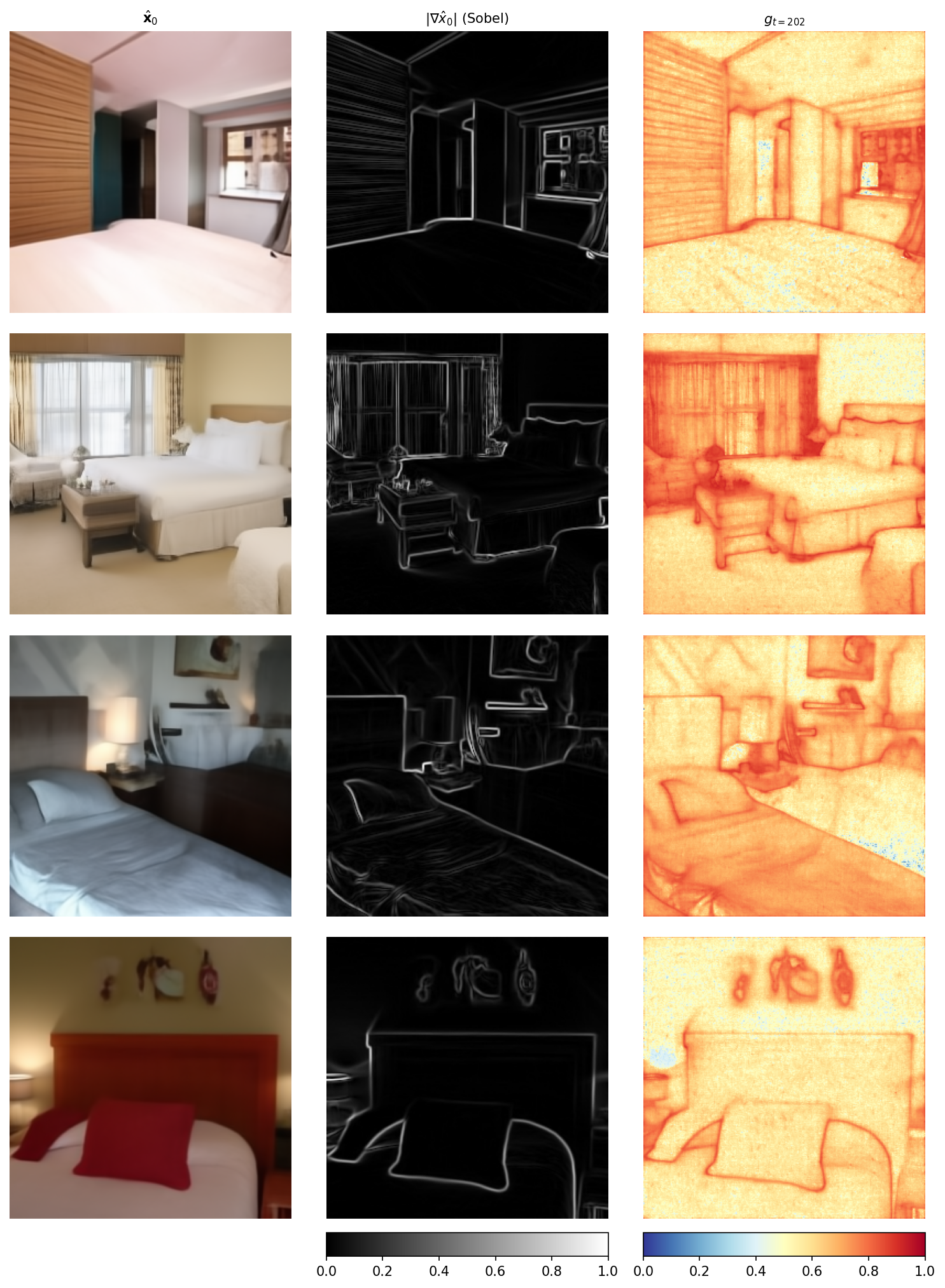}
        \caption{\footnotesize $t\!=\!200$}
        %\label{fig:gating_vs_edges_lsun_200}
    \end{subfigure}
    \begin{subfigure}[b]{0.24\textwidth}
        \centering
        \includegraphics[width=\linewidth]{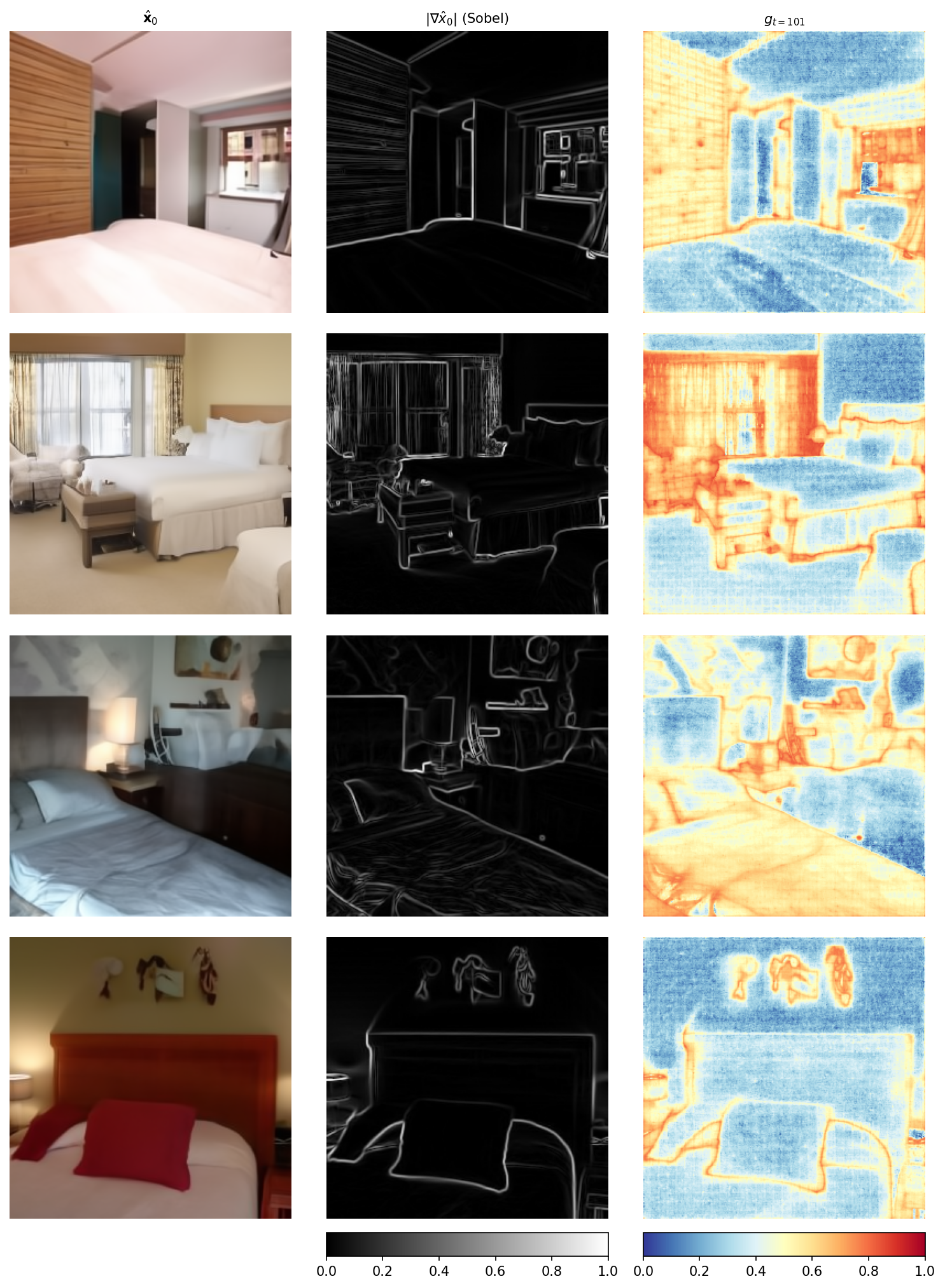}
        \caption{\footnotesize $t\!=\!100$}
    %\label{fig:gating_vs_edges_lsun_100}
    \end{subfigure}
    \begin{subfigure}[b]{0.24\textwidth}
        \centering
        \includegraphics[width=\linewidth]{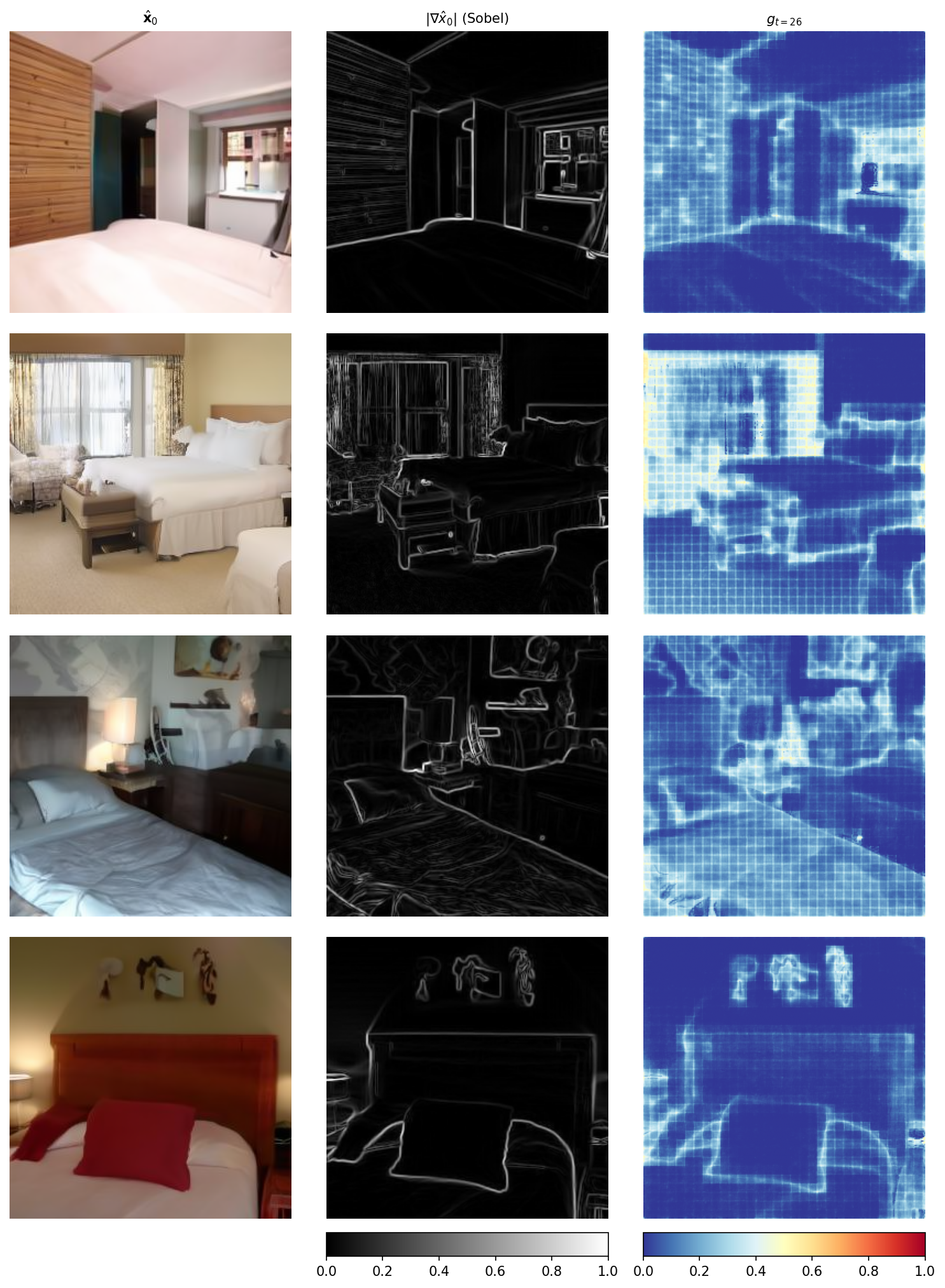}
        \caption{\footnotesize $t\!=\!25$}
                %\label{fig:gating_vs_edges_lsun_25}
    \end{subfigure}
    \caption{Per-pixel gating $g_{t,i}$ vs.\ Sobel edge map of $\hat{\mathbf{x}}_0$ on LSUN Bedroom. Each subplot: predicted image (left), Sobel magnitude (center), gating map (right, blue$\,{=}\,0$, red$\,{=}\,1$). The gating map progressively aligns with image structure without any edge-detection supervision.}
    \label{fig:app_gating_vs_edges_lsunb}
\end{figure}

\begin{figure}[t]
    \centering
    \begin{subfigure}[b]{0.24\textwidth}
        \centering
        \includegraphics[width=\linewidth]{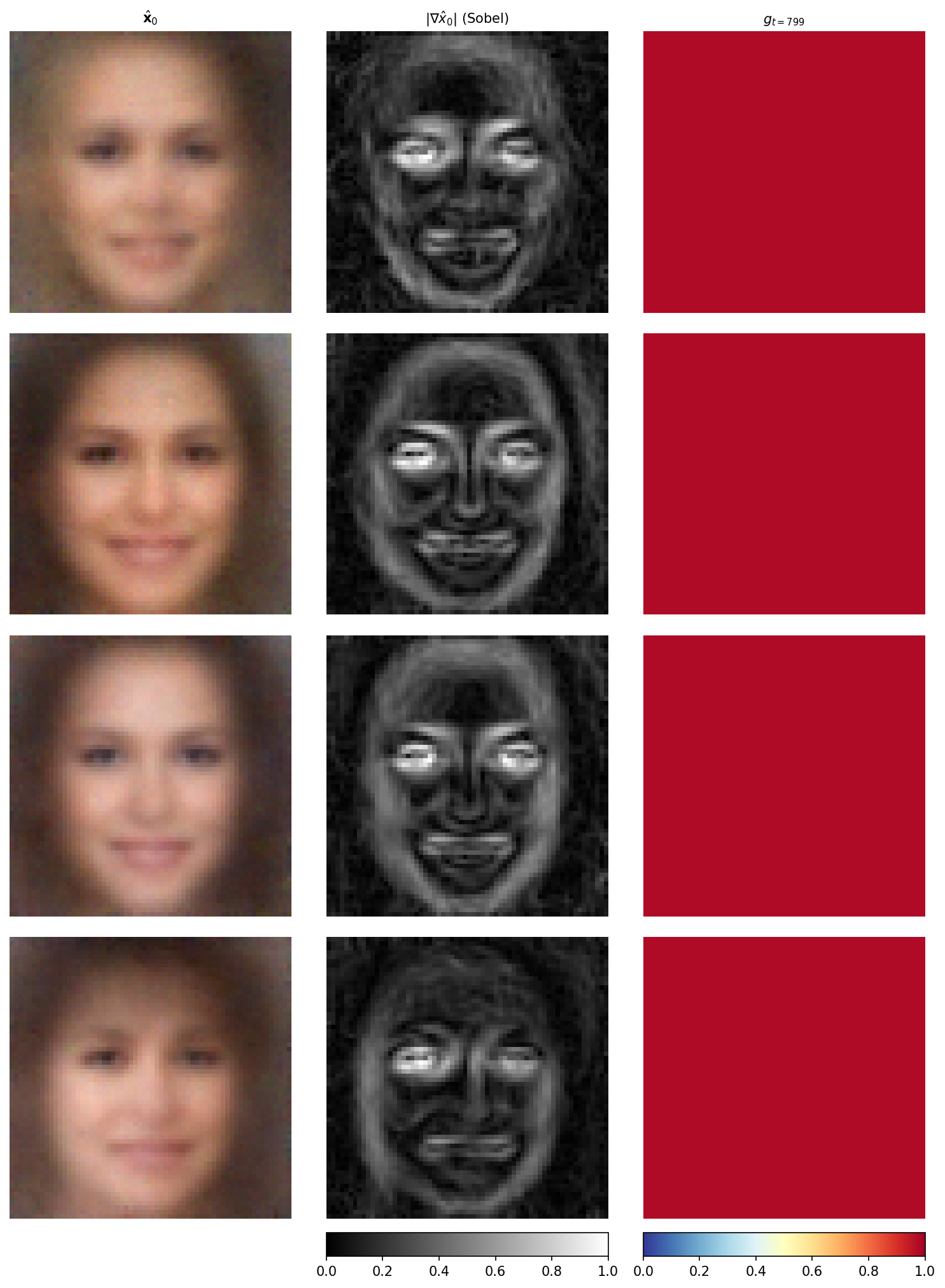}
        \caption{\footnotesize $t\!=\!800$}
        \label{fig:gating_vs_edges_celeba_800}
    \end{subfigure}
    \begin{subfigure}[b]{0.24\textwidth}
        \centering
        \includegraphics[width=\linewidth]{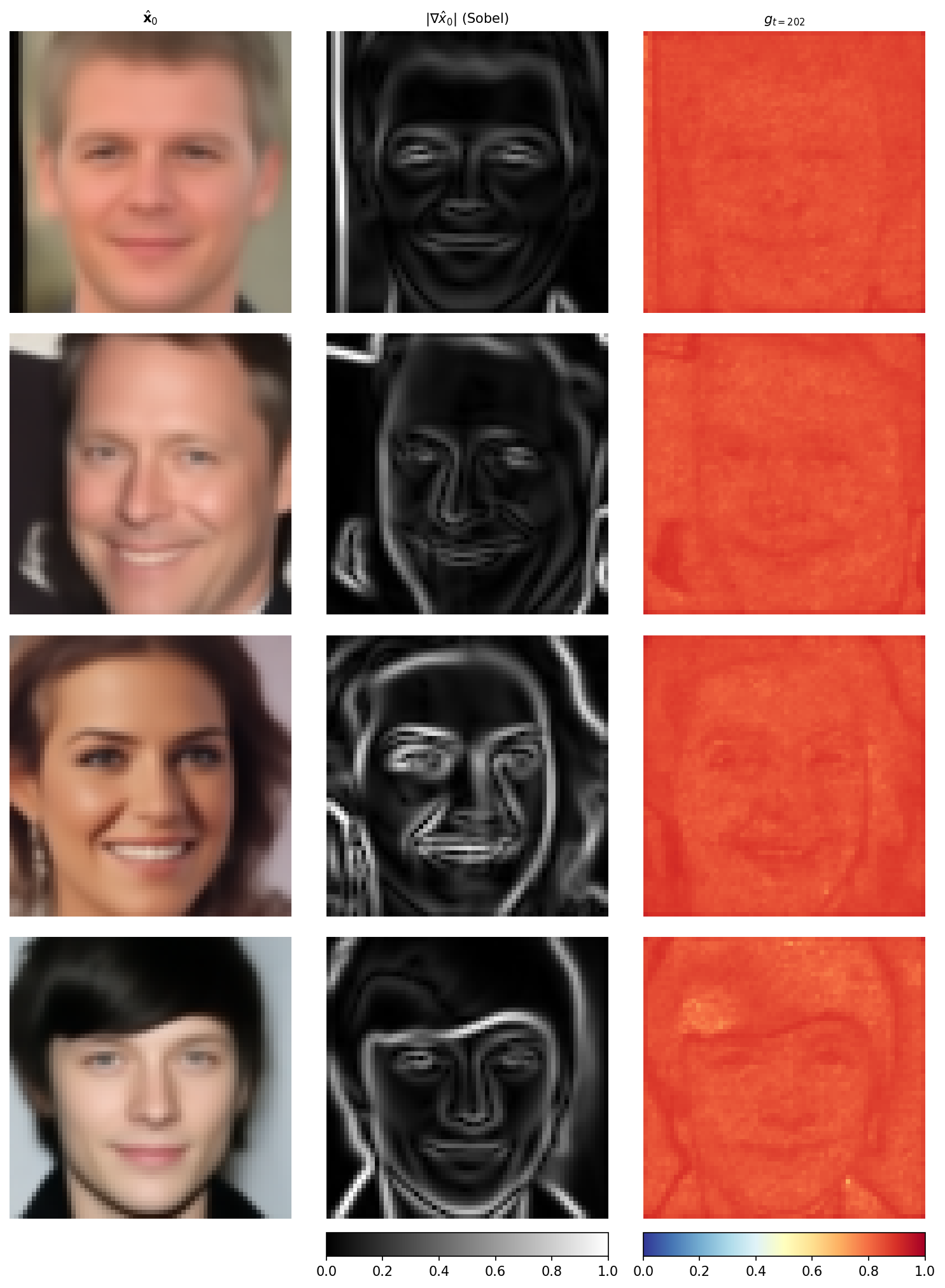}
        \caption{\footnotesize $t\!=\!200$}
        \label{fig:gating_vs_edges_celeba_200}
    \end{subfigure}
    \begin{subfigure}[b]{0.24\textwidth}
        \centering
        \includegraphics[width=\linewidth]{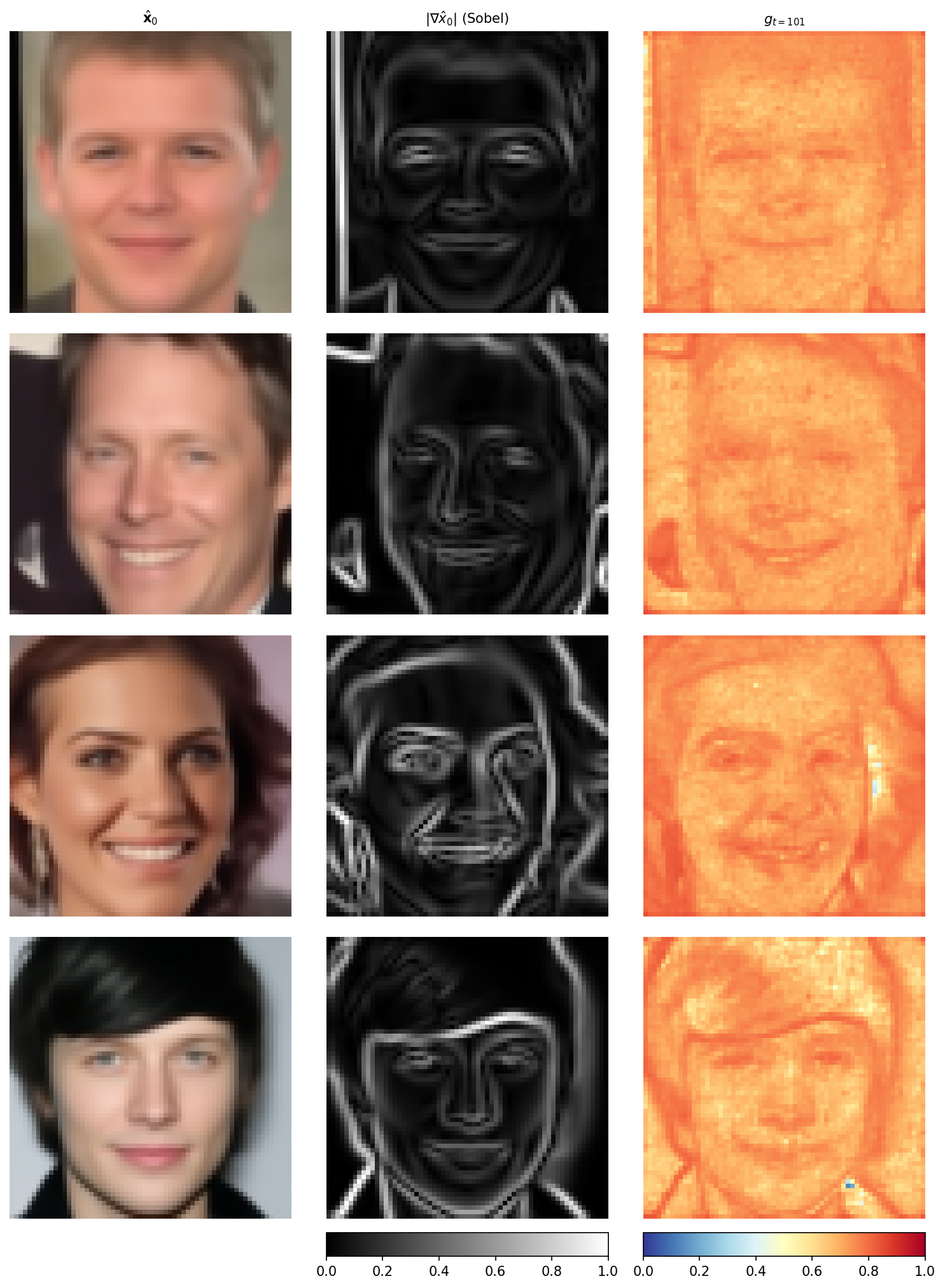}
        \caption{\footnotesize $t\!=\!100$}
    \label{fig:gating_vs_edges_celeba_100}
    \end{subfigure}
    \begin{subfigure}[b]{0.24\textwidth}
        \centering
        \includegraphics[width=\linewidth]{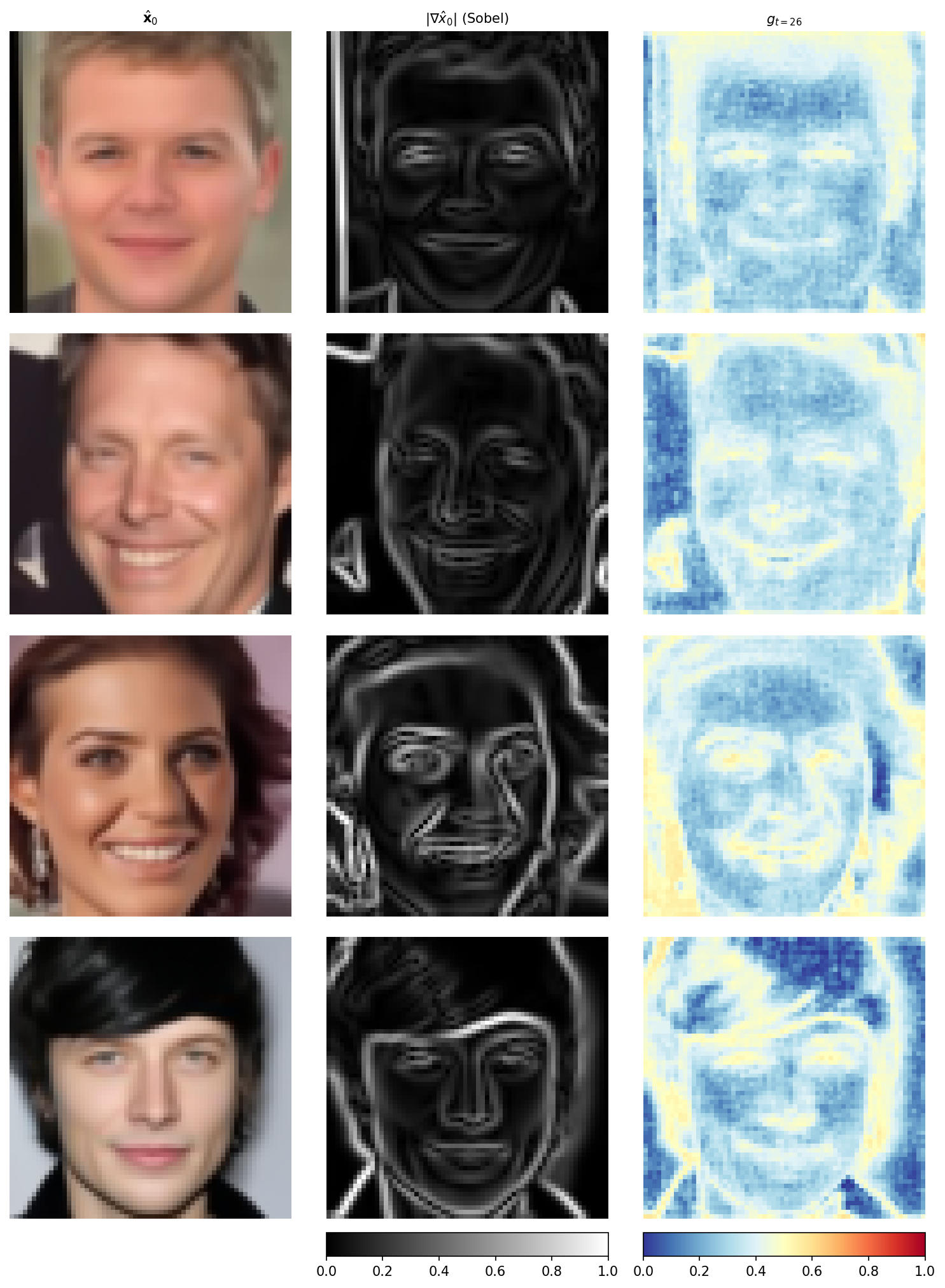}
        \caption{\footnotesize $t\!=\!25$}
                \label{fig:gating_vs_edges_celeba_25}
    \end{subfigure}
    \caption{Per-pixel gating $g_{t,i}$ vs.\ Sobel edge map of $\hat{\mathbf{x}}_0$ on CelebA. Each subplot: predicted image (left), Sobel magnitude (center), gating map (right, blue$\,{=}\,0$, red$\,{=}\,1$). The gating map progressively aligns with image structure without any edge-detection supervision.}
    \label{fig:app_gating_vs_edges_celeba}
\end{figure}

\begin{figure}[t]
    \centering
    \begin{subfigure}[b]{0.24\textwidth}
        \centering
        \includegraphics[width=\linewidth]{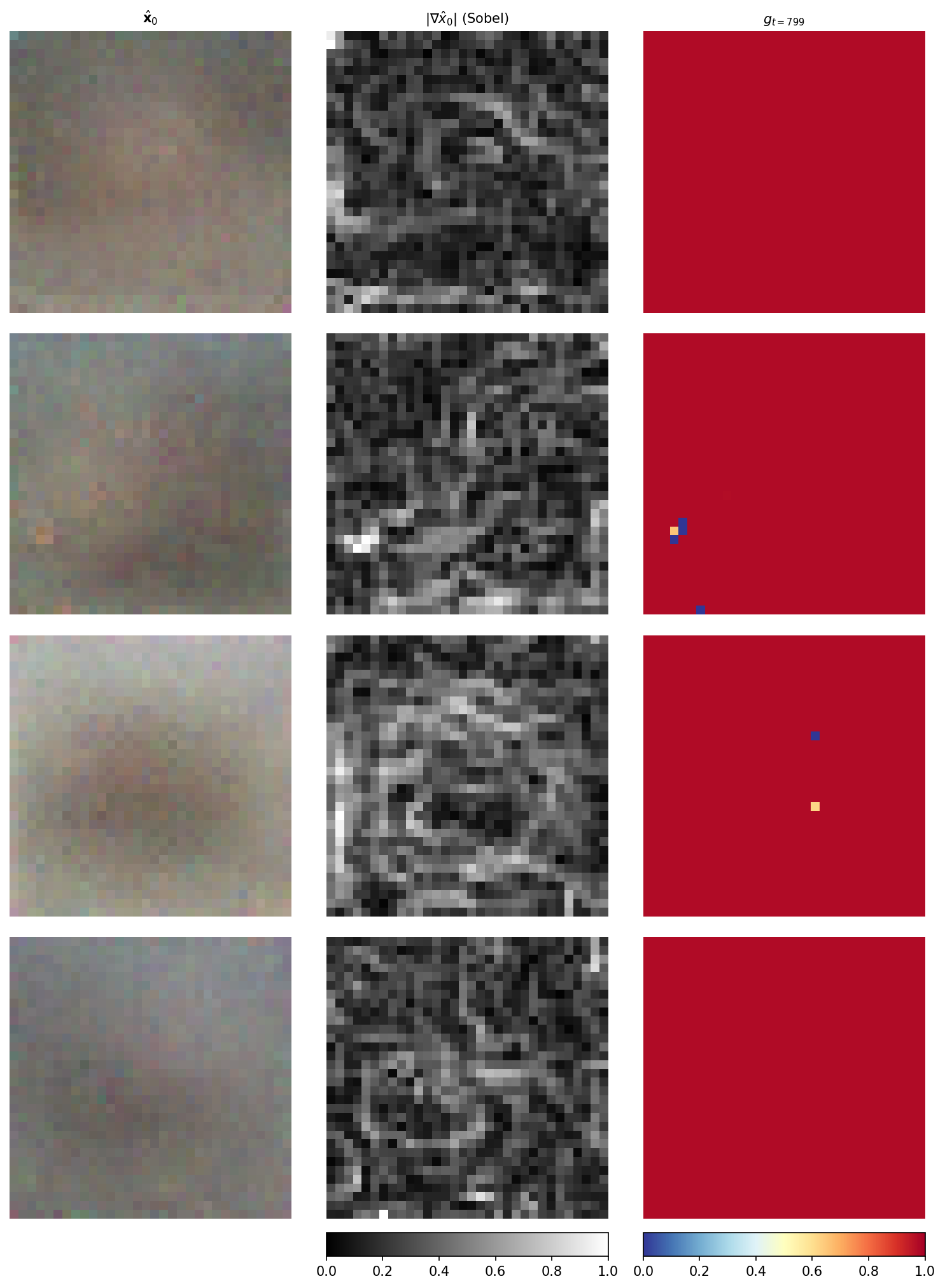}
        \caption{\footnotesize $t\!=\!800$}
        \label{fig:gating_vs_edges_cifar10ls_800}
    \end{subfigure}
    \begin{subfigure}[b]{0.24\textwidth}
        \centering
        \includegraphics[width=\linewidth]{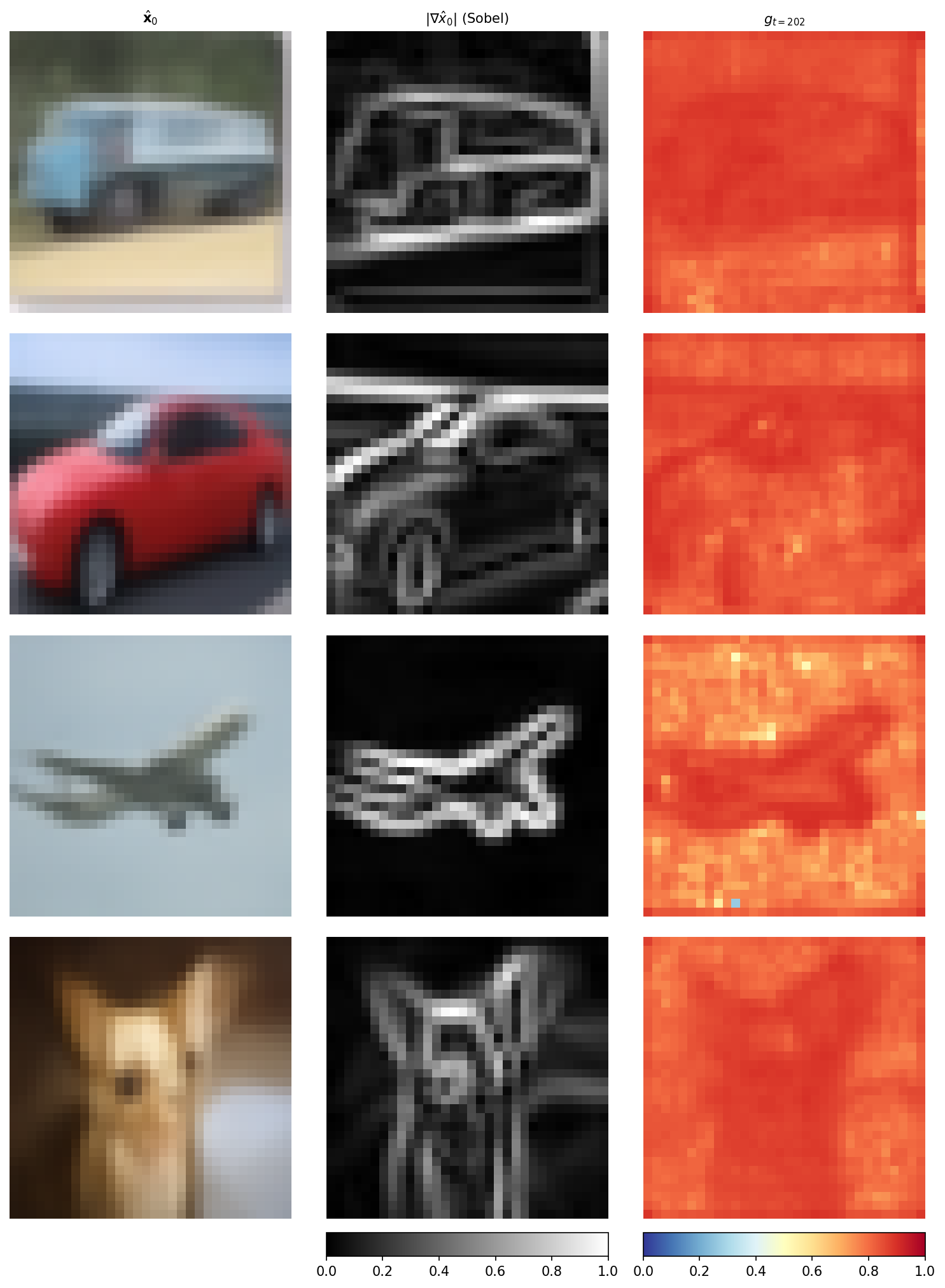}
        \caption{\footnotesize $t\!=\!200$}
        \label{fig:gating_vs_edges_cifar10ls_200}
    \end{subfigure}
    \begin{subfigure}[b]{0.24\textwidth}
        \centering
        \includegraphics[width=\linewidth]{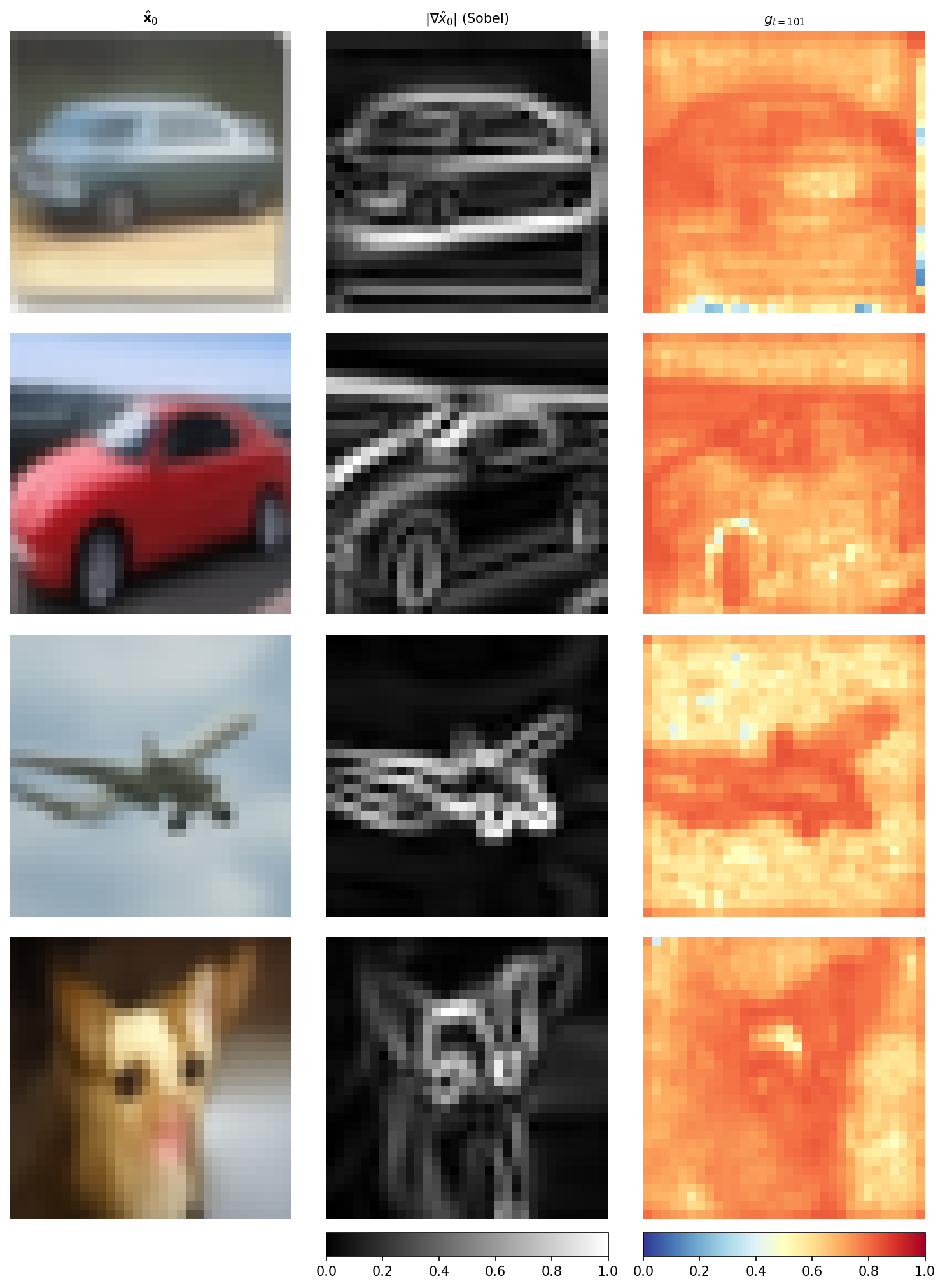}
        \caption{\footnotesize $t\!=\!100$}
    \label{fig:gating_vs_edges_cifar10ls_100}
    \end{subfigure}
    \begin{subfigure}[b]{0.24\textwidth}
        \centering
        \includegraphics[width=\linewidth]{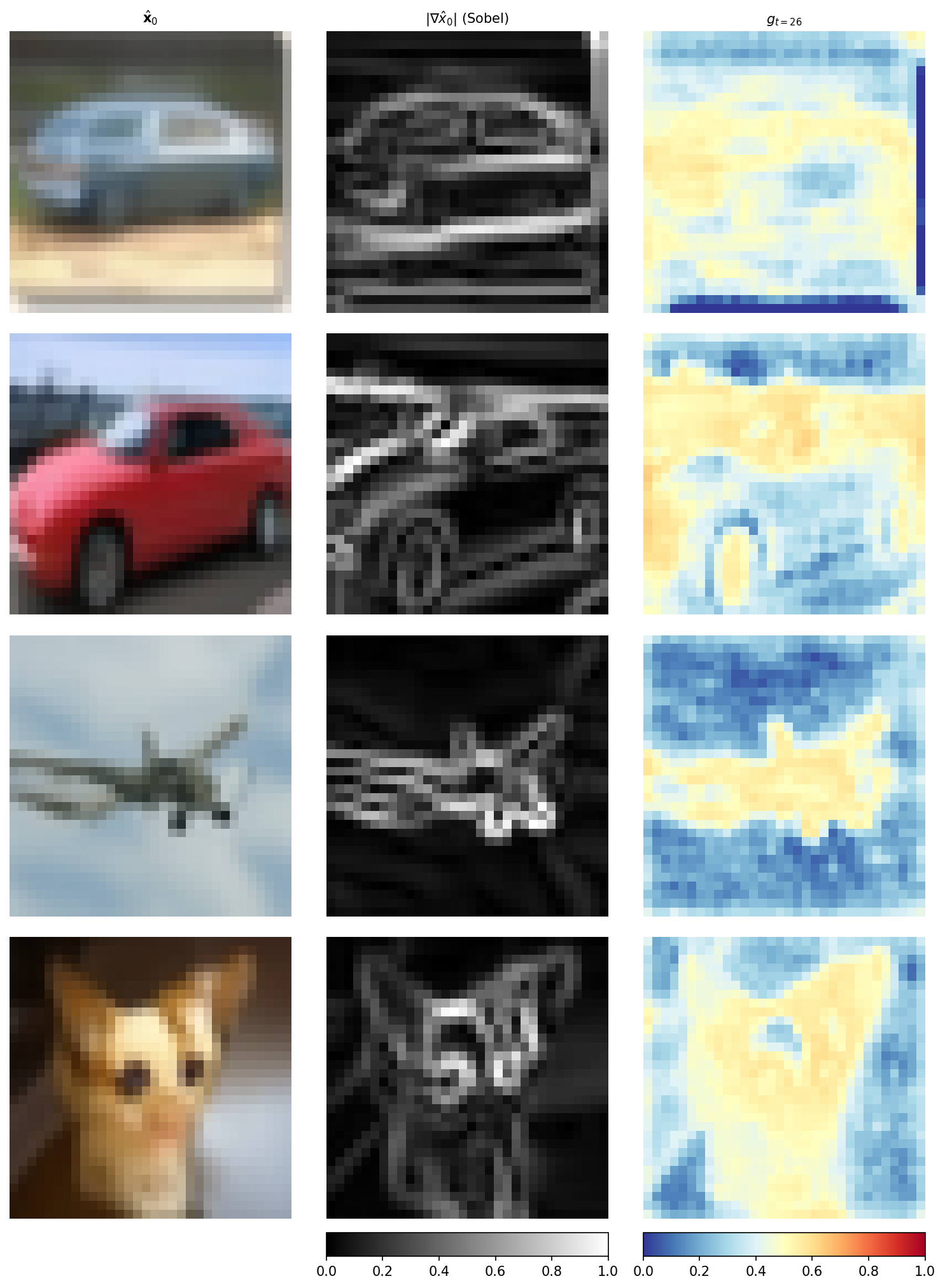}
        \caption{\footnotesize $t\!=\!25$}
                \label{fig:gating_vs_edges_cifar10ls_25}
    \end{subfigure}
    \caption{Per-pixel gating $g_{t,i}$ vs.\ Sobel edge map of $\hat{\mathbf{x}}_0$ on Cifar10 (LS). Each subplot: predicted image (left), Sobel magnitude (center), gating map (right, blue$\,{=}\,0$, red$\,{=}\,1$). The gating map progressively aligns with image structure without any edge-detection supervision.}
    \label{fig:app_gating_vs_edges_cifar10ls}
\end{figure}

\begin{figure}%[t]
    \centering
    \begin{subfigure}[b]{0.24\textwidth}
        \centering
        \includegraphics[width=\linewidth]{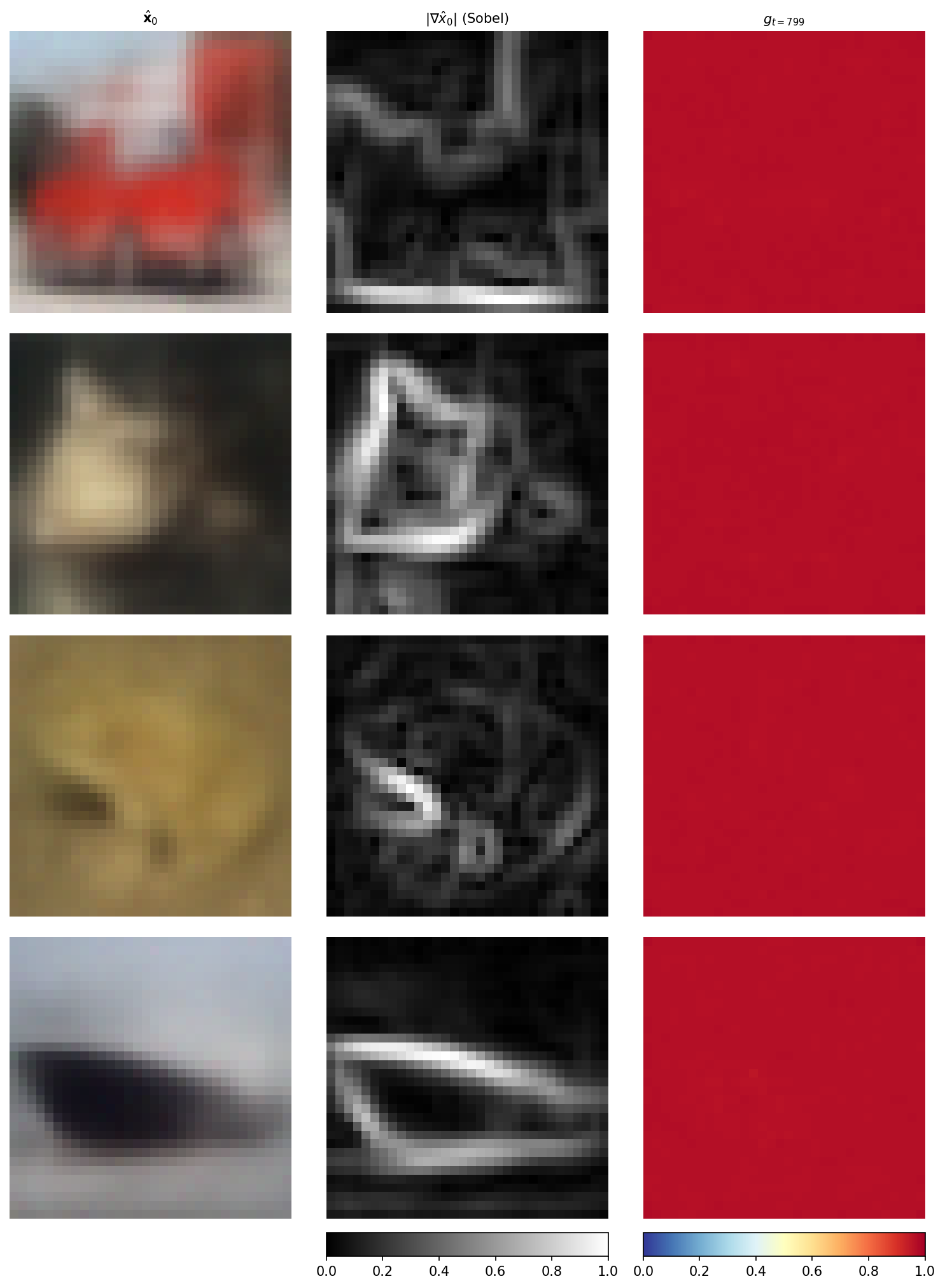}
        \caption{\footnotesize $t\!=\!800$}
        \label{fig:gating_vs_edges_cifar10cs_800}
    \end{subfigure}
    \begin{subfigure}[b]{0.24\textwidth}
        \centering
        \includegraphics[width=\linewidth]{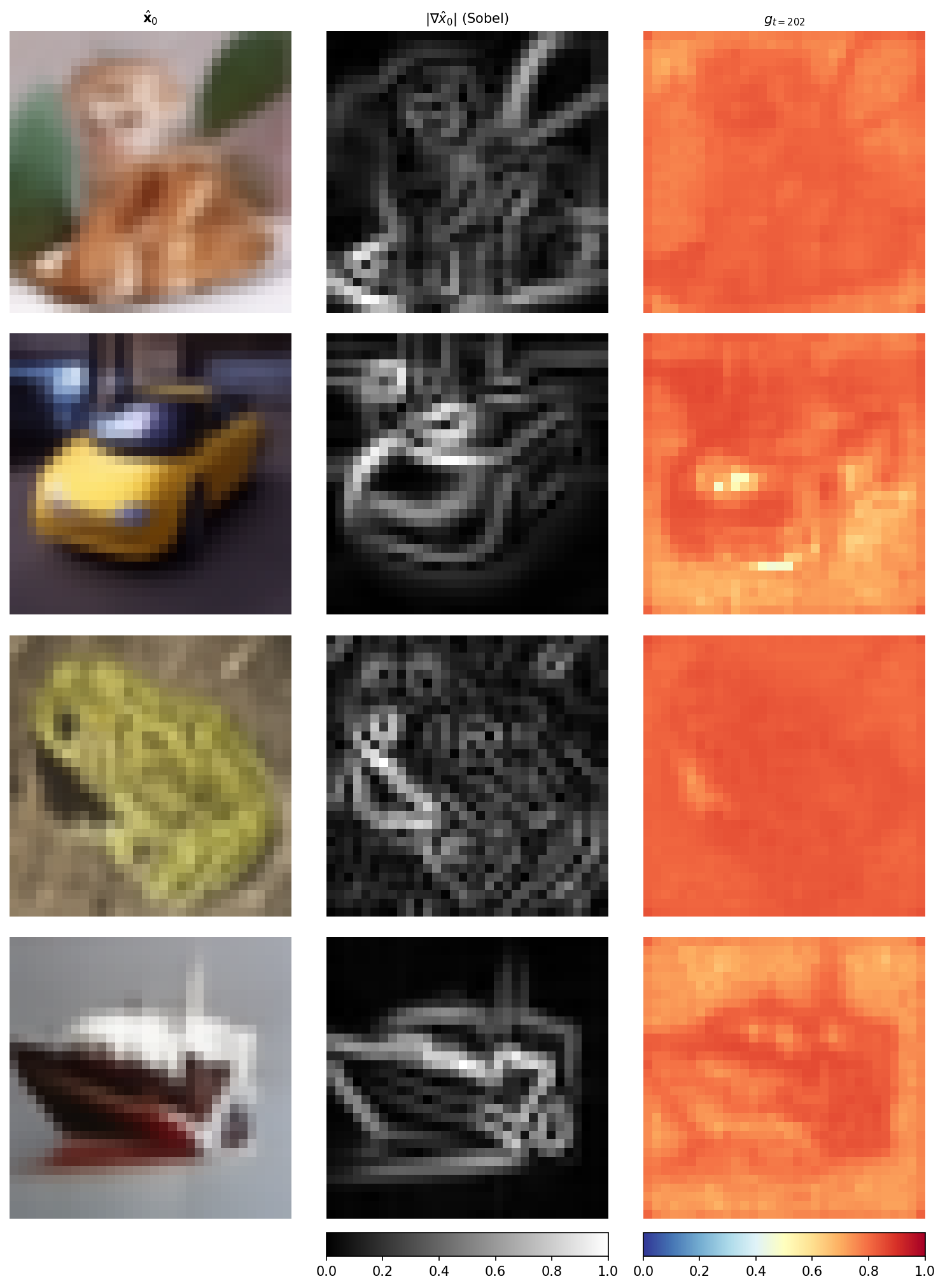}
        \caption{\footnotesize $t\!=\!200$}
        \label{fig:gating_vs_edges_cifar10cs_200}
    \end{subfigure}
    \begin{subfigure}[b]{0.24\textwidth}
        \centering
        \includegraphics[width=\linewidth]{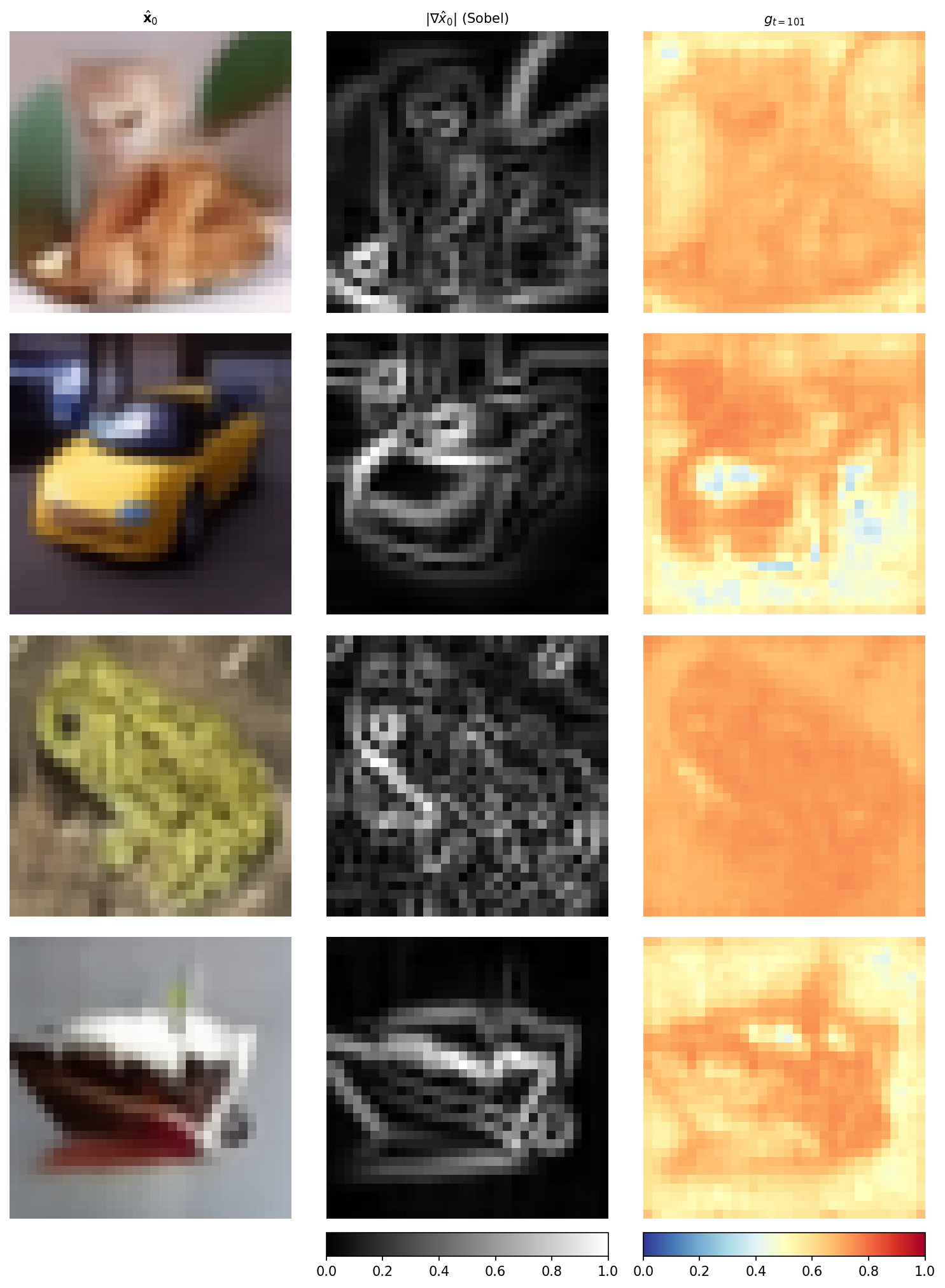}
        \caption{\footnotesize $t\!=\!100$}
    \label{fig:gating_vs_edges_cifar10cs_100}
    \end{subfigure}
    \begin{subfigure}[b]{0.24\textwidth}
        \centering
        \includegraphics[width=\linewidth]{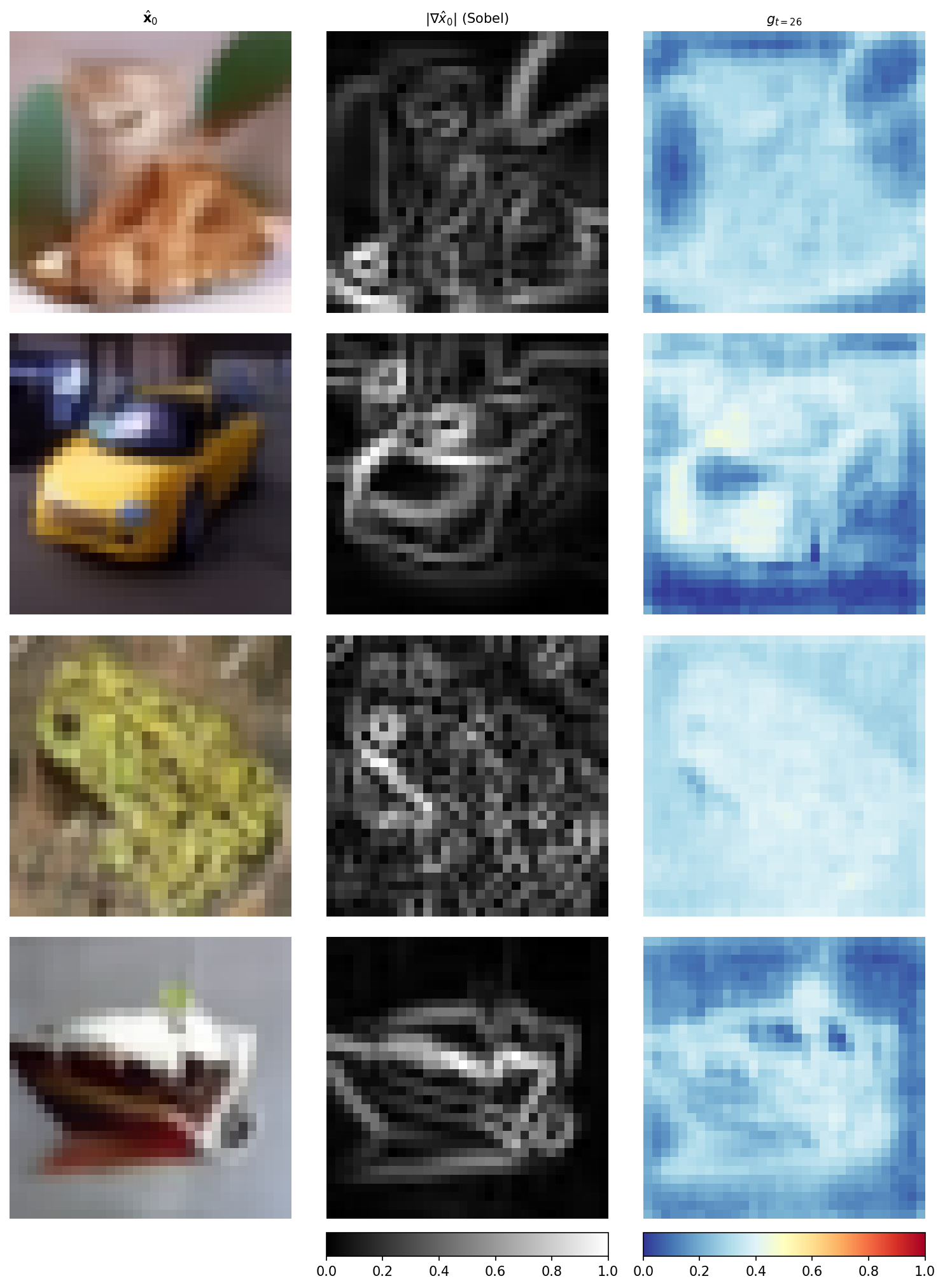}
        \caption{\footnotesize $t\!=\!25$}
                \label{fig:app_gating_vs_edges_cifar10cs_25}
    \end{subfigure}
    \caption{Per-pixel gating $g_{t,i}$ vs.\ Sobel edge map of $\hat{\mathbf{x}}_0$ on Cifar10 (CS). Each subplot: predicted image (left), Sobel magnitude (center), gating map (right, blue$\,{=}\,0$, red$\,{=}\,1$). The gating map progressively aligns with image structure without any edge-detection supervision.}
    \label{fig:app_gating_vs_edges_cifar10cs}
\end{figure}

\section{Extended Residual Assumption \& Gating Dynamics}
\label{app:gating_dynamics_extended}
 
This appendix provides additional analysis for Sections~\ref{sec:exp_residual}-\ref{sec:exp_gating_behavior} extending the discussion to all four datasets.

\subsection{Gating Dynamics}
\label{app:gating_dynamics}
Figure~\ref{fig:app_cross_dataset_histograms} presents gating histograms for $\tau \in \{10^{-4}, 10^{-3}, 10^{-2 }10^{-1}\}$. Despite substantial variation in image content and resolution (CIFAR10 at $32 \times 32$, LSUN Bedroom at $256 \times 256$), the qualitative behavior of $g_{t,i}$ remains consistent, exhibiting a three-stage pattern across all datasets. The distribution is unimodal near $1$ st early timesteps, flattens and spreads broadly at intermediate steps(indicating  spatial discrimination between confident and uncertain regions), and collapses near $0$ at late timesteps. This invariance  across dataset indicates that the spatial adaptivity of SANI is governed by the intrinsic information geometry of the denoising distribution, rather than by dataset-specific characteristics.
 \begin{figure}[t]
  \centering
  % ========== Row 1: cifar10_ls ==========
  \begin{minipage}{0.05\textwidth}
    \centering
    \rotatebox{90}{{Cifar10 (LS)}}
  \end{minipage}%
  \begin{minipage}{0.92\textwidth}
    \begin{subfigure}[b]{0.23\textwidth}
      \centering
      \includegraphics[width=\linewidth]{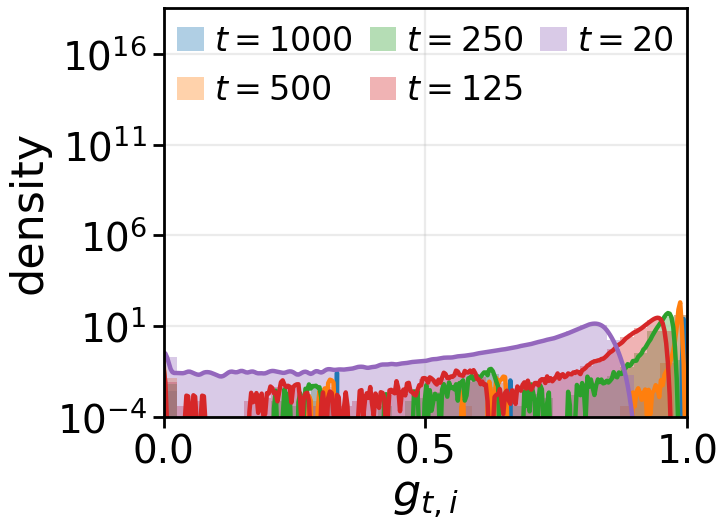}
      \label{fig:gating_histograms_a}
    \end{subfigure}%
    \begin{subfigure}[b]{0.23\textwidth}
      \centering
      \includegraphics[width=\linewidth]{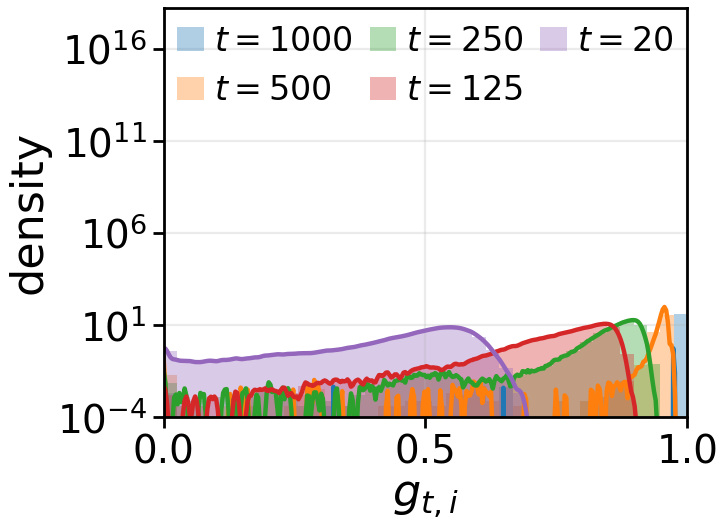}
      \label{fig:gating_histograms_b}
    \end{subfigure}%
    \begin{subfigure}[b]{0.23\textwidth}
      \centering
      \includegraphics[width=\linewidth]{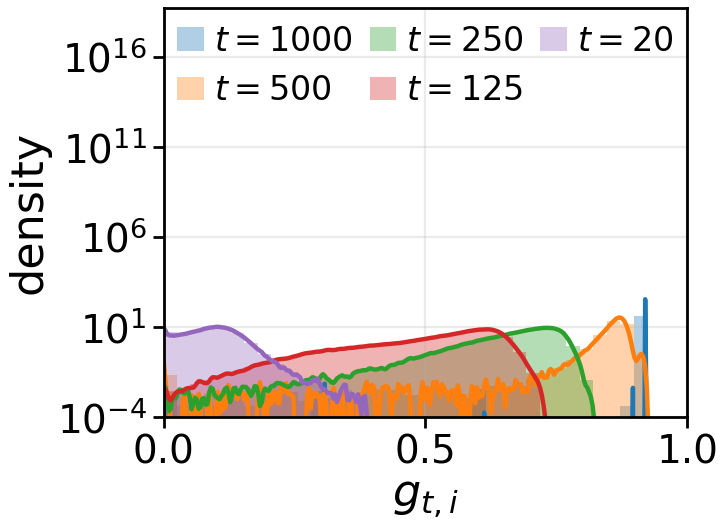}
      \label{fig:gating_histograms_c}
    \end{subfigure}%
    \begin{subfigure}[b]{0.23\textwidth}
      \centering
      \includegraphics[width=\linewidth]{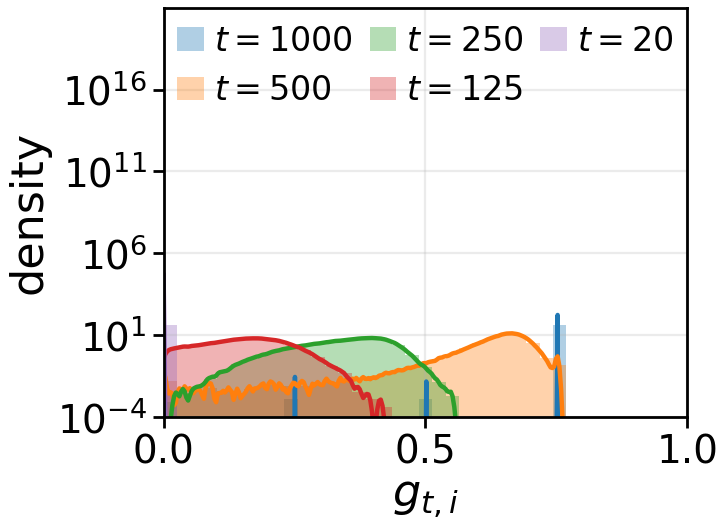}
      \label{fig:gating_histograms_d}
    \end{subfigure}
  \end{minipage}

  \vspace{1ex}  % vertical space between rows

  % ========== Row 2: cifar10_cs ==========
  \begin{minipage}{0.05\textwidth}
    \centering
    \rotatebox{90}{{Cifar10 (CS)}}
  \end{minipage}%
  \begin{minipage}{0.92\textwidth}
    \begin{subfigure}[b]{0.23\textwidth}
      \centering
      \includegraphics[width=\linewidth]{figs/gating_histogram/cifar10_cs/gating_histogram_cifar10_cs_c0.0001_K1000_all_1000.png}
      \label{fig:gating_histograms_e}
    \end{subfigure}%
    \begin{subfigure}[b]{0.23\textwidth}
      \centering
      \includegraphics[width=\linewidth]{figs/gating_histogram/cifar10_cs/gating_histogram_cifar10_cs_c0.001_K1000_all_1000.png}
      \label{fig:gating_histograms_f}
    \end{subfigure}%
    \begin{subfigure}[b]{0.23\textwidth}
      \centering
      \includegraphics[width=\linewidth]{figs/gating_histogram/cifar10_cs/gating_histogram_cifar10_cs_c0.01_K1000_all_1000.png}
      \label{fig:gating_histograms_g}
    \end{subfigure}%
    \begin{subfigure}[b]{0.23\textwidth}
      \centering
      \includegraphics[width=\linewidth]{figs/gating_histogram/cifar10_cs/gating_histogram_cifar10_cs_c0.1_K1000_all_1000.png}
      \label{fig:gating_histograms_h}
    \end{subfigure}
  \end{minipage}

  \vspace{1ex}

  % ========== Row 3: celeba64 ==========
  \begin{minipage}{0.05\textwidth}
    \centering
    \rotatebox{90}{{Celeba64}}
  \end{minipage}%
  \begin{minipage}{0.92\textwidth}
    \begin{subfigure}[b]{0.23\textwidth}
      \centering
      \includegraphics[width=\linewidth]{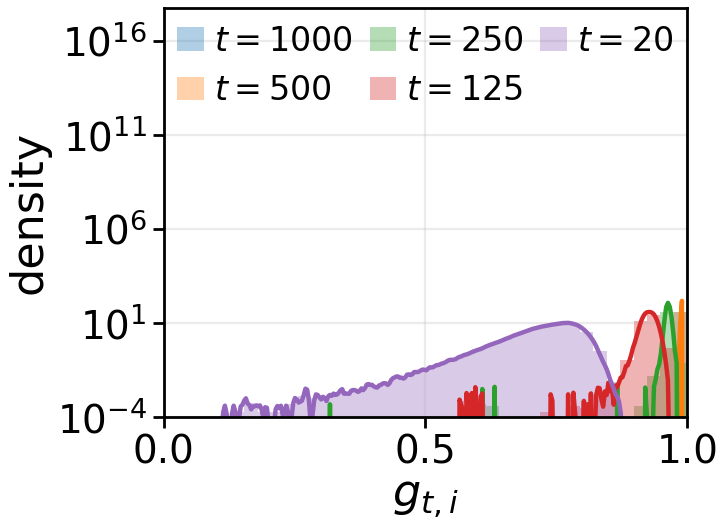}
      \label{fig:gating_histograms_i}
    \end{subfigure}%
    \begin{subfigure}[b]{0.23\textwidth}
      \centering
      \includegraphics[width=\linewidth]{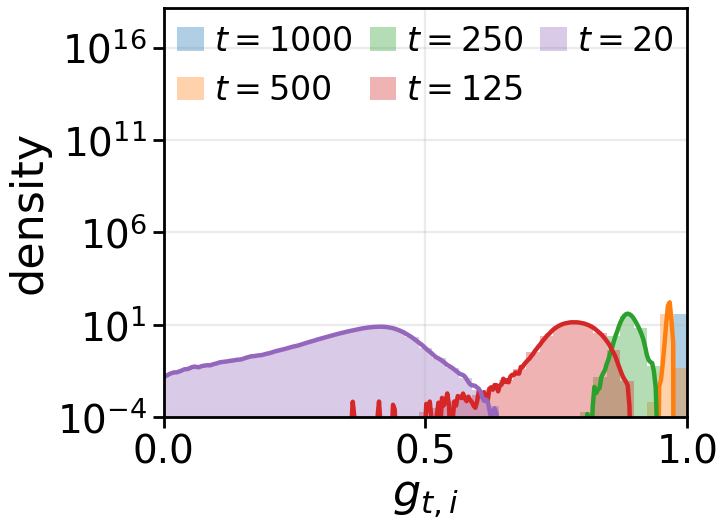}
      \label{fig:gating_histograms_j}
    \end{subfigure}%
    \begin{subfigure}[b]{0.23\textwidth}
      \centering
      \includegraphics[width=\linewidth]{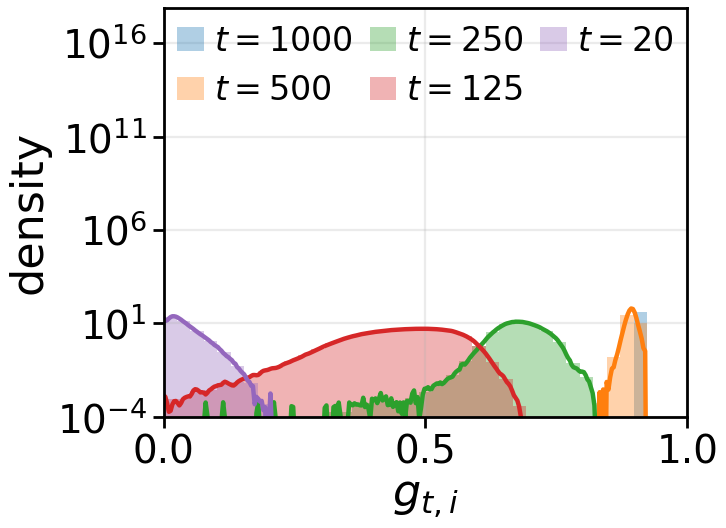}
      \label{fig:gating_histograms_k}
    \end{subfigure}%
    \begin{subfigure}[b]{0.23\textwidth}
      \centering
      \includegraphics[width=\linewidth]{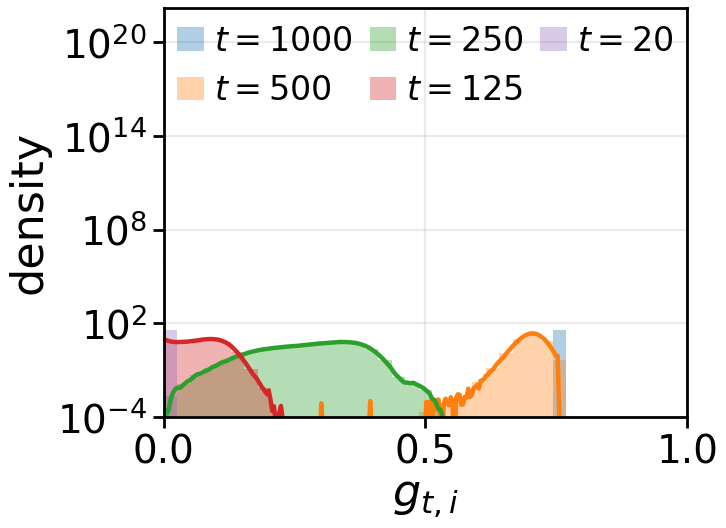}
      \label{fig:gating_histograms_l}
    \end{subfigure}
  \end{minipage}

  \vspace{1ex}

  % ========== Row 4: lsun_bedroom ==========
  \begin{minipage}{0.05\textwidth}
    \centering
    \rotatebox{90}{{LSUN}}
  \end{minipage}%
  \begin{minipage}{0.92\textwidth}
    \begin{subfigure}[b]{0.23\textwidth}
      \centering
      \includegraphics[width=\linewidth]{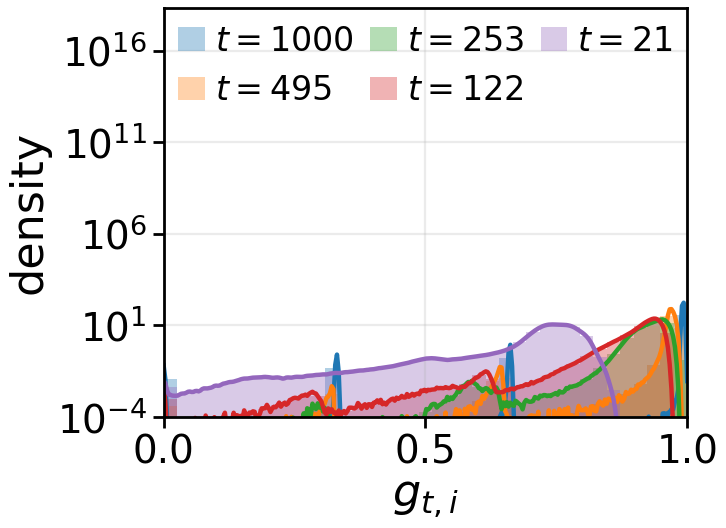}
      \caption{\footnotesize $\tau\!=\!10^{-4}$}
      \label{fig:gating_histograms_m}
    \end{subfigure}%
    \begin{subfigure}[b]{0.23\textwidth}
      \centering
      \includegraphics[width=\linewidth]{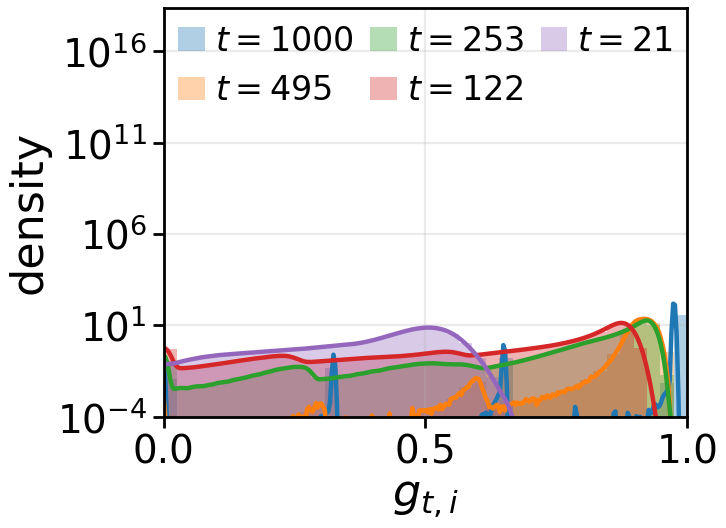}
      \caption{\footnotesize $\tau\!=\!10^{-3}$}
      \label{fig:gating_histograms_n}
    \end{subfigure}%
    \begin{subfigure}[b]{0.23\textwidth}
      \centering
      \includegraphics[width=\linewidth]{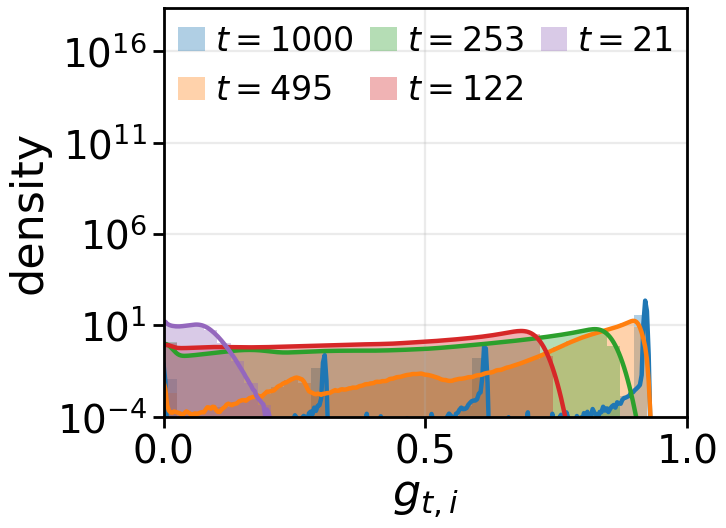}
      \caption{\footnotesize $\tau\!=\!10^{-2}$}
      \label{fig:gating_histograms_o}
    \end{subfigure}%
    \begin{subfigure}[b]{0.23\textwidth}
      \centering
      \includegraphics[width=\linewidth]{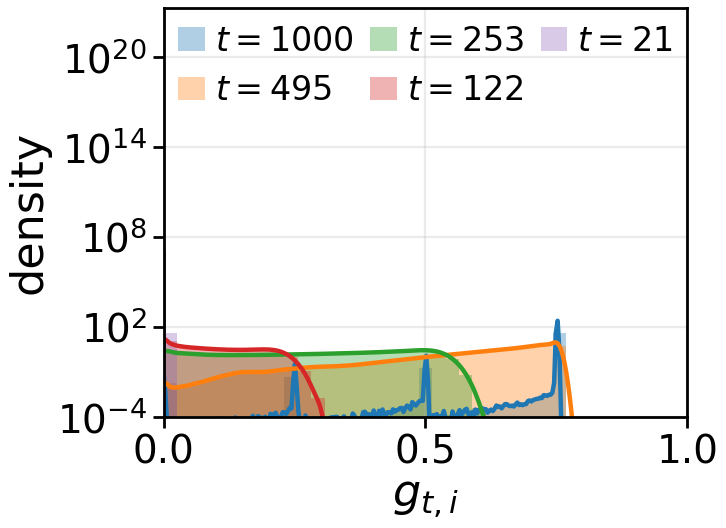}
      \caption{\footnotesize $\tau\!=\!10^{-1}$}
      \label{fig:gating_histograms_p}
    \end{subfigure}
  \end{minipage}
  \caption{Spatial distribution of gating values $g_{t,i}$ along the full reverse trajectory across datasets (rows), for four tolerances (columns). Each subplot overlays five timesteps $t \in \{1000, 500, 250, 125, 20\}$, with density on a logarithmic scale. Despite varying resolution and content, the temporal evolution of $g_{t,i}$ is qualitatively consistent.}
  %The distribution is a near-point mass at t = 1000 (spatially uniform uncertainty under pure noise), spreads over [0, 1] at intermediate timesteps (spatial discrimination between confident and uncertain regions), and collapses toward g = 0 at the final steps. Increasing τ by a decade translates this stochastic-to-deterministic transition earlier along the trajectory without changing its shape, as predicted by the dependence of the gate on τ /vi alone }
  \label{fig:app_cross_dataset_histograms}
\end{figure}

Sections~\ref{sec:exp_gating_behavior} analyzes the effecs of tolerance~$\tau$ and temporal gating dynamics independently. The first examines how~$\tau$ influences the distribution of~$g_{t,i}$ at fixed timesteps, while the second tracks the mean and variance of~$g_{t,i}$ over time for a single~$\tau$. \Cref{fig:app_gating_dynamics_tau} integrates both perspectives by plotting the mean gating value~$\bar{g}$ (solid line) with $\pm 1$ standard deviation (shaded band) throughout the reverse process for four tolerance values $\tau \in \{10^{-5}, 10^{-4}, 10^{-3}, 10^{-2}\}$, across all four datasets with $K = 1000$.

%\textbf{Shared structure across datasets.} The qualitative pattern is remarkably consistent.
All four~$\tau$ curves begin near $\bar{g} = 1$ at $t = T$, indicating global uncertainty, and decay monotonically toward $\bar{g} \approx 0$ as $t$ approaches $1$, reflecting increasing model confidence. Larger~$\tau$ values result in an earlier and steeper decay. For $\tau = 10^{-2}$ (yellow), $\bar{g}$ fails below $0.75$ by $t \approx 600$ on CIFAR10, areas for $\tau = 10^{-5}$ (dark blue), it remains above $0.95$ until $t \approx 100$. These results support the theoretical prediction from Eq.~\ref{eq:gating_closed} that the gating function transitions from stochastic to deterministic when the per-pixel variance~$v_i$ drops below~$\tau$, with larger~$\tau$ inducing this transition earlier.

%\textbf{Spatial heterogeneity from the standard deviation bands.}
The width of the $\pm\mathrm{std}$ band quantifies the spatial heterogeneity of the gating map at each timestep. At the beggining and end of the process ($t$ near $T$ or near $1$), the bands are narrow indicating that all pixels are either uniformly uncertain or uniformly confident. The bands reach their maximum width at intermediate timesteps, corresponding to the phase where some image regions have been resolved while others remain ambiguous. This peak in heterogeneity occurs earlier for larger~$\tau$ values, consistent with the observation that a more permissive tolerance initiates pixel discrimination sooner.%(the yellow band widens before the orange, which widens before the purple)

The ordering and monotonicity of the~$\tau$ curves are maintained across all datasets.
%\textbf{Dataset-dependent features.} On CIFAR10 with both schedules and on CelebA-64, the mean curves are smooth and the standard deviation bands are moderate, reflecting the relatively homogeneous spatial structure at $32 \times 32$ and $64 \times 64$ resolutions. In contrast, for LSUN Bedroom ($256 \times 256$), the bands are substantially wider and the mean curves display mild irregularity, indicative of the increased spatial complexity in high-resolution scenes. At a given timestep, the variance between a smooth wall region and a textured furniture boundary is considerably greater than in low-resolution images. Despite these quantitative differences, the ordering and monotonicity of the~$\tau$ curves are preserved across all datasets.

 \begin{figure}[h!]
    \centering
    \begin{subfigure}[b]{0.24\textwidth}
        \centering
        \includegraphics[width=\linewidth]{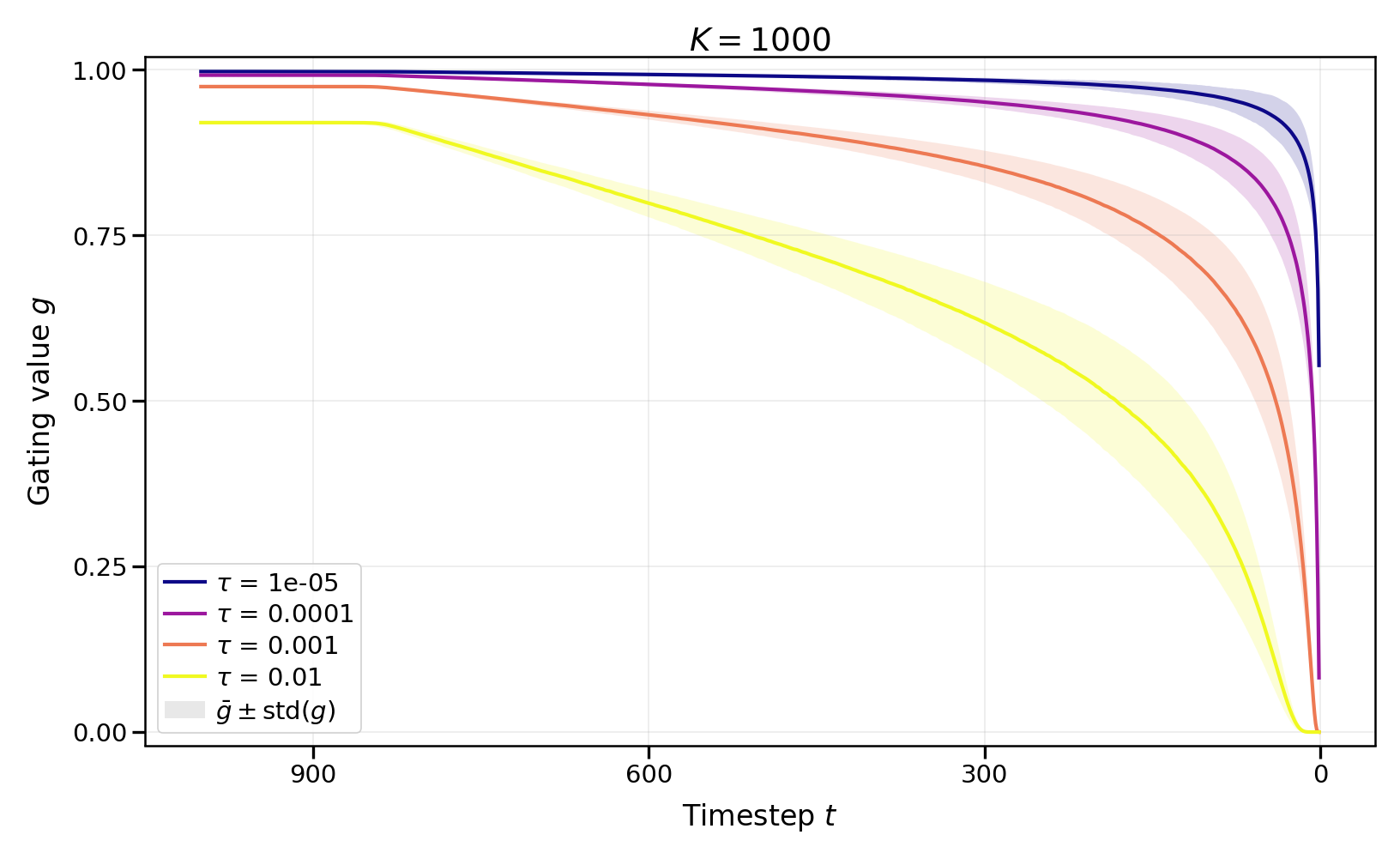}
        \caption{CIFAR10 (CS)}
        \label{fig:gating_dyn_cifar_cs}
    \end{subfigure}
    \hfill
    \begin{subfigure}[b]{0.24\textwidth}
        \centering
        \includegraphics[width=\linewidth]{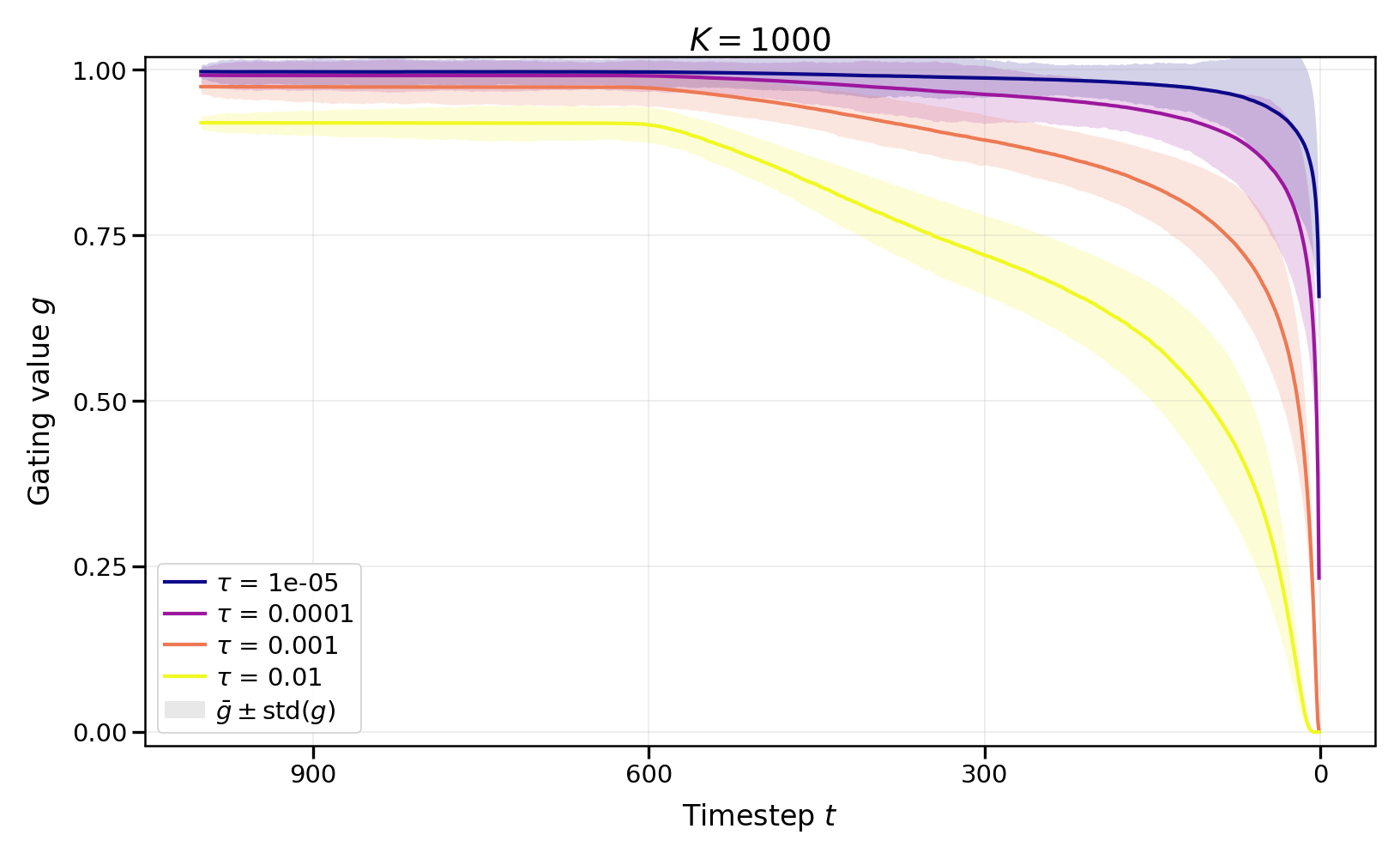}
        \caption{CIFAR10 (LS)}
        \label{fig:gating_dyn_cifar_ls}
    \end{subfigure}
    %\vspace{0.3cm}
    \begin{subfigure}[b]{0.24\textwidth}
        \centering
        \includegraphics[width=\linewidth]{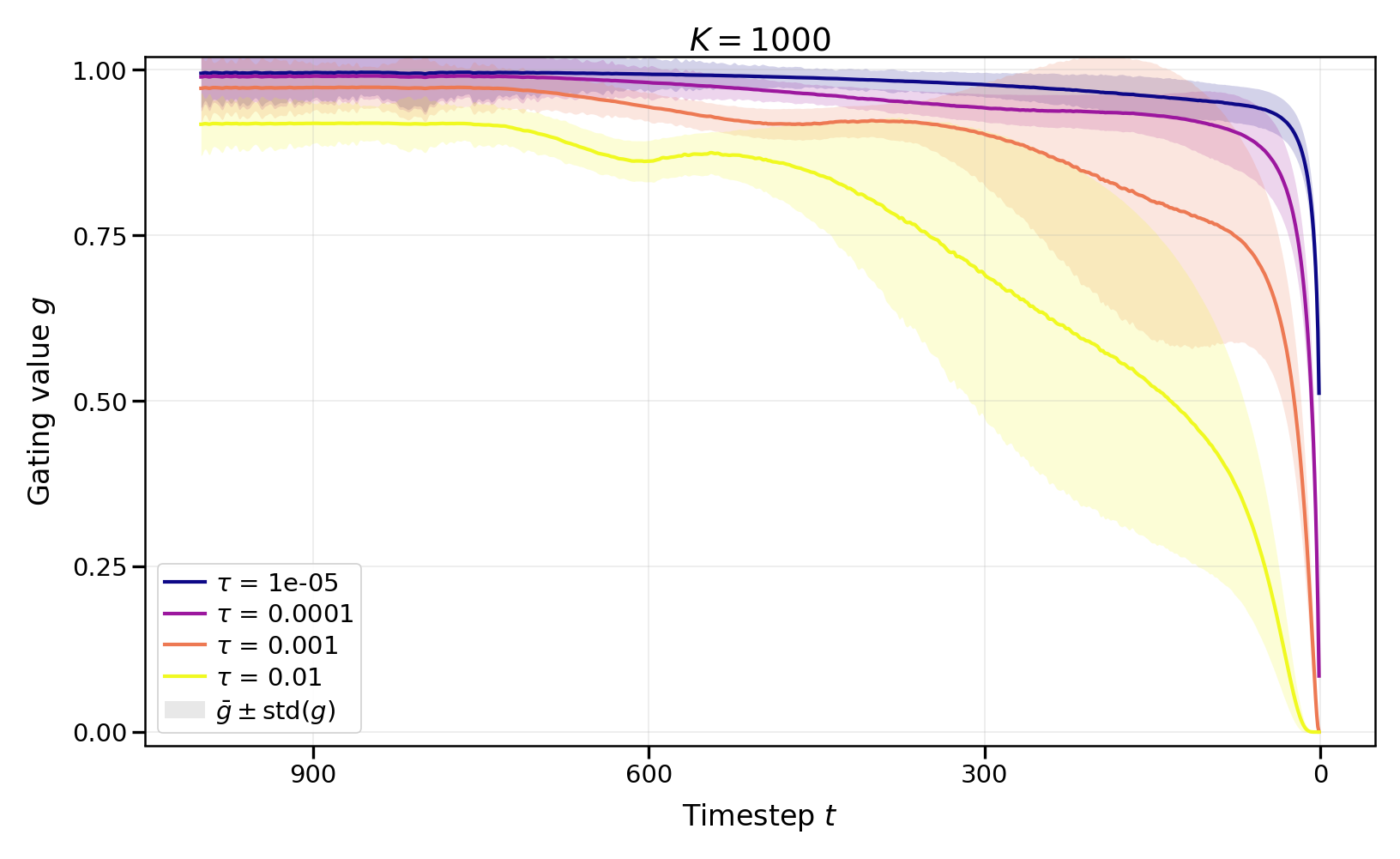}
        \caption{LSUN Bedroom}
        \label{fig:gating_dyn_lsun}
    \end{subfigure}
    \hfill
    \begin{subfigure}[b]{0.24\textwidth}
        \centering
        \includegraphics[width=\linewidth]{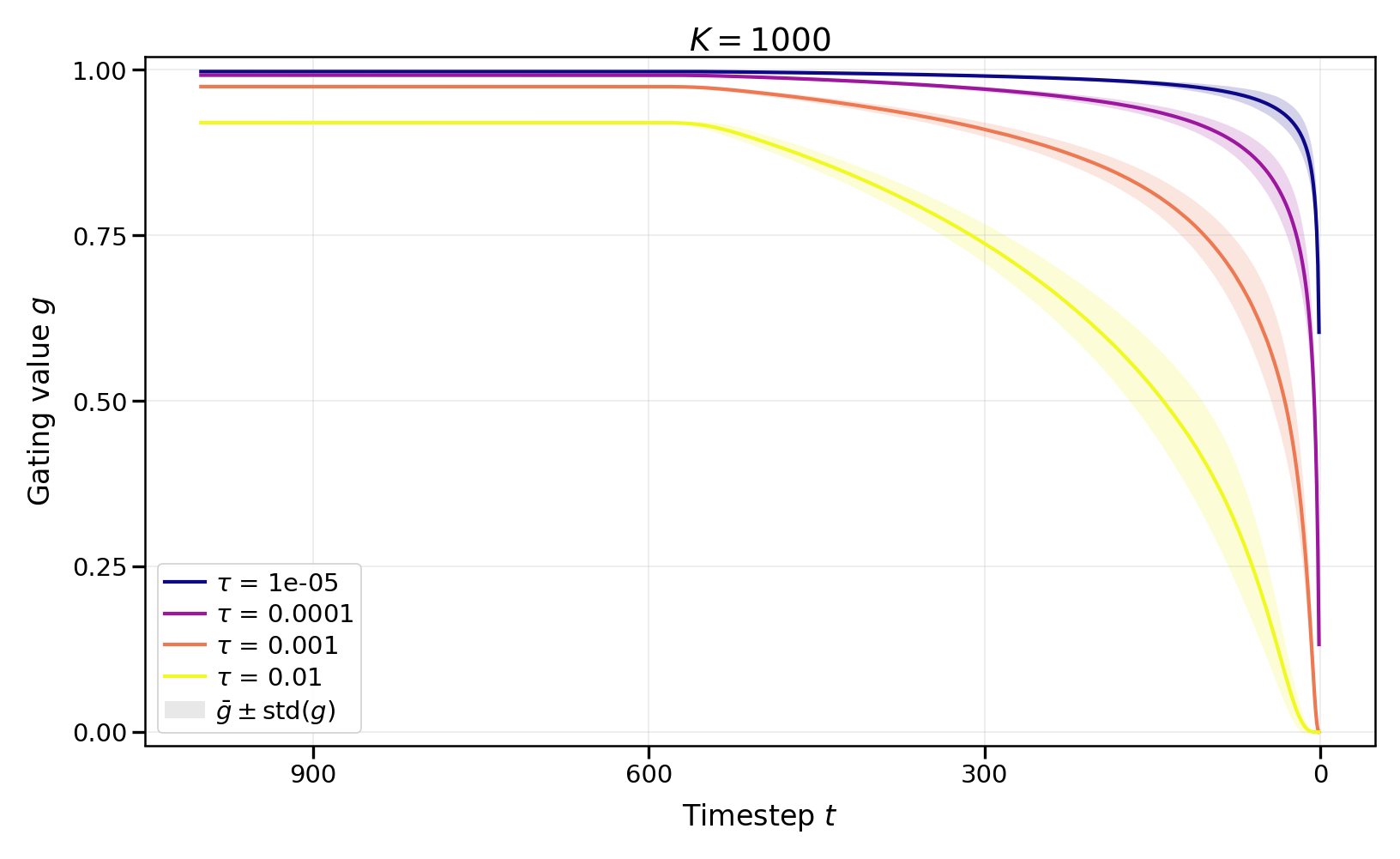}
        \caption{CelebA-64}
        \label{fig:gating_dyn_celeba}
    \end{subfigure}
    \caption{Mean gating value $\bar{g} \pm \mathrm{std}(g)$ over the
    reverse process ($K = 1000$) for four tolerance values.
    All datasets exhibit the same qualitative pattern: monotonic decay
    from~$1$ to~$0$, with larger~$\tau$ producing earlier transition.
    The standard deviation bands peak at intermediate timesteps where
    spatial heterogeneity is maximal. LSUN Bedroom shows the widest
    bands, reflecting the greater spatial complexity at
    $256 \times 256$ resolution.}
    \label{fig:app_gating_dynamics_tau}
\end{figure}
% The per-step statistics $(\bar g_t, \sigma_{g,t})$ summarize the spatial distribution of $g_{t,i}$ by its first two moments. \ref{fig:app_gating_histograms} shows the full distribution at five timesteps along the trajectory, for each of the four tolerances, on CIFAR-10 (CS). The pattern that the moments suggest is borne out exactly: at $t=1000$ the distribution is a near-point mass (all pixels uniformly uncertain under pure noise); at intermediate timesteps the support spreads broadly over $[0,1]$ (different regions resolve at different rates --- this is the regime where SANI departs most from a uniform sampler); at the final steps the mass collapses toward $g=0$ with a residual tail at the hardest pixels (edges and textures). Increasing $\tau$ by a decade translates this three-phase progression earlier along the trajectory without distorting its shape, consistent with the $\tau/v_i$ dependence of~\eqref{eq:gating_closed}. Histograms for the remaining three datasets exhibit the same qualitative behavior and are omitted for brevity.
% ---------------------------------------------------------------------------
\subsection{Residual-Distribution Assumption}
\label{app:residual_p
assumption}
% ---------------------------------------------------------------------------
\textbf{Protocol.}
At each timestep $t\in\{800,200,100,25\}$ we draw $500$ test images $\mathbf{x}_0\sim q(\mathbf{x}_0)$, and sample $\mathbf{x}_t\sim q(\mathbf{x}_t\mid\mathbf{x}_0)$. The Tweedie estimate $\hat{\mathbf{x}}_0$  and per-pixel variance~$v_i$ are computed, and standardized residuals $Z_i=(x_{0,i}-\hat{x}_{0,i})/\sqrt{v_i}$ are formed over all pixels and samples. Under Assumption~2 $Z_i\sim\mathcal{N}(0,1)$. We compare the empirical distribution of $Z_i$ to two unit-variance references: Gaussian and Laplace ($b=1/\sqrt{2}$), each moment-matched to zero mean and unit variance. 

Figure~\ref{fig:app_residual} compares the empirical density of $Z_i$ to the parametric references across multiple datasets. No single parametric family provides a consistent fit throughout the entire trajectory. At intermediate noise levels, the Laplace reference aligns most closely with the central mode. At high noise, the empirical peak surpasses the Laplace, indicating super-Laplacian concentration, while at low noise, the distribution bulk is closest to the Gaussian. Consequently, the Kolmogorov-Smirnov distance between the standardized residuals and the standard normal decreases monotonically as $t$ decreases.

%The empirical density is more sharply peaked than the Gaussian, with probability mass redistributed from the tails into the central peak. This is consistent with the heterogeneous structure of natural images.

\begin{figure}[h]
  \centering
  \begin{subfigure}[b]{0.49\linewidth}\centering
    \includegraphics[width=0.9\linewidth]{figs/cifar10_cs/residual_diagnostics_multidist_histograms_ks.png}
    \caption{CIFAR10 (CS)}
  \end{subfigure}\hfill
  \begin{subfigure}[b]{0.49\linewidth}\centering
    \includegraphics[width=0.9\linewidth]{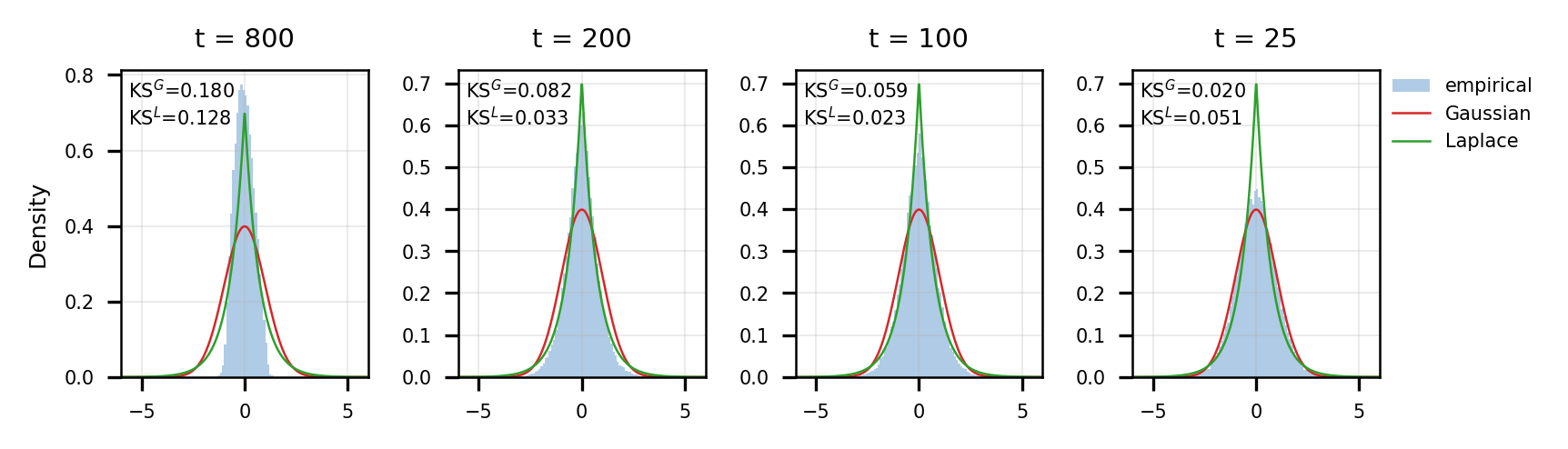}
        \caption{CIFAR10 (LS)}
  \end{subfigure}\hfill
  
  \begin{subfigure}[b]{0.49\linewidth}\centering
    \includegraphics[width=0.9\linewidth]{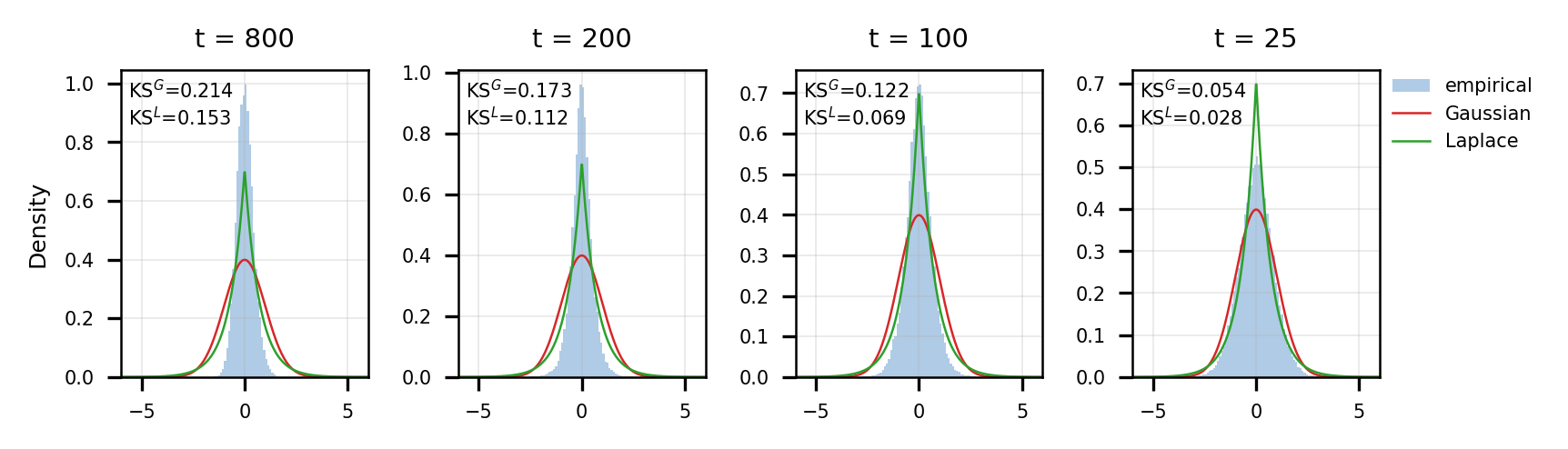}
        \caption{CelebA-$64$}
  \end{subfigure}\hfill
  \begin{subfigure}[b]{0.49\linewidth}\centering
    \includegraphics[width=0.9\linewidth]{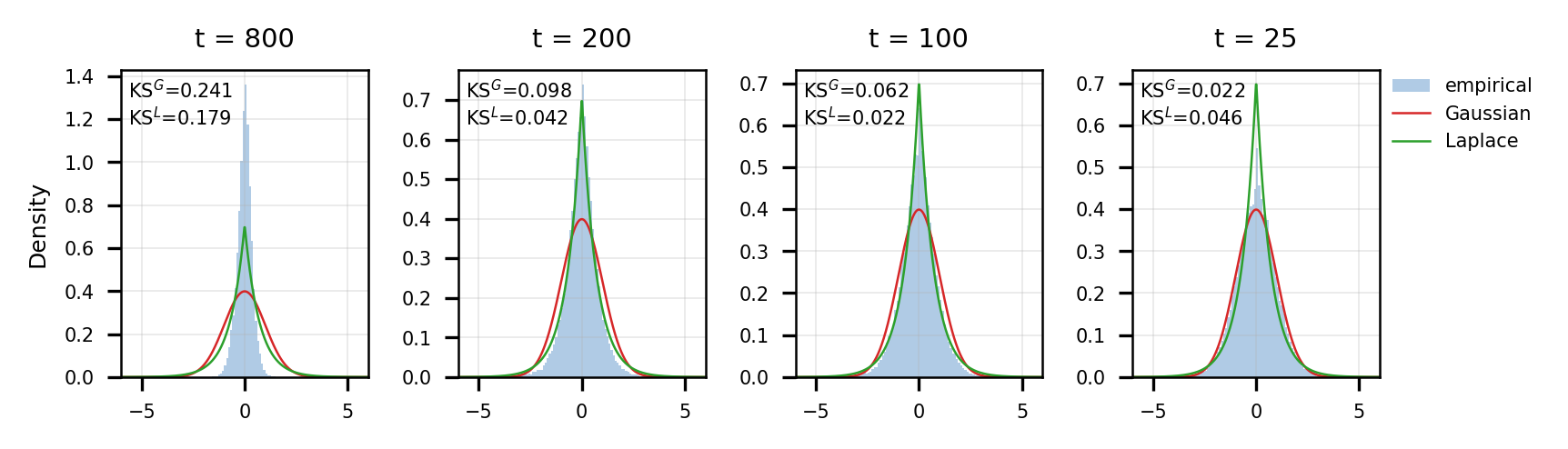}
    \caption{LSUN Bedroom }
  \end{subfigure}
  \caption{Empirical standardized residuals $Z_i = R_i/\sqrt{v_i}$ against Gaussian and Laplace residual distributions (all standardised to unit variance) on various datasets. Each subplot visualize  density histograms at four reverse-process timesteps.}
  \label{fig:app_residual}
\end{figure}

% ===========================================================================
% ===========================================================================
%  CORRESPONDING APPENDICES FOR SECTION 5.5
%  - app:nll        : Likelihood evaluation (bound + clamping + table)
% ===========================================================================

% ---------------------------------------------------------------------------
\section{Likelihood Evaluation}
\label{app:nll}
% ---------------------------------------------------------------------------
 
This appendix provides a detailed description of the likelihood computation referenced in~\Cref{sec:exp_quantitative}. We report variational upper bounds on $-\log p_\theta(\mathbf{x}_0)$ in bits per dimension (BPD) and evaluate  along the same timestep sub-sequence used during sampling. This ensures that both likelihood and sample quality assessed under consistent schedules.
 
Following the discrete-time decomposition of~\cite{ho20,nichol2021improved}:
\begin{align}\label{eq:app_nll_bound}
  \mathcal{L}
  \;=\;
  &\underbrace{\mathrm{KL}\!\bigl(q(\mathbf{x}_T\mid\mathbf{x}_0)\,\|\,\mathcal{N}(\mathbf{0},\mathbf{I})\bigr)}_{\text{prior term}}
  \;+\;\sum_{(s,t)}\mathrm{KL}\!\bigl(q(\mathbf{x}_s\mid\mathbf{x}_t,\mathbf{x}_0)\,\|\,p_\theta^{\mathrm{SANI}}(\mathbf{x}_s\mid\mathbf{x}_t)\bigr) \notag\\
  &-\;\mathbb{E}_q\bigl[\log p_\theta^{\mathrm{SANI}}(\mathbf{x}_0\mid\mathbf{x}_1)\bigr].
\end{align}
Here, the summation is taken over adjacent pairs $(s,t)$ in the sampling sub-sequence. The true reverse posterior $q(\mathbf{x}_s\mid\mathbf{x}_t,\mathbf{x}_0)$ depends solely on the forward process and remains unchanged from the standard DDPM~\citep{ho20}. It is a Gaussian distribution with mean $\tilde\mu_{t,i}(\mathbf{x}_t,x_{0,i})$ and variance $\tilde\beta_t$.
 
The model transition $p_\theta^{\mathrm{SANI}}(\mathbf{x}_s\mid\mathbf{x}_t)$ utilizes the SANI mean and per-pixel variance as defined in~\Cref{prop:update}:
\[
  \mu_{p,i}^{\mathrm{SANI}}=\sqrt{\bar\alpha_{t-1}}\,\hat{x}_{0,i}+\sqrt{1-\bar\alpha_{t-1}-\Sigma_{t,i}^2}\,\epsilon_{\theta,i},
  \qquad
  (\sigma_{p,i}^{\mathrm{SANI}})^2=\Sigma_{t,i}^2=c_t^2\,v_i+g_{t,i}\,\tilde\beta_t.
\]
Each KL term in Eq.~\ref{eq:app_nll_bound} thus decomposes coordinate-wise between two diagonal Gaussians. The terminal term at $t=1$ corresponds to the log-likelihood of $\mathbf{x}_0$ under a discretized Gaussian with the standard fixed variance $\tilde\beta_1$, consistent with~\cite{ho20}.
 
\paragraph{Clamping for Near-Deterministic Pixels}
 
For pixels routed to deterministic dynamics by the gate ($g_{t,i}\approx 0$) and exhibiting low per-pixel posterior variance ($v_i\approx 0$), the model variance $\Sigma_{t,i}^2$ collapses, whereas the true posterior maintains strictly positive variance $\tilde\beta_t$. This results in an unbounded KL divergence, a degeneracy also observed in DDIM at $\eta=0$, where the ELBO is similarly undefined.
 
To obtain a finite bound, we clamp the model variance per pixel and per step:
\begin{equation}\label{eq:app_nll_clamp}
  (\sigma_{p,i}^{\mathrm{SANI}})^2\;\gets\;\max\bigl(\Sigma_{t,i}^2,\,\tilde\beta_t\bigr).
\end{equation}
This approach leaves stochastic pixels unaffected ($\Sigma_{t,i}^2\geq\tilde\beta_t$ whenever $g_{t,i}$ or $\tilde\beta_t$ alone reaches $\tilde\beta_t$, since $c_t^2 v_i\geq 0$ is always non-negative) and reduces near-deterministic pixels to the standard DDPM small-variance reverse process. In this case, the contribution to Eq.~\ref{eq:app_nll_bound} matches exactly with the DDPM ELBO term at that pixel. The clamping affects only the likelihood computation, while sampling continuous to use the unclamped $\Sigma_{t,i}^2$.
 
The bound is averaged over the test set and converted to bits per dimension by dividing by $D\log 2$.
 
\paragraph{\textbf{Results}.} \Cref{tab:app_nll} presents NLL results on CIFAR-10 (LS), CIFAR-10 (CS), and CelebA-64 for both Gaussian and laplace residual assumption. Together with the FID results (\Cref{tab:fid_all}), the NLL evaluation demonstrates that spatial gating does not compromise likelihood in favor of sample quality. Across all steps counts and dataset, SANI remains competitive with, or close to, the variance-learning baselines, while also providing an interpretable per-pixel uncertainty diagnostic as shown in~\cref{fig:gating_vs_edges_lsun}.
%
%SANI is within $0.07$ BPD of the strongest variance-learning baseline on CIFAR-10 at all $K\leq 100$, and within $0.02$ BPD on CelebA-64 at $K\leq 50$. The gap widens to $0.16$ BPD at $K=100$ on CelebA-64 ($\mathrm{SANI}^G$: $3.15$, OCM-DDPM: $3.13$), consistent with the FID pattern of~\ref{sec:exp_fid}: a single tolerance~$\tau$ becomes suboptimal at high step counts, over-suppressing stochasticity that would tighten the bound. Importantly, the gating clamp does not artificially favor SANI: it pads near-deterministic pixels with the DDPM term rather than removing them, which conservatively over-estimates the bound.

\begin{table}[h]
\centering
\caption{Negative log-likelihood ($\downarrow$, bits per dimension) across datasets and sampling steps. \textbf{Bold}: best overall. \underline{Underline}: second-best. SANI is competitive with the variance-learning baselines at every step count.}
\label{tab:app_nll}
\tiny
\setlength{\tabcolsep}{2.5pt}
\begin{tabular}{lrrrrr|rrrrr|rrrrr}
\toprule
& \multicolumn{5}{c}{\textbf{CIFAR-10 (LS)}} & \multicolumn{5}{c}{\textbf{CIFAR-10 (CS)}} & \multicolumn{5}{c}{\textbf{CelebA $64\times 64$}} \\
\cmidrule(lr){2-6} \cmidrule(lr){7-11} \cmidrule(lr){12-16}
 \textbf{Method} & 10 & 25 & 50 & 100 & 200 & 10 & 25 & 50 & 100 & 200 & 10 & 25 & 50 & 100 & 200 \\
\midrule
DDPM, $\tilde{\beta}$  & 74.95 & 24.98 & 12.01 & 7.08 & 5.03 & 75.96 & 24.94 & 11.96 & 7.04 & 4.95 & 33.42 & 13.09 & 7.14 & 4.60 & 3.45 \\
DDPM, $\beta$       & 6.99 & 6.11 & 5.44 & 4.86 & 4.39 & 6.51 & 5.55 & 4.92 & 4.41 & 4.03 & 6.67 & 5.72 & 4.98 & 4.31 & 3.74 \\
A-DDPM              & 5.47 & 4.79 & 4.38 & 4.07 & 3.84 & 5.08 & 4.45 & 4.09 & 3.83 & 3.64 & \underline{4.54} & 3.89 & 3.48 & 3.16 & 2.92 \\
NPR-DDPM            & 5.40 & \underline{4.64} & \textbf{4.25} & \underline{3.98} & 3.79 & 5.03 & 4.33 & 3.99 & 3.76 & 3.59 & \textbf{4.46} & \textbf{3.78} & \textbf{3.40} & 3.11 & \underline{2.89} \\
SN-DDPM            & 30.79 & 11.83 & 7.13 & 5.24 & 4.39 & 90.85 & 19.81 & 9.72 & 6.72 & 5.58 & 18.09 & 8.05 & 5.29 & 4.05 & 3.40 \\
OCM-DDPM            & \underline{5.32} & \textbf{4.63} & \textbf{4.25} & \textbf{3.97} & \textbf{3.78} & \underline{4.99} & \underline{4.34} & \underline{3.99} & \underline{3.76} & \underline{3.59} & 4.69 & \underline{3.86} & \underline{3.43} & \textbf{3.13} & \textbf{2.90} \\
$\mathrm{SANI}^G$   & \textbf{5.31} & 4.67 & \underline{4.31} & 4.04 & \underline{3.85} & \textbf{4.97} & \textbf{4.39} & \textbf{4.07} & \textbf{3.83} & \textbf{3.66} & 4.67 & 3.88 & 3.45 & \underline{3.15} & 2.95 \\
$\mathrm{SANI}^L$   & 5.34 & 4.68 & 4.34 & 4.06 & 3.87 & 5.04 & 4.40 & 4.08 & 3.82 & 3.68 & 4.71 & 3.87 & 3.48 & \underline{3.15} & 2.97 \\
\bottomrule
\end{tabular}
\end{table}

% ===========================================================================
% ------------------------------------------------------------
\section{Generated Samples}
\label{app:extended_samples}
% ------------------------------------------------------------
Figures~\ref{fig:celebA}-\ref{fig:cifar10_cs} display generated samples produced by SANI with varying numbers of sampling steps $K$.

\begin{figure}[!t]
  \centering
  % --- First Subfigure ---
    \begin{subfigure}[b]{0.27\textwidth}
    \centering
    \includegraphics[width=\linewidth]{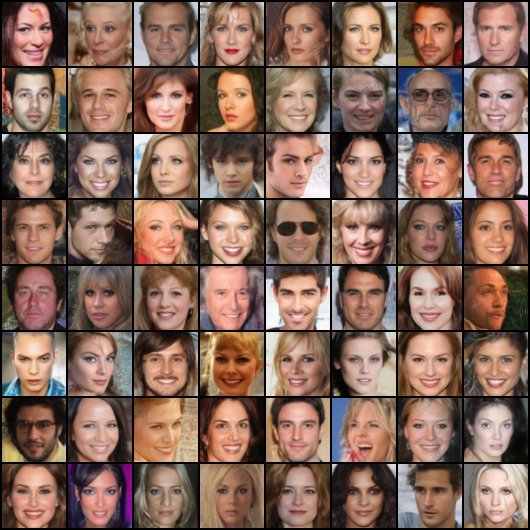}
    \caption{$K=10$}
    %\label{fig:gamma_optimal_gated_v2_coupled}
  \end{subfigure}
    \begin{subfigure}[b]{0.27\textwidth}
    \centering
    \includegraphics[width=\linewidth]{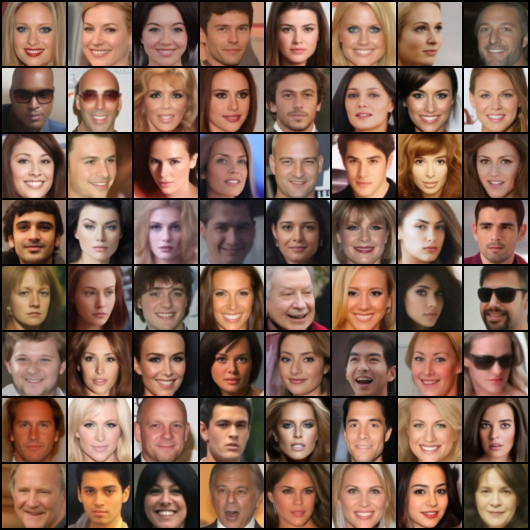}
    \caption{$K=25$}
    %\label{fig:gamma_optimal_gated_v2_coupled}
  \end{subfigure}
  \begin{subfigure}[b]{0.27\textwidth}
    \centering
    \includegraphics[width=\linewidth]{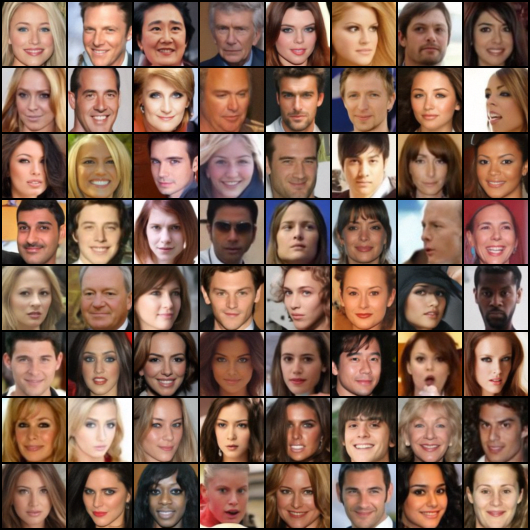}
    \caption{$K=50$}
    %\label{fig:gamma_optimal}
  \end{subfigure}
  %\hfill 
  
  % --- Third Subfigure ---
  \begin{subfigure}[b]{0.27\textwidth}
    \centering
    \includegraphics[width=\linewidth]{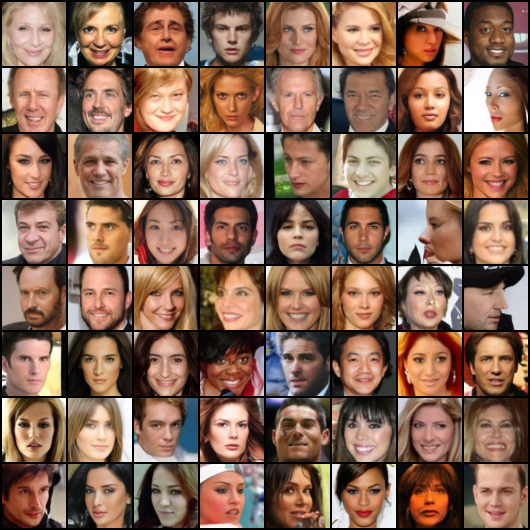}
    \caption{$K=100$}
    %\label{fig:gamma_small}
  \end{subfigure}
    %\hfill 
    % --- Third Subfigure ---
  \begin{subfigure}[b]{0.27\textwidth}
    \centering
    \includegraphics[width=\linewidth]{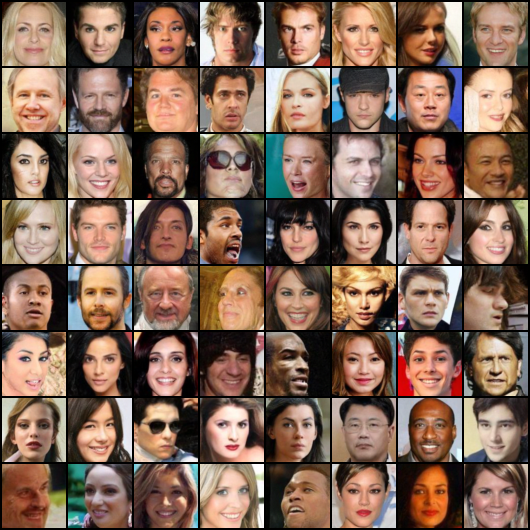}
    \caption{$K=200$}
    %\label{fig:gamma_small}
  \end{subfigure}
    %\hfill 
    % --- Third Subfigure ---
  \begin{subfigure}[b]{0.27\textwidth}
    \centering
    \includegraphics[width=\linewidth]{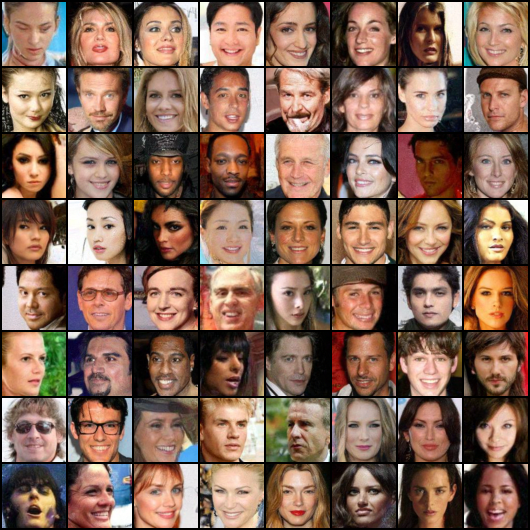}
    \caption{$K=1000$}
    %\label{fig:gamma_small}
  \end{subfigure}
    %\hfill 
  \caption{CelebA ($64\times64$). Generated samples using SANI using varying number of timesteps.}%[Comparison of generated samples on CelebA ($64\times64$) using different Gamma schedules. From left to right: Big, Optimal Gated, and Small.}
  \label{fig:celebA}
\end{figure}

%--------------------------------------------------------------------
\begin{figure}[htbp]
  \centering
  % --- First Subfigure ---
  \begin{subfigure}[b]{0.19\textwidth}
    \centering
    \includegraphics[width=\linewidth]{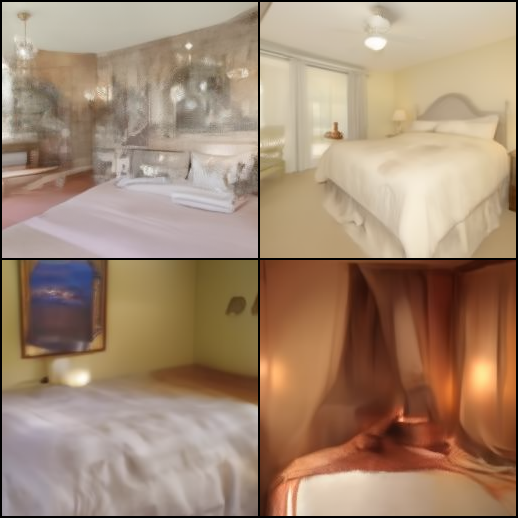}
    \caption{$K=25$}
    %\label{fig:gamma_big}
  \end{subfigure}
  \hfill 
  % --- Second Subfigure ---
  \begin{subfigure}[b]{0.19\textwidth}
    \centering
    \includegraphics[width=\linewidth]{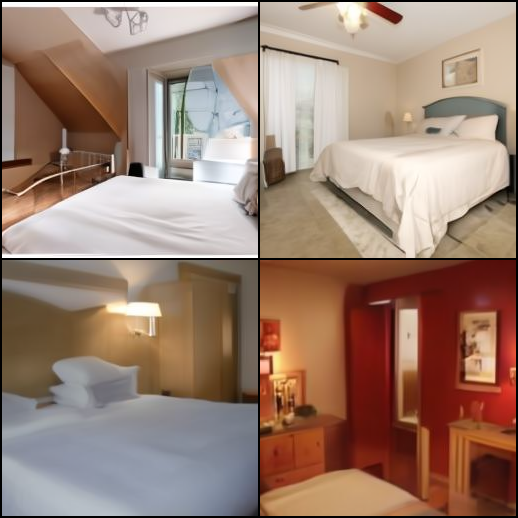}
    \caption{$K=50$}
    %\label{fig:gamma_optimal}
  \end{subfigure}
  \hfill
  % --- Third Subfigure ---
  \begin{subfigure}[b]{0.19\textwidth}
    \centering
    \includegraphics[width=\linewidth]{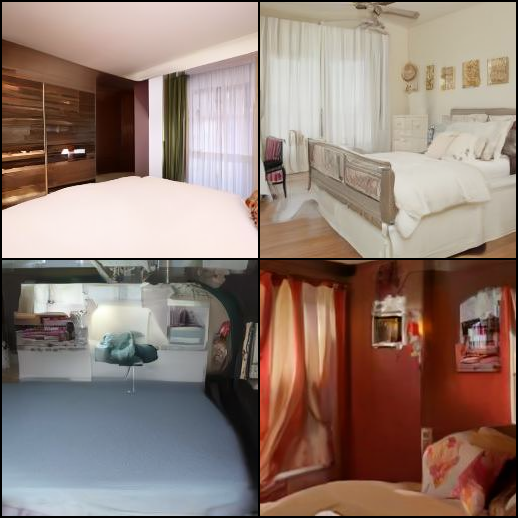}
    \caption{$K=100$}
    %\label{fig:gamma_small}
  \end{subfigure}
    % --- Third Subfigure ---
  \begin{subfigure}[b]{0.19\textwidth}
    \centering
    \includegraphics[width=\linewidth]{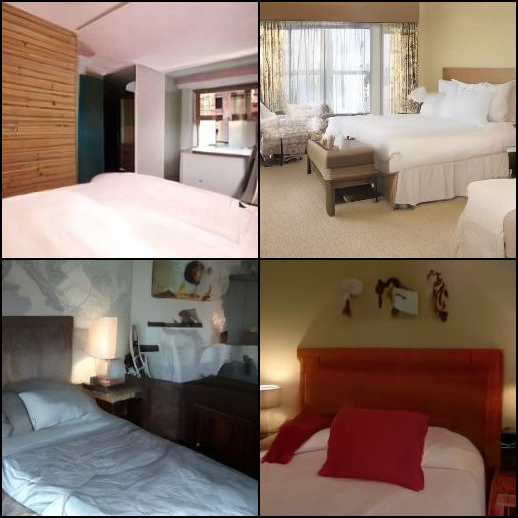}
    \caption{$K=200$}
    %\label{fig:gamma_small}
  \end{subfigure}
    % --- Third Subfigure ---
  \begin{subfigure}[b]{0.19\textwidth}
    \centering
    \includegraphics[width=\linewidth]{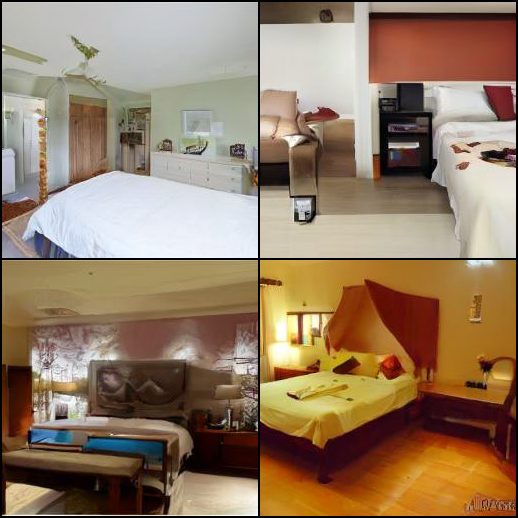}
    \caption{$K=1000$}
    %\label{fig:gamma_small}
  \end{subfigure}
  \caption{LSUN Bedroom. Generated samples using SANI using varying number of timesteps}
  \label{fig:lsun}
\end{figure}
%--------------------------------------------------------------------
\begin{figure}[!t]
  \centering
  % --- First Subfigure ---
    \begin{subfigure}[b]{0.27\textwidth}
    \centering
   \includegraphics[width=\linewidth]{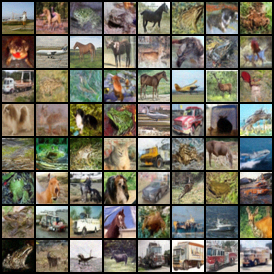}
    \caption{$K=10$}
    %\label{fig:gamma_optimal_gated_v2_coupled}
  \end{subfigure}
  %%\hfill  
  \begin{subfigure}[b]{0.27\textwidth}
    \centering
   \includegraphics[width=\linewidth]{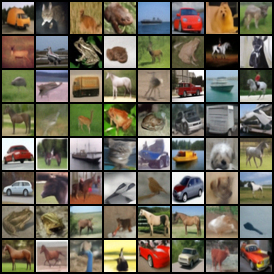}
    \caption{$K=25$}
    %\label{fig:gamma_optimal_gated_v2_coupled}
  \end{subfigure}
  %%\hfill  
  % --- Second Subfigure ---
  \begin{subfigure}[b]{0.27\textwidth}
    \centering
    \includegraphics[width=\linewidth]{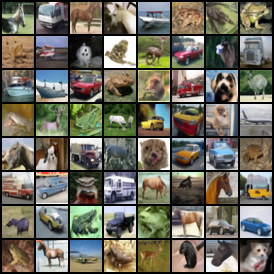}
    \caption{$K=50$}
    %\label{fig:gamma_optimal}
  \end{subfigure}
  
  %\hfill 
  % --- Third Subfigure ---
  \begin{subfigure}[b]{0.27\textwidth}
    \centering
    \includegraphics[width=\linewidth]{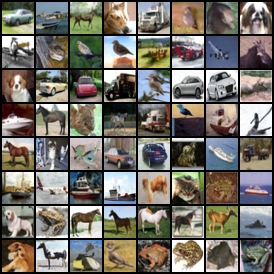}
    \caption{$K=100$}
    %\label{fig:gamma_small}
  \end{subfigure}
    % --- Third Subfigure ---
  \begin{subfigure}[b]{0.27\textwidth}
    \centering
    \includegraphics[width=\linewidth]{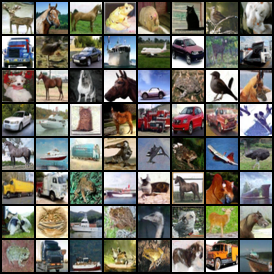}
    \caption{$K=200$}
    %\label{fig:gamma_small}
  \end{subfigure}
    % --- Third Subfigure ---
  \begin{subfigure}[b]{0.27\textwidth}
    \centering
    \includegraphics[width=\linewidth]{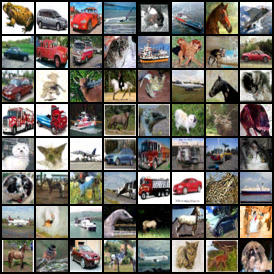}
    \caption{$K=1000$}
    %\label{fig:gamma_small}
  \end{subfigure}
  \caption{Cifar10 (LS). Generated samples using SANI using varying number of timesteps.}
  \label{fig:cifar10_ls}
\end{figure}
%--------------------------------------------------------------------
\begin{figure}[!t]
  \centering
  % --- First Subfigure ---
    \begin{subfigure}[b]{0.27\textwidth}
    \centering
    \includegraphics[width=\linewidth]{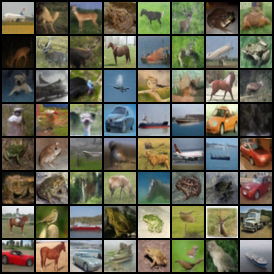}
    \caption{$K=10$}
    %\label{fig:gamma_optimal_gated_v2_coupled}
  \end{subfigure}
  \begin{subfigure}[b]{0.27\textwidth}
    \centering
    \includegraphics[width=\linewidth]{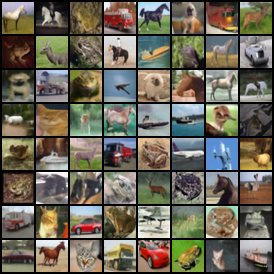}
    \caption{$K=25$}
    %\label{fig:gamma_optimal_gated_v2_coupled}
  \end{subfigure}
  %%\hfill  
  % --- Second Subfigure ---
  \begin{subfigure}[b]{0.27\textwidth}
    \centering
    \includegraphics[width=\linewidth]{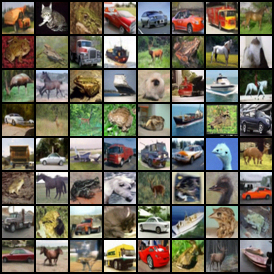}
    \caption{$K=50$}
    %\label{fig:gamma_optimal}
  \end{subfigure}
  
  %\hfill 
  % --- Third Subfigure ---
  \begin{subfigure}[b]{0.27\textwidth}
    \centering
    \includegraphics[width=\linewidth]{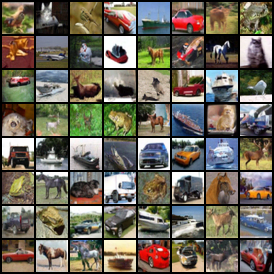}
    \caption{$K=100$}
    %\label{fig:gamma_small}
  \end{subfigure}
    % --- Third Subfigure ---
  \begin{subfigure}[b]{0.27\textwidth}
    \centering
    \includegraphics[width=\linewidth]{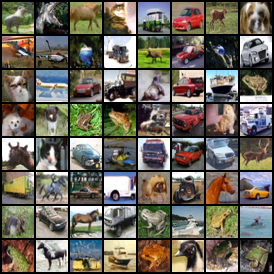}
    \caption{$K=200$}
    %\label{fig:gamma_small}
  \end{subfigure}
    % --- Third Subfigure ---
  \begin{subfigure}[b]{0.27\textwidth}
    \centering
    \includegraphics[width=\linewidth]{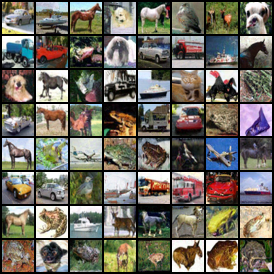}
    \caption{$K=1000$}
    %\label{fig:gamma_small}
  \end{subfigure}
  \caption{Cifar10 (CS). Generated samples using SANI using varying number of timesteps.}
  \label{fig:cifar10_cs}
\end{figure}
% 

% ===========================================================================
\end{document}